\documentclass[11pt]{article}

\usepackage[margin=1in]{geometry}
\usepackage[T1]{fontenc}
\usepackage{lmodern}          
\usepackage{float}
\usepackage{multirow}
\usepackage{algorithm}
\usepackage[noend]{algpseudocode}

\usepackage{amsfonts}
\usepackage{amsmath,amsthm}
\usepackage{array}
\usepackage{mathrsfs}
\usepackage{graphicx}
\usepackage{subcaption}
\usepackage{algpseudocode}
\usepackage{placeins}
\usepackage{booktabs}
\usepackage{url}
\usepackage[hidelinks]{hyperref}

\theoremstyle{proposition}

\theoremstyle{lemma}
\newtheorem{lemma}{Lemma}

\theoremstyle{corollary}
\newtheorem{corollary}{Corollary}

\theoremstyle{theorem}

\theoremstyle{definition}
\newtheorem{definition}{Definition}

\theoremstyle{assumption}

\theoremstyle{remark}

\usepackage{amsmath,amssymb}
\usepackage{graphicx}
\usepackage{booktabs}
\usepackage[numbers]{natbib}  
\usepackage{authblk}          

\title{Internal Guidance as a Sampling Method for Tilted Diffusion Models}

\author[1]{Curtis McDonald}
\author[1,2]{Peter Bartlett}
\affil[1]{Simons Institute, University of California Berkeley}
\affil[2]{Google DeepMind}

\date{Oct 7, 2026}  

\begin{document}
\maketitle

\begin{abstract}
Reward guidance methods sample an exponential tilt of a base density using a diffusion model trained with only samples from the base distribution. Existing methods use what we call external guidance: additive terms in the drift or multiplicative terms on the transition probabilities. We propose internal guidance which leaves the output of the base denoiser unmodified and instead changes the point at which it is evaluated. For discrete distributions embedded in the hypercube $\{-1,1\}^{d}$, the denoiser of any linearly tilted density is exactly the base denoiser evaluated at a shifted and rescaled input. In practice the base denoiser is a neural network trained on noised samples of the base distribution. The required evaluation points for sampling the tilted density would be out-of-distribution for basic noising of the base distribution. We modify the training algorithm with random diagonal rescalings of the data, so that the denoiser network learns both in- and out-of-distribution evaluations of the denoiser relative to the base distribution. A single denoiser network samples binary reward tilts, class labels, and infilling with no fine-tuning and no external guidance term.\footnote{Code available at \url{https://github.com/CMcDonald-1/internal_guidance_diffusion_models}}
\end{abstract}

\section{Introduction}

Diffusion models designed via Gaussian noising have become the
dominant generative algorithm for continuous distributions ~\citep{sohl2015deep, ho2020denoising, song2021score}, and continue to show promise for discrete targets as well~\citep{chen2023analog, hu2026elf, yang2026continuous}.
The original theory of diffusion models was established for unconditional
generation~\citep{song2019generative, song2021score}: given access to an initial pool of samples
from a distribution $\pi$, we train a model to sample from $\pi$. As the field has matured, interest has turned towards conditional generation and tilting. Given access to only a pool of samples from $\pi$, and a reward function $r$, we want to train a sampling algorithm to sample from an exponential tilt by the reward function for some strength $\tau$ with density $\tilde{\pi}(x) \propto \pi(x)e^{\tau r(x)}$. 

Existing methods utilize what the authors refer to as \textbf{external guidance}, implementing an additive vector field called a guidance term to modify SDE paths at sampling time \citep{dandapanthula2026tilting, skreta2025feynman, denker2024deft}, or a density ratio to reweight Markov chain conditional probabilities \citep{wan2026discrete,  nisonoff2025unlocking, lee2025debiasing, he2026rne}. Methods such as classier-fee guidance \citep{ho2022classifier} implement ad hoc and closed-form guidance terms, but are known to structurally sample the wrong invariant measure and introduce sampling artifacts such as corrupted samples and mode collapse \citep{chidambaram2024guidance, bradley2025classifier}. Exact guidance schemes derive the analytically correct guidance terms, but result in an intractable object that must either be learned in an expensive neural network training run, or approximated by a tractable analog such as an empirical average that introduces compounding errors at sample time.

A recent line of work \citep{holderrieth2025glass, moitra2026steering, mcdonald2026technical, moitra2026tractability} has opened up a new method for reward-tilted sampling we call \textbf{internal guidance}. Rather than modify the output of the denoiser of the base distribution with additive or multiplicative terms outside of the object, we realize we can instead shift our evaluation point inside the existing base denoiser. Recent work of \citet[Lemma 3.3]{moitra2026steering} has shown for an internal mapping function $H$ we call a symmetry map, the evaluation of the denoiser of a linear tilted density is an evaluation of the base denoiser at a modified input
\begin{align*}
G_{\text{tilt}}(u,a) = G_{\text{base}}(\tilde{u} = H_{1}(u,a), \tilde{a} = H_{2}(u,a)).
\end{align*}
These symmetry mappings are closed-form objects involving shifting and rescaling inputs, derived from the interactions of linear tilts with the quadratic log-density terms introduced via Gaussian convolution.

Internal guidance has clear benefits. There is no intractable external term to approximate, and the symmetry maps $H$ are all explicit. A single denoising diffusion model, trained on unconditional sampling of its input distribution, is able to sample an entire family of tilted densities at inference time with no additional reinforcement learning, fine tuning, or external guidance required.

The catch with this method is one of coverage. If we had access to the oracle denoiser $G$ we could implement such a method with no additional error over base sampling. However, we instead have a network $g_{\theta}$ trained to approximate $G$. This network is only accurate in an $L^{2}$ sense along the sequence of marginal densities produced by noising the base distribution, which we call the training manifold. The equivalent points $(\tilde{u}, \tilde{a})$ induced by the symmetry map represent an out-of-distribution query from the perspective of the base model denoiser, and we have no control on the accuracy of our trained network at this point.

The relationship of the symmetry map can be captured by an equivalence class: for a tilted density $\tilde{\pi}$ and an evaluation point $(u,a)$, the equivalence class is all points $(\tilde{u}, \tilde{a})$ where the base denoiser evaluated at $(\tilde{u},\tilde{a})$ is equal to the tilted denoiser at $(u,a)$. The work of \citet{moitra2026steering} can be interpreted as showing this equivalence class exists and is non-empty, and established one symmetry map which creates the class. However, this symmetry alone is not enough to force overlap between the equivalence class of points produced by tilted sampling and the high probability region of the base distribution training manifold the denoiser network was trained on.

To have equivalent points that live near the training manifold, we need to introduce additional symmetry relationships that expand the equivalence class. We restrict our study to discrete distributions embedded in the hypercube $\{-1,1\}^{d}$, noised under continuous Gaussian convolution, and a stochastic localization algorithm to implement the reverse sampling and produce discrete outputs. The benefit of embedding in the hypercube is all discrete points are constant norm $\|x\|^{2} = d$. The conditional densities which define the denoiser are quadratic tilts to the base density which introduce linear and quadratic terms to the log-density. By the constant norm property, all quadratic terms normalize out and the denoiser is a pure linear tilt to the base density. This opens up additional symmetries that expand the equivalence class and allow modification to the training manifold. With these additional symmetries, we train a denoiser network that can be evaluated at points which are out-of-distribution from the pure noising of the base distribution, but lie in the equivalence class of points seen at sampling time. We visualize a sketch of this idea in Figure \ref{manifold_overlapp}. Our contributions are threefold:

\begin{enumerate}
\item Conditional generation tasks such as infilling and reward tilting by binary rewards can be cast as linear tilts in an extended state space by promoting the reward coordinate to an additional random variable noised and denoised when sampling.
\item We establish the theory of \textbf{denoiser symmetry}, the evaluation of the denoiser of a linear tilt to a base density can be recast as an evaluation of the base denoiser at a shifted evaluation point. We establish additional symmetries specific to hypercube embedded discrete distributions, by the removal of quadratic terms in the conditional log-density.
\item We introduce diagonal matrix rescalings to our hypercube embedding during training called \textbf{$\Lambda$ focusing}, turning an out-of-distribution learning problem into an in-distribution learning problem, and allowing the base denoiser to be evaluated at locations off the typical noising path.
\end{enumerate}

We demonstrate performance on small discrete sampling targets such as a non-linear reward tilt to a small 12 dimension Ising model with ground truth comparison denoiser; generating and solving 4x4, 6x6, and 9x9 Sudoku puzzles; and conditional class generation of MNIST digit images. These problems are small enough to act as a test bed where the quality of our sampling can be compared against exact ground truth values as a diagnostic of effective sampling, with future work to focus on scaling and more complex target distributions for internal guidance.

\begin{figure}
\begin{center}
\includegraphics[width = \textwidth]{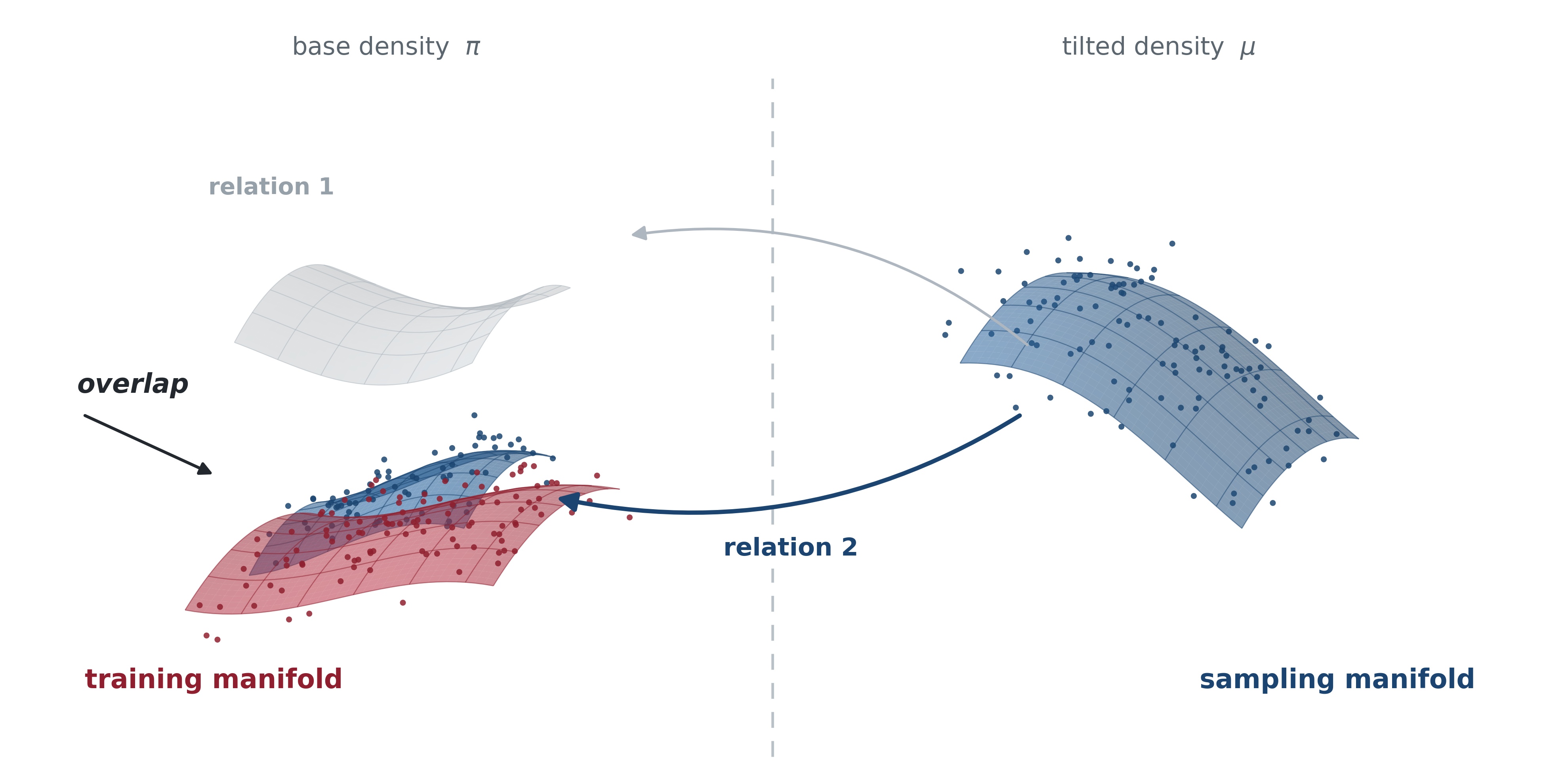}
\end{center}
\caption{Sketch of denoiser symmetry. The training manifold is where the denoiser network has learned the base denoiser. The sampling manifold is the typical vectors that will be encountered when sampling a tilted density. The equivalence relationships can find a point with equal evaluation of the base denoiser near the training manifold, not all equivalence relationships map back to the region of the training manifold.}\label{manifold_overlapp}
\end{figure}

\section{Setup}

We consider discrete distributions embedded in the hypercube $\mathcal{X} \to \{-1,1\}^{d}$, noised under Gaussian convolution, and sampled via a reverse process. This form of sampling was considered in \citet{chen2023analog}, referred to as ``analog bits'', although the structural results of linear tilts and internal guidance we derive here were not known at the time. Any discrete distribution can be embedded in the hypercube by assigning each element $x$ of the original discrete state space $\mathcal{X}$ of size less than $2^{d}$, a unique embedding vector $e(x) \in \{-1,1\}^{d}$. This is natural for token generation tasks with a finite library such as language models; constraint satisfaction tasks such as generating and solving Sudoku puzzles; image generation with discrete pixel values such as grayscale MNIST digits; and distributions natively in the hypercube such as Ising models.

Our target density $\pi$ is now restricted to a distribution on the hypercube $\{-1,1\}^{d}$. We are interested in sampling two forms of tilts to $\pi$. For any vector $v \in \mathbb{R}^{d}$, a linear tilt to $\pi$, $\pi_{v}(x) \propto \pi(x)e^{\langle v,x \rangle}$. For a binary reward function $r$ with outputs $\{-1,1\}$ tilts to $\pi$ by $\tau>0$ times $r$, $q_{\tau}(x) \propto  \pi(x)e^{\tau r(x)}$.

To unify these two objectives, we promote the reward value $r(x)$ to a new random variable $x_{*} = r(x)$ in an extended $d+1$ dimensional state space and consider linear tilts in the extended state space. Our original coordinates $x_{i}$ are already $\{-1,1\}$ valued, so from the perspective of our density $x_{*}$ is simply another $\{-1,1\}$ coordinate. It has a deterministic conditional distribution $\pi(x^{*}|x) = 1\{x^{*} = r(x)\}$, but there is no issue with having a probability 1 event in a discrete support set.  We then define the \textbf{extended base density} with $d+1$ coordinates where the reward-tilted version of $\pi$ is equivalent to a linear tilt in the extended state space by the one-hot vector $e_{*} = (0,\cdots, 0,1) \in \mathbb{R}^{d+1}$,\vspace{-0.8cm}

\begin{align*}
\pi(x,x^{*}) &= \pi(x)1\{r(x) = x^{*}\}\\
\pi(x)e^{\tau r(x)} &= \pi(x,x^{*})e^{\tau \langle e_{*}, (x,x^{*})\rangle}.
\end{align*}\vspace{-0.8cm}

As such, linear tilts by vectors $v$ and tilts by the reward function can both be unified as pure linear tilts in the extended state space. We thus consider $\pi$ as a distribution on the $d+1$ dimensional hypercube, with the final coordinate being the reward coordinate.

For values $a,\sigma>0$ consider a random variable defined by a linear combination of a sample from $\pi$ with an independent normal random variable
\begin{align*}
U = aX+\sigma Z, X \sim \pi, Z \sim N(0,I).
\end{align*}
The \textbf{denoiser} of $\pi$ is the conditional expectation of $X$ given $U$,
\begin{small}
\begin{align*}
F^{\text{base}}(u,a,\sigma) = E_{\pi}[X|u,a,\sigma] = \sum_{x}x \frac{\pi(x)e^{-\frac{1}{2\sigma^{2}}\|u - a x\|^{2}}}{E_{\pi}[e^{-\frac{1}{2\sigma^{2}}\|u - a X\|^{2}}]}.
\end{align*}
\end{small}
The denoiser defines the drift which implements the reverse SDE sampling algorithm. A special case is the variance preserving (VP) combination, when $a^{2}+\sigma^{2} = 1$, or $\sigma = \sqrt{1-a^{2}}$. In that case, we simply indicate $F^{\text{base}}(u,a)$ and note it is a VP combination.

In the extended state space, we are really considering the joint noising of the token bit string and the reward coordinate at the same time, $u = (u^{\text{token}}, u^{*})$
\begin{align*}
\begin{pmatrix}
U^{\text{token}}\\
U_{*}
\end{pmatrix} = a \begin{pmatrix}
x^{\text{token}}\\
x_{*}
\end{pmatrix}+\sqrt{1-a^{2}}\begin{pmatrix}
Z\\
Z_{*}
\end{pmatrix}\\
a \in [0,1], X,X* \sim \pi, Z \sim N(0,I),Z_{*} \sim  N(0,1).
\end{align*}
Thus our sampler carries a ``noisy reward'' value in its $d+1$ coordinate, and does joint denoising of the reward coordinate and the token string at the same time through the reverse diffusion.

Denoisers are expectations of quadratic tilts to the base density. By restricting our base density to be supported on the hypercube, all quadratic terms can be dropped from the exponent of the conditional densities that define the denoiser expectation. As such, all denoiser evaluations of $\pi$ are expectations of linear tilts to $\pi$, and all denoisers of linear tilts to $\pi$ are themselves further linear tilts to this object. One operator defines the denoiser of $\pi$ and every linear tilt to $\pi$: the general mean-tilt operator.
\begin{definition}
For any vector $v \in \mathbb{R}^{d+1}$ and any value $a \in [0,1]$, define VP tilt strength $\beta(a) = \frac{a}{1-a^{2}}$, and define the general mean-tilt operator,
\begin{align*}
G(v,a) = \sum_{x} x \frac{\pi(x)e^{ \beta(a) \langle v, x \rangle}}{ \sum_{\bar{x}} \pi(\bar{x})e^{ \beta(a) \langle v, \bar{x} \rangle}}.
\end{align*}
\end{definition}
The base denoiser $F^{\text{base}}$ is equal to the operator $G$.

\subsection{Sampling Via Stochastic Localization}

Our sampling algorithm will be the stochastic localization path, which is equivalent to a time shifting and rescaling of the marginal densities of score-based diffusion. See \citet{montanari2023sampling} and our discussion in Appendix \ref{app:SL} for further connection. The benefit of stochastic localization is the drift term is exactly the denoiser, rather than being related to the score via Tweedie's formula \citep{robbins1992empirical}. This makes it simpler to construct denoiser-to-denoiser relationships rather than score-to-score relationships.
\begin{definition}[Stochastic Localization]
For a target density $\pi$, initialize a random variable at the origin $X_{0} = 0$. Then follow the SDE using the denoiser as a drift at a specific time scaling,
\begin{align*}
dX_{t} = F^{\pi}(X_{t},t,\sqrt{t})dt+dB_{t},
\end{align*}
then the marginals at time $t$ are equivalent to
\begin{align*}
X_{t} = tX+\sqrt{t}Z, X \sim \pi, Z \sim N(0,I).
\end{align*}
\end{definition}\vspace{-0.2cm}
For large $T$, the normalized variable $X_{T}/T$ is equal to the target density with a small normal perturbation of scaling $1/\sqrt{T}$. The discrete implementation of this SDE as an algorithm is shown in Algorithm \ref{alg:sl-sampling}.

\section{Denoiser Symmetries}

First, we consider properties of the training manifold distribution. Typical points produced under VP noising of the base distribution are norm $\sqrt{d+1}$ with high probability at any signal level, and all roughly the same magnitude in each coordinate.

\begin{lemma}[Constant Norm Shell]\label{lem:const_norm}
Let $\pi$ be any distribution on $\{-1,1\}^{d+1}$ and $U = aX+\sqrt{1-a^{2}}Z$ with
$X\sim\pi$, $Z\sim N(0,I)$, $a\in[0,1]$. Then for any $a$,
\begin{enumerate}
\item $E\|U\|^{2} = d+1$ and
      $\mathrm{Var}\,\|U\|^{2} = 2(1-a^{4})(d+1)$; hence w.h.p. the norm of $U$ is near $\sqrt{d+1}$, for any $\epsilon>0$,
      $P\big(\,\big|\tfrac{\|U\|^{2}}{d+1}-1\big| > \epsilon\big)
        \le \tfrac{2(1-a^{4})}{\epsilon^{2}(d+1)}$.
\item The magnitude vector $(|U_{1}|,\dots,|U_{d+1}|)$ is i.i.d.\ with common
      law $|N(a,\,1-a^{2})|$. Its distribution does not depend on $\pi$.
\end{enumerate}
\end{lemma}

We then consider the points produced under a stochastic localization path of a tilted target density. The non-VP scaling of localization paired with a shift term from the linear tilt results in evaluation points that can be of various norm and have one large dominant coordinate in magnitude.

\begin{lemma}[Stochastic Localization Tilt]\label{lem:SL_tilt}
Let $X_{t}$ be a point observed at time $t$ in a stochastic localization for a target density which is a linear tilt to $\pi$ by a vector $ \tau v$. Let $\varphi = \frac{-1+\sqrt{5}}{2} \approx 0.618$ be the number such that $\varphi = 1-\varphi^{2} \implies \varphi = \beta^{-1}(1)$. Then the denoiser of the stochastic localization is the base denoiser at signal level $\varphi$ with a shift by $\tau v$ in the evaluation point,
\begin{align*}
E_{\pi_{v}}[X|X_{t} = tX+\sqrt{t}Z] &= G(\tau v+X_{t}, \varphi)\\
&=\sum_{x} x \frac{\pi(x)e^{ \langle \tau v+X_{t}, x \rangle}}{E_{\pi}[e^{ \langle \tau v+X_{t}, x \rangle}]}.
\end{align*}
In particular, for a reward tilt $v = e_{*}$, this amounts to large additive term in the $d+1$ coordinate.
\end{lemma}
Based on these two lemmas alone, we see the typical evaluation points for $G$ seen in VP noising of the base distribution and those required by stochastic localization of a tilted density look quite different. We introduce the following two lemmas to allow for overlap of these two kinds of evaluations. Both of these lemmas are only possible for hypercube embedded distributions where the quadratic tilt of the denoiser simplifies to a pure linear tilt.

First, any $G$ evaluation can be renormalized to a norm $\sqrt{d+1}$ vector and an equivalent VP signal level $\tilde{a}$.

\begin{lemma}[Renormalization]\label{renorm_lemma}
For any vector $v$ and signal level $a$, the evaluation of the general mean-tilt operator is equal to an evaluation at a norm $\sqrt{d+1}$ scaling of $v$ and an equivalent signal level $\tilde{a}$,
\begin{align*}
G(v, a) &= G(\sqrt{d+1}\frac{v}{\|v\|}, \tilde{a})\\
C &= \beta(a)\frac{\|v\|}{\sqrt{d+1}},  \tilde{a} = \frac{1}{2}(\sqrt{\frac{1}{C^{2}}+4}-\frac{1}{C}).
\end{align*}
\end{lemma}

Second, for the base density we can multiply our vectors by any diagonal matrix $\Lambda$ and maintain the property that each conditional expectation is a pure linear tilt. We call this method 
\textbf{$\Lambda$ focusing}.

\begin{figure}
\begin{center}

\end{center}
\end{figure}

\begin{lemma}[$\Lambda$ focusing]\label{lem:Lambda_focus}
Let $\Lambda$ be a diagonal matrix, and $U$ a vector produced by first putting $X \sim \pi$ through the matrix and then performing VP noising,
\begin{align*}
U = a \Lambda X+\sqrt{1-a^{2}}Z, X \sim \pi, Z \sim N(0,I), a \in [0,1].
\end{align*}
Then the denoiser is an evaluation of $G$ at the same signal level, but a $\Lambda$ mapping of the observed $u$
\begin{align*}
E[X|U = a \Lambda X+\sqrt{1-a^{2}}Z]&= G(\Lambda u, a).
\end{align*}
\end{lemma}
We construct our diagonal matrix $\Lambda$ to maintain $\text{tr}(\Lambda^{2})= d+1$ by taking a norm $\sqrt{d}$ vector $q \in \mathbb{R}^{d}$, and for some $ \lambda \in [1,\sqrt{d+1}]$,
\begin{align*}
\Lambda_{j,j} &= \sqrt{1+\frac{1}{d}-\frac{\lambda^{2}}{d}}q_{j}, j \neq d+1, &\Lambda_{d+1, d+1} = \lambda.
\end{align*}
This effectively moves our embedding points $x$ from the symmetric hypercube $\{-1,1\}^{d+1}$ to a non-symmetric hypercube more focused near the one-hot reward vector $e_{d+1}$ we are likely to encounter at sample time, visualized in Figure \ref{reward_focusing}. We consider two choices of $q$ which we call Mode 0 and Mode 1.

\begin{figure}
\begin{center}
\includegraphics[width = \textwidth]{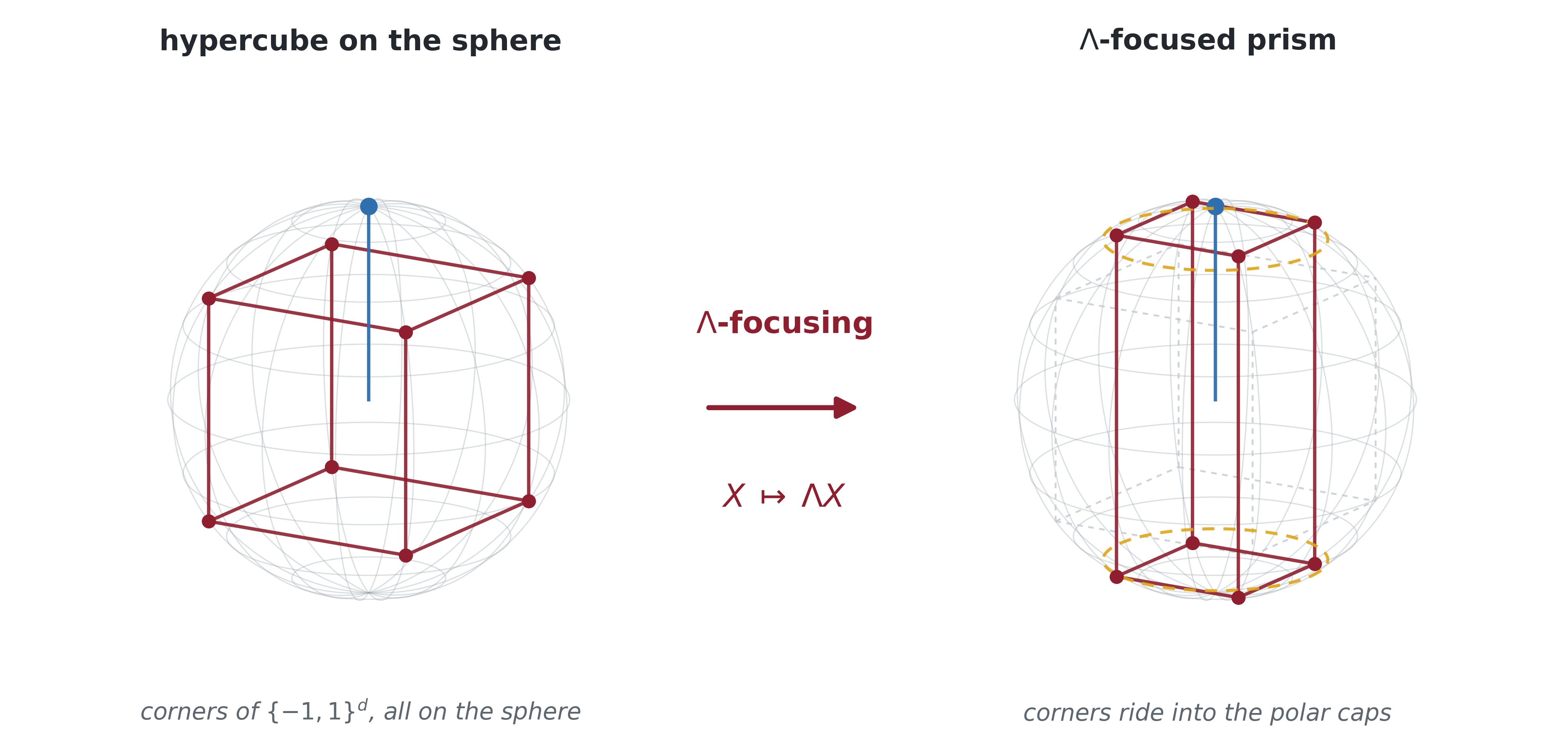}
\end{center}\caption{Reward based $\Lambda$ focusing Mode 0 maintains norm of hypersphere points while placing more mass in reward coordinate and increasing cosine similarity with the one-hot vector of the reward coordinate at the pole of the hypersphere.}\label{reward_focusing}
\end{figure}

\textbf{Mode 0: Reward Focusing}
Here we set $q = \vec{1}_{d}$ so all coordinates are present in the $\Lambda$ matrix and we trade mass from the token component of the vector to the reward coordinate, forcing the hypercube to reside near the pole of the hypersphere.

\textbf{Mode 1: Infill Training}
Let $A_{+}$ be a set of indexes which we would like to fix as positive, and $A_{-}$ a set of indexes we would like to fix as negative. Then $\pi$ conditioned on these indexes is the limit of an infinite strength linear tilt by the positive and negative vector
\begin{small}
\begin{align*}
\pi(x|x_{A_{+}} = 1, x_{A_{-}} = -1) = \lim_{\tau \to \infty} \pi(x)e^{\tau (\langle e_{A_{+}}, x \rangle - \langle e_{A_{-}}, x \rangle)}.
\end{align*}
\end{small}
As such, pick a set of indexes $A$ of size $|A| = k$, and set $q$ as the norm $d$ indicator set of A, shutting off coordinates outside the set, $q_{j} = 1\{j \in A\} \sqrt{d/k}$. With the indexes $A^{C}$ set to 0, this is forcing the trained network denoiser to learn conditionals with only partial information supplied by the indexes present in the set,
\begin{align*}
E[X|U_{A}, U = aX+\sqrt{1-a^{2}}Z] = G(\text{Diag}(1_{A}) u, a).
\end{align*}
We refer to this concept as ``soft masking'', similar in spirit to masked diffusion models \citep{austin2021structured,sahoo2024simple, shi2024simplified} where certain tokens are set to a `MASK' token and do not influence the conditional expectation. Here, we can get a similar effect by shutting off certain indexes with a $q_{j} = 0$, forcing the denoiser to learn conditional distributions with full information of the $U$ vector, and partial information with various random subsets $U_{A}$.

\subsection{Combined Results}
Combining these symmetries of shifting, normalizing, and $\Lambda$ focusing, we show how points generated via $\Lambda$ focusing the training set and points generated via the stochastic localization path of a tilted density can both be recast as evaluations of the base denoiser at a norm $\sqrt{d+1}$ vector and an equivalent VP signal level.
\begin{corollary}\label{Lambda_train_cor}
Let $U$ be a point generated at training time by $\Lambda$ focusing a training point and then applying VP noising
\begin{align*}
U = a \Lambda X+\sqrt{1-a^{2}}Z, X \sim \pi, Z \sim N(0,I), a \in [0,1],
\end{align*}
then the denoiser at this observation point is equal to an evaluation of the base denoiser at a norm $\sqrt{d+1}$ vector and an equivalent signal level $\tilde{a}$,
\begin{align*}
&E_{\pi}[X|U = a\Lambda X+\sqrt{1-a^{2}}Z]= G( \sqrt{d+1}\frac{\Lambda u}{\|\Lambda u\|}, \tilde{a})\\
&C = \frac{a}{1-a^{2}}\frac{\|\Lambda u\|}{\sqrt{d+1}}, \tilde{a} = \frac{1}{2}(\sqrt{\frac{1}{C^{2}}+4}-\frac{1}{C}).
\end{align*}
\end{corollary}

\begin{corollary}\label{SL_cor}
Let $X_{t}$ be a point observed at time $t$ in a stochastic localization for a target density which is a linear tilt to $\pi$ by a vector $ \tau v$. Then the evaluation of the denoiser can be cast as an evaluation of the base denoiser at a norm $\sqrt{d+1}$ vector and an equivalent signal level $\tilde{a}$,
\begin{align*}
&E_{\pi_{v}}[X|X_{t} = tX+\sqrt{t}Z]= G( \sqrt{d+1} \frac{\tau v + X_{t}}{\|\tau v+X_{t}\|}, \tilde{a})\\
C &= \frac{\|\tau v + X_{t}\|}{\sqrt{d+1}},  \tilde{a} = \frac{1}{2}(\sqrt{\frac{1}{C^{2}}+4}-\frac{1}{C}).
\end{align*}
\end{corollary}
At sampling time, we encounter evaluations of the $G$ operator at various $\tilde{a}$ levels and norm $\sqrt{d+1}$ vectors with most of their magnitude contained in the $d+1$ coordinate. Via $\Lambda$  focusing, points in our training set which do achieve reward create vectors with similar properties the network can train on. Both types of vectors inhabit the portion of the hypersphere near the one-hot pole $e_{*}$, which forces overlap of the training distribution and the typical points witnessed at sample time. A further visualization of these symmetries can be found in Appendix Figure \ref{overlap_sketch}.

\section{Training and Sampling Algorithms}

In training we may noise the data under any joint distribution of signal level $a$ and $\Lambda$ focusing matrix. After renormalizing the noisy vector to norm $\sqrt{d+1}$ with equivalent signal level $\tilde{a}$, the mean-tilt operator $G$ is the population minimizer of the inverse problem regardless of that distribution. Standard VP noising is the special case $\Lambda = I$ with no renormalization.

\begin{lemma}[General regression minimizer]\label{lem:reg_target}
Let $\mu$ be any noise distribution over pairs $(a,\Lambda)$ with $a\in[0,1]$ and
$\Lambda$ a diagonal matrix. Draw $(a,\Lambda)\sim\mu$, $X\sim\pi$, $Z\sim N(0,I)$
independently, and set $U = a\Lambda X + \sqrt{1-a^{2}}Z$. Let
$H:\mathbb{R}^{d+1}\times[0,1]\to\mathcal{S} \times [0,1]$ be the renormalization map of
Lemma~\ref{renorm_lemma},
\begin{align*}
H(v,a) &= \Big(\sqrt{d+1}\,\tfrac{v}{\|v\|},\ \tilde{a}\Big)\\ C&=\beta(a)\tfrac{\|v\|}{\sqrt{d+1}},\quad
\tilde{a}=\frac{1}{2}(\sqrt{\frac{1}{C^{2}}+4}-\frac{1}{C}).
\end{align*}
Then over all measurable functions $g$, the tilt mean operator $G$ is the population minimizer,
\begin{align*}
G = \operatorname*{argmin}_{g}\ E\big\|X - g\big(H(\Lambda U,a)\big)\big\|^{2}.
\end{align*}
\end{lemma}
We minimize the empirical loss by SGD in Algorithm~\ref{alg:batch-sampler}, drawing a fresh $(a,\Lambda)$ per training point. We sample $\beta$ log-uniformly on $[\beta_{\min},\beta_{\max}] = [0.02, 50]$ and set $a = a(\beta)$, the inverse of $\beta(a)=a/(1-a^{2})$. With probability $p_0$ we take $q = 1_d$, else a random $k$-subset is active. With probability $p_\lambda$ we take $\lambda = 1$, else $\lambda$ is log-uniform on $[1,\sqrt{d+1}]$. When $q = 1_d$ and $\lambda = 1$ this is plain VP noising and no renormalization is applied. We use $p_0 = 3/4$, $p_\lambda = 1/4$, so $3/16$ of draws are plain VP noising. 

\begin{algorithm}
\caption{Batch SGD step for the tilt--mean network $g_\theta$}
\label{alg:batch-sampler}
\begin{algorithmic}[1]
\Require Data $\{X_i\} \subset \{-1,1\}^{d+1}$ (reward in coordinate $d+1$); batch size $N$; step size $\gamma$
\State $\mathcal{L} \gets 0$
\For{$n = 1, \dots, N$}
        \State Draw $x \sim \{X_i\}$,\; $Z \sim \mathcal N(0, I_{d+1})$,\; $\beta \sim \mathrm{LogUnif}[\beta_{\min},\beta_{\max}]$,\; $a \gets a(\beta)$
    \State W.p.\ $p_{0}$: $q \gets \mathbf 1_d$;\; else draw $k \sim \mathrm{Unif}\{2,\dots,d\}$ and set $q \gets \sqrt{d/k}1_{S}$ for a random $k$-subset $S$
    \State W.p.\ $p_{\lambda}$: $\lambda \gets 1$;\; else $\lambda \sim \mathrm{LogUnif}[1,\sqrt{d+1}]$
    \State $\Lambda_{j,j} \gets \sqrt{1 + \tfrac1d - \tfrac{\lambda^2}{d}} q_j ,j \le d$;\quad $\Lambda_{d+1,d+1} \gets \lambda$
    \State $\eta \gets\bigl(a \Lambda^{2} x + \sqrt{(1-a^2)} \Lambda Z\bigr)$ \Comment{tilt vector}
    \If{$q = \mathbf 1_d$ \textbf{and} $\lambda = 1$ } 
        \State $(v^*, s^*) \gets \bigl(\eta,\ a\bigr)$ 
        \Comment{basic VP noising}
    \Else
        \State $C \gets \beta(a)\tfrac{\|\eta\|}{\sqrt{d+1}},\quad
\tilde{a} \gets \frac{1}{2}(\sqrt{\frac{1}{C^{2}}+4}-\frac{1}{C})$;\\ \quad\quad ~~~$v^* \gets \sqrt{d+1}\;\eta/\|\eta\|$
    \EndIf
    \State $\mathcal{L} \gets \mathcal{L} + \|x - g_\theta(v^*, s^*)\|^2$
\EndFor
\State $\theta \gets \theta - \gamma \nabla_\theta \mathcal{L}$
\end{algorithmic}
\end{algorithm}

We note our sampler also makes use of self-conditioning \citep{chen2023analog} which is a common practice in the field. We discuss the form of self-conditioning in Appendix \ref{app:self-cond} to avoid over complication of the discussion in the main body.

Sampling a linear tilt follows Corollary~\ref{SL_cor}: at each step the state is shifted by the tilt vector and renormalized before querying the network, shown in Algorithm~\ref{alg:sl-sampling}. The terminal $X_{T}$ is continuous, and is decoded with one final call to the denoiser network.

\begin{algorithm}
\caption{Stochastic-localization sampling from the tilted density $\mu \propto \pi\, e^{\langle \nu,\cdot\rangle}$ using symmetries}
\label{alg:sl-sampling}
\begin{algorithmic}[1]
\Require Trained network $g_\theta$; tilt vector $\nu = \tau_{\rm infill}\bigl(e_{A^+} - e_{A^-}\bigr) + \tau_{\rm reward}\, e_*$; time grid $0 = t_0 < t_1 < \dots < t_K = T$
\State $X_{t_0} \gets 0$
\For{$k = 0, \dots, K-1$}
    \State $\delta \gets t_{k+1} - t_k$
    \State $\eta \gets \nu + X_{t_k}$
    \State $C\gets \tfrac{\|\eta\|}{\sqrt{d+1}},\quad
s^{*}\gets \frac{1}{2}(\sqrt{\frac{1}{C^{2}}+4}-\frac{1}{C})$,\\~~~ $v^* \gets \sqrt{d+1}\;\eta/\|\eta\|$ 
    \State Draw $\xi \sim \mathcal N(0, I_{d+1})$
    \State $X_{t_{k+1}} \gets X_{t_k} + \delta\, g_\theta(v^*, s^{*}) + \sqrt{\delta}\,\xi$ 
\EndFor
\State $\eta \gets \nu + X_{T}$, \quad compute $(v^*, s^*)$ from $\eta$ as above
\State \Return $\hat x \gets \mathrm{sign}\big(g_\theta(v^*, s^*)\big)$
\end{algorithmic}
\end{algorithm}

\section{Experiments}

We demonstrate the performance of the method in three experimental settings: a 12-dimensional Ising model with a non-linear reward tilt, 4x4, 6x6, and 9x9 Sudoku unconditional grid generation and puzzle solving (infilling), and MNIST conditional digit generation. Full technical descriptions of the target density, training set, model architecture, and additional details for each experiment can be found in Appendix \ref{appendix:experiments}. All experiments are sampled up to a stochastic localization time $T = 20$, with the discrete grid $\{t_{k}\}_{k=0}^{K}$ is constructed with $K = 1000$ and a scheduling algorithm defined in Appendix \ref{appendix:schedule}. We note that at each $\tau$ tilt strength, we know the theoretical probability of producing a reward-achieving solution using the log-odds curve as discussed in Appendix \ref{appendix:logodds}.
\subsection{12-Dimensional Ising Model}

The 12-dimensional Ising model has a reward tilt which gives reward when the number of positive bits is exactly 10. The state space of the Ising model is small enough to enumerate giving exact ground truth for the denoiser to compare our network against. In Appendix \ref{appendix:experiments} Figure \ref{fig:ising_ood} we see the network denoiser has learned the ground truth denoiser well with low $L^{2}$ error across tilt strengths, and produces the correct proportion of reward-achieving solutions as predicted by the log-odds curve in Appendix \ref{appendix:experiments} Figure \ref{fig:ising_logodds}. The Ising model is a small scale experiment with ground truth denoiser to calibrate the method before proceeding to larger state spaces.

\subsection{4x4 Sudoku}

The 4x4 Sudoku problem has a small enough state space to enumerate the ground truth denoiser. We compare the network denoiser to the ground truth in $L^{2}$ error in Appendix \ref{appendix:experiments} Figure \ref{fig:4x4_ood}, the network learns the denoiser well.
For the log-odds curve in Appendix \ref{appendix:experiments} Figure \ref{fig:4x4_logodds} we can consider the stochastic localization algorithm implemented with either the ground truth denoiser or the network estimation. Both methods track the log-odds up to 98\% reward-achieving solutions at $\tau \approx 3.5$, but the network sampler does start to lag at higher tilt strengths while the ground truth denoiser does continue to track at high tilts up to $\tau = 5$.

With 288 reward-achieving puzzles, we can track if the network is uniformly sampling these puzzles. In 3,000 generated reward-achieving puzzles, the exact-denoiser sampler visits all 288 solutions, each between 4 and 21 times, and
passes a $\chi^{2}$ test of uniformity ($\chi^{2}_{287} = 280.1$, $p = 0.60$). The
network visits 287 of the 288 solutions, each between 0 and 25 times, and fails the
test ($\chi^{2}_{287} = 391.9$, $p < 10^{-4}$). The network therefore covers nearly
the whole solution set, but it does not sample it uniformly: under the reward tilt it
has learned to favor some solutions over others. Results are visualized in Appendix \ref{appendix:experiments} Figure \ref{fig:visit_hist}.

An added benefit of embedding in the hypercube is we can do early bit-by-bit decoding of the continuous vector $X_{t}$ at a time $t$ based on the coordinate-wise entries of the decoder $|E[X_{i}|X_{t}]| \geq 0.75$. One of the main complaints against continuous diffusion samplers for discrete targets relative to autoregressive models and masked diffusion is that because they operate in continuous space there is no discernible discrete token string until the final decoding step. This technique solves that problem and produces a partially complete token string before terminal decoding, visualized in Appendix \ref{appendix:experiments} Figure \ref{fig:4x4_fuzzy}.

\subsection{6x6 Sudoku}

For 6x6 Sudoku, the state space is large and there are vastly more reward-achieving puzzles than the 4x4 problem. In the 4x4 problem, the state space was size $2^{32}$ with only 288 reward-achieving puzzles. The training set of size 150,000 contained 5\% reward-achieving puzzles indicated by a 1 in their reward coordinate. Thus each reward-achieving puzzle was seen in the training set and a memorization check was performed with a $\chi^{2}$ test of uniform sampling.

For 6x6 Sudoku the state space is $2^{108}$ with 28 million reward-achieving puzzles. The training set of size 5 million contained 5\% reward-achieving puzzles, which is 250,000 of the 28 million possible valid grids, 0.88\%.

We cannot fully enumerate the denoiser for reward tilts in the 6x6 setting, but we can show the $L^{2}$ error for the untilted density $\tau = 0$ is controlled for the learned network in Appendix \ref{appendix:experiments} Figure \ref{fig:6x6_ood}. The network also properly tracks the log-odds across tilt strengths in Appendix \ref{appendix:experiments} Figure \ref{fig:6x6_logodds}.

We check memorization by the rate the network reproduces a reward-achieving puzzle provided in the training set. With the training set containing $0.88\%$ of the reward-achieving puzzles, in 1,000 reward-achieving generations proper uniform sampling would produce exact matches with the training set by chance $3$--$14$ with a $95\%$ interval. We observe $10$ matches in 1,000 samples, consistent with no memorization.

\subsection{9x9 Sudoku}

For 9x9 Sudoku the state space is too large to enumerate a ground truth denoiser. The number of reward-achieving puzzles is so large there is a near 0 probability of recreating a training set puzzle under proper sampling. The training set is 10 million Sudoku grids with 5\% being reward-achieving puzzles.

The 9x9 Sudoku target density is the first example where the effect of self-conditioning is noticeable, self-conditioning is discussed in more detail in Appendix \ref{app:self-cond}. In summary, we consider two modes of self-conditioning: mode ``two-pass'' reuses the network's denoiser evaluation at a given time and maintains a memory one Markov structure. Mode ``carry'' reuses a previous time point's denoiser value and breaks this Markov structure but is more commonly used in modern methods.

The two methods give different results when tracking the log-odds in Appendix \ref{appendix:experiments} Figures~\ref{fig:9x9_logodds_twopass} and~\ref{fig:9x9_logodds_carry}. The reward coordinate is denoised jointly with the tokens, so the sampler reports its own prediction of reward in the $d+1$ coordinate, and a correctly sampling model should have this prediction agree with the graded reward. The two-pass self-conditioning mode tracks the log-odds in its own reward prediction in its $d+1$ coordinate, but not in the graded validity, which plateaus near $60\%$. The model's $d+1$ coordinate reports it is producing reward-achieving grids far more often than it actually is. The carry self-conditioning mode plateaus near $90\%$ graded validity with a $8$--$12\%$ gap between predicted and graded reward. The reward function is very sharp, mis-sampling a single bit gives 0 reward. Sampled non-reward grids are typically close to reward ones: $90\%$ of two-pass samples lie within $10$ out of $324$ bits of the nearest valid grid in Hamming distance (Appendix \ref{appendix:experiments} Figure~\ref{fig:ham_cdf}), consistent with a few bits mis-sampled late in the trajectory being penalized by the sharp validity reward.

The model performs both unconditional grid generation and infilling, which some may call ``solving'' a puzzle, by adding a cell-fixing tilt to the reward tilt (Figure~\ref{fig:sudoku9-both}). This type of sampling is natural for the diffusion model here which operates under linear tilts and decodes in any order, but would be difficult for a left-to-right autoregressive model which would need to backtrack against later given values. We use a reward tilt strength and infill tilt strength of $\tau = 3$ and carry self-condition mode across various numbers of given cell values in Appendix \ref{appendix:experiments} Table~\ref{tab:sudoku9-infill} and Figure~\ref{fig:infill_vals_plot}. Note the sampler may not maintain the initial given cell values as they are a tilt vector not a hard constraint. Given $71/81$ cells the model maintains the initial cells and achieved reward in $100\%$ of $2000$ draws. Given $11/81$ cells, the initial cell values are maintained $92\%$ of the time, a reward achieved $87\%$ of the time, and both objectives satisfied $80\%$ of the time.

Finally, the model has no indication of memorizing the training set. Of $1000$ unconditional draws at $\tau = 3$ none match the training set. Splitting the training set into two halves, the cell-wise Hamming distance from each reward-achieving sample to its nearest reference grid in one half has the same distribution as that of sample from the split training set half in Appendix Figure~\ref{fig:9x9_mem_check}. Model produced reward-achieving puzzles are novel and no more similar than a two part split of training points to themselves.

\begin{figure}
    \centering
    \begin{subfigure}[t]{0.4\textwidth}
        \centering
        \includegraphics[width=\linewidth]{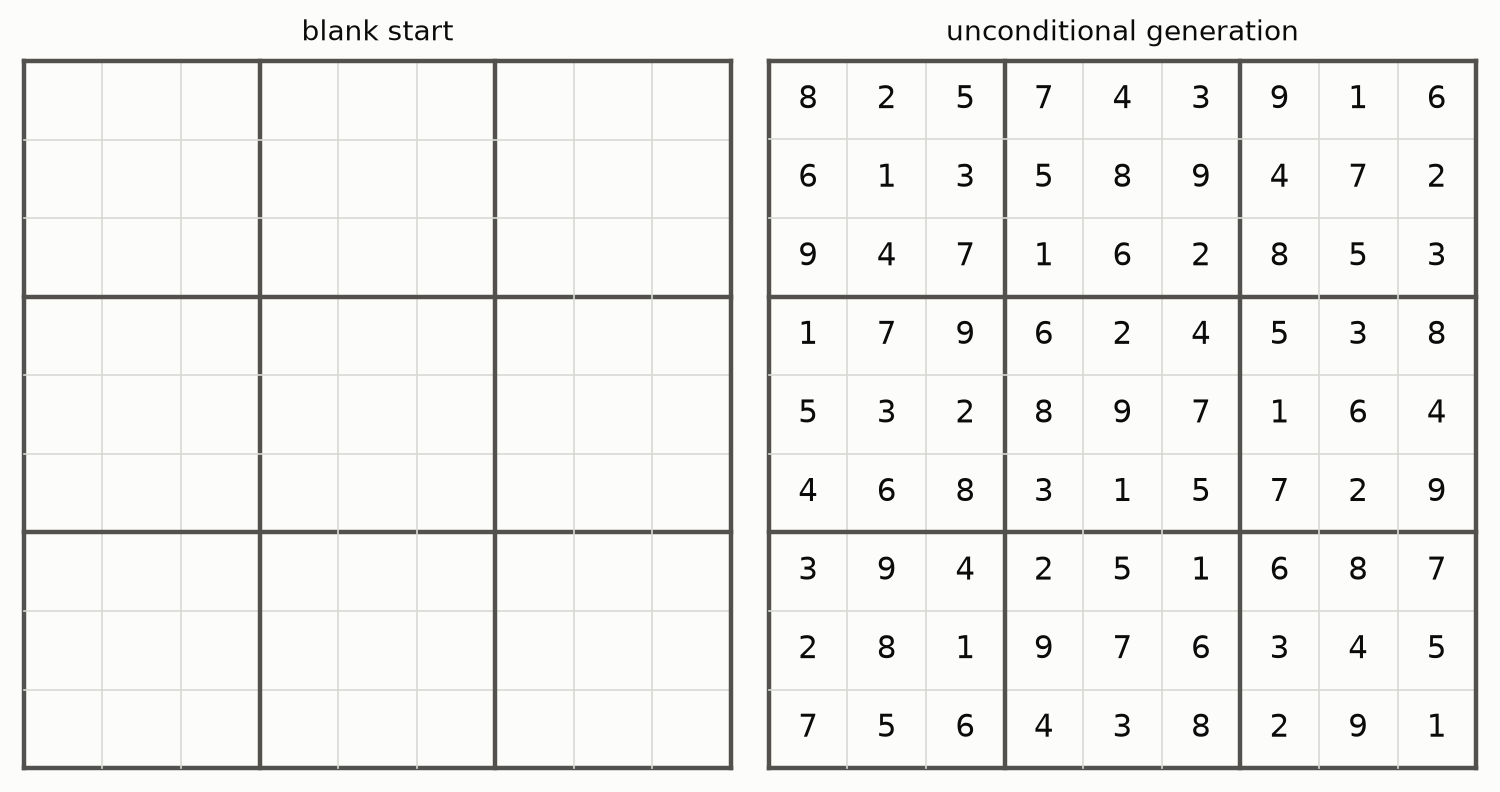}
        \caption{Unconditional generation from a blank grid.}
        \label{fig:sudoku9-unconditional}
    \end{subfigure}
    \hfill
    \begin{subfigure}[t]{0.4\textwidth}
        \centering
        \includegraphics[width=\linewidth]{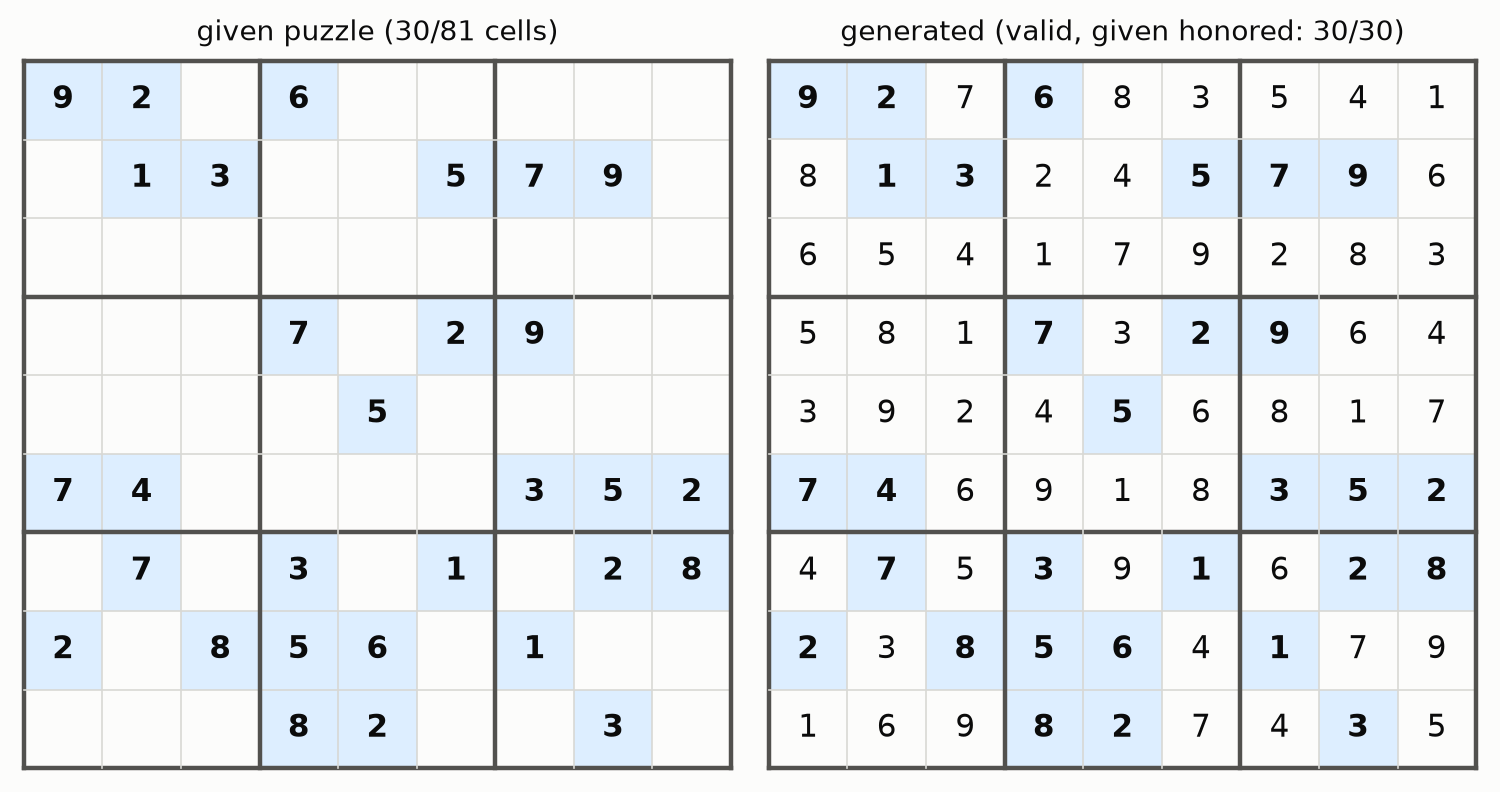}
        \caption{Conditional generation from a partially given puzzle.}
        \label{fig:sudoku9-infill}
    \end{subfigure}
    \caption{Sudoku 9x9 completions sampled from the tilt-operator network.}
    \label{fig:sudoku9-both}
\end{figure}

\begin{figure}
\begin{center}
\includegraphics[width = 0.44\textwidth]{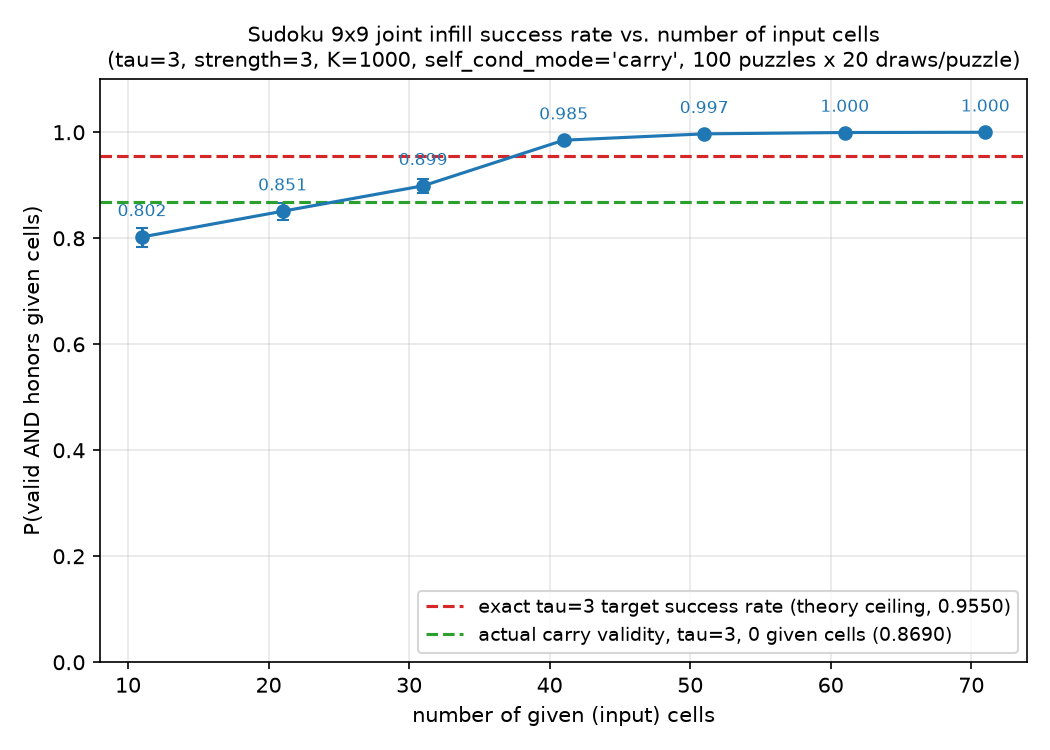}
\end{center}\caption{Probability of correct grid infilling across input cell numbers. Given no initial cells, should return a valid grid 95\% of the time, we measure saturation at 87\% without any infilling.}\label{fig:infill_vals_plot}
\end{figure}

\subsection{MNIST Digit Images}
The MNIST data set consists of 60,000 images of hand drawn digits 0-9 and 10,000 hold out examples. We encode reward as a $10$-bit one-hot class block, and generating class tilts along the whole block with 1 class being positively tilted and the other 9 negatively tilted. Generated images are classified by an independent CNN classifier ($98.74\%$ held-out accuracy). Since the scorer has a ceiling to its classifier accuracy, the observable success rate saturates at the classifier's per-class sensitivity; points beyond $\tau \approx 1.69$ are uninformative.

In Appendix \ref{appendix:experiments} Figure \ref{fig:MNIST_big_grid} we see generated pools of images show similar variety and quality as hold out examples not used for training indicating no mode collapse or degradation of digit quality under tilting. 

In Appendix \ref{appendix:experiments} Figures \ref{fig:mnist-logodds-twopass}-\ref{fig:mnist-logodds-carry-pooled} self-conditioning mode two-pass and carry are also compared for their tracking of the log-odds. In this experiment, carry mode seems to over-sample a class at medium tilts. Two-pass tracks the classifier-corrected exact rate within its confidence interval through $\tau \approx 1$, in both the classifier score and the network's own reward block. Carry overshoots the tilt in every class: at $\tau = 0.51$ it produces the requested class $57.8\%$ of the time against an exact $45.9\%$, with the $95\%$ interval excluding the exact value at every $\tau$ from $0.26$ to $1.03$. Both methods produce 95\% class correct images at large tilts, and further tilt strengths go beyond the judge CNN's sensitivity.

\subsubsection*{Acknowledgments}
We gratefully acknowledge the support of the NSF for FODSI through grant DMS-2023505, of the NSF and the Simons Foundation for the Collaboration on the Theoretical Foundations of Deep Learning through awards DMS-2031883 and 814639, and of the ONR through MURI award N000142112431, and a research grant from Bridgewater to the Simons Institute.

\section*{AI Use Statement}

In this work, we used generative AI tools to generate synthetic data sets, help develop theoretical models or conceptual frameworks, formulate mathematical claims, assist in the writing of proofs, design or provide feedback on research methodology or experiments, implement methods, clean and reformat datasets, support qualitative and thematic data analysis, interpret results. AI assisted coding was used in the project to generate synthetic training sets, write training code, and write evaluation code. AI was used to suggest hyper-parameter choices and present analysis of outputs. All code and outputs also human verified. All theorems and proofs are original constructs of the authors. A first draft written with purely human analysis. This draft is then presented to the AI for suggested edits, which may be incorporated by the authors into a refined proof. All proofs are human verified as correct and began with the authors own steps of logic as the foundation.

We have not used generative AI tools to provide critical ingredients for proving mathematical claims, propose or refine hypotheses, or assist with translation. ``Critical ingredients'' for proofs is subjective. The authors do not consider AI as presenting any of the critical ideas that make the proofs work, but rather suggesting useful notation and suggesting more succinct summaries of steps. All logical steps in the proofs were human designed initially.

Additionally, we used generative AI tools to suggest experimental parameters, create or edit software code, summarize or analyse existing literature, discover research topics or identify gaps, brainstorming, sourcing/searching for information, edit a research paper to improve readability, identify relevant literature.

We have reviewed all AI-assisted work. We take responsibility for the final content of this work, including text, claims or artifacts produced with the aid of generative AI.

\bibliographystyle{apalike}
\bibliography{tilted_diff}


\clearpage
\appendix
\thispagestyle{empty}

\onecolumn

\section{PROOFS}

\textbf{Proof of Lemma \ref{lem:const_norm}}
\begin{proof}
Throughout the proof $X$ and $Z$ are independent and $X_i^{2}=1$ for every $i$. 

\textbf{Condition (1):}

Note the moments $E[Z_i]=E[Z_i^{3}]=0$, $E[Z_i^{2}]=1$, $E[Z_i^{4}]=3$ for a standard normal random variable.

The random variable $U$ is the linear combination of $X$ and $Z$, and $Z$ is mean 0. The norm of the sum decomposes into the sum of the norms with a cross term that is 0,
\begin{align*}
E[\|U\|^{2}]&= a^{2}E[\|X\|^{2}]+2a\sqrt{1-a^{2}}\langle E[X], E[Z] \rangle+(1-a^{2})E[\|Z\|^{2}]\\
&= a^{2}E[\|X\|^{2}]+(1-a^{2})E[\|Z\|^{2}].
\end{align*}
Since the state space of $X$ is restricted to the hypercube $\{-1,1\}^{d+1}, \|X\|^{2}=d+1$ w.p. 1. Since $Z$ is standard normal. $E[\|Z\|^{2}] = d+1$. The expectation is then $d+1$ for any VP combination $a$. The variance decomposes as
\begin{align*}
E[(\|U\|^{2}-(d+1))^{2}]&= \sum_{i,j = 1}^{d+1}E[(U_{i}^{2}-1)(U_{j}^{2}-1)].
\end{align*}
Note that $X_{i}^{2}=1$ constantly, so we expand
\begin{align*}
 U_i^{2}-1 &= 2a\sqrt{1-a^{2}}X_iZ_i + (1-a^{2})(Z_i^{2}-1).
\end{align*}
Consider first the diagonal terms $i = j$,
\begin{align*}
E[(U_i^{2}-1)^{2}]
  &= 4a^{2}(1-a^{2})E[Z_i^{2}]
  + (1-a^{2})^{2}E[(Z_i^{2}-1)^{2}]
  + 4a(1-a^{2})^{3/2}E[X_i]E[Z_i^{3}-Z_i]\\
  &= 4a^{2}(1-a^{2}) + 2(1-a^{2})^{2}\\
  &= 2(1-a^{4}).
\end{align*}
Then the cross terms $i \neq j$ are mean 0. Expand $E[(U_{i}^{2}-1)(U_{j}^{2}-1)]$, every term vanishes since $Z_i\perp Z_j$ and both are independent of $X$:
\begin{align*}
  E[X_iX_jZ_iZ_j] &= E[X_iX_j]\,E[Z_i]\,E[Z_j] = 0,\\
  E[X_iZ_i(Z_j^{2}-1)] &= E[X_i]\,E[Z_i]\,E[Z_j^{2}-1] = 0,\\
  E[(Z_i^{2}-1)(Z_j^{2}-1)] &= E[Z_i^{2}-1]\,E[Z_j^{2}-1] = 0.
\end{align*}
Therefore $\operatorname{Var}\|U\|^{2}=\sum_i\operatorname{Var}U_i^{2}=2(1-a^{4})(d+1)$. Since the variance is on the same scale as the expectation $d+1$, a Chebyshev's inequality shows the norm is highly concentrated around its expected value in ratio,
\[
  P\Big(\Big|\tfrac{\|U\|^{2}}{d+1}-1\Big|>\epsilon\Big)
  = P\big(|\|U\|^{2}-(d+1)|>\epsilon(d+1)\big)
  \le \frac{2(1-a^{4})(d+1)}{\epsilon^{2}(d+1)^{2}}
  = \frac{2(1-a^{4})}{\epsilon^{2}(d+1)}.
\]

\textbf{Condition (2):}

Condition on $X=x\in\{-1,1\}^{d+1}$. Since $|x_i|=1$,
\[
  |U_i| = |a x_i + \sqrt{1-a^{2}}\,Z_i| = |a + \sqrt{1-a^{2}}\,x_iZ_i|.
\]
Coordinate-wise sign flips preserve $N(0,I)$, so $(x_1Z_1,\dots,x_{d+1}Z_{d+1})$
has the same law as $Z$. Hence, conditionally on $X=x$,
$(|U_1|,\dots,|U_{d+1}|)$ has the law of
$(|a+\sqrt{1-a^{2}}\,Z_1|,\dots,|a+\sqrt{1-a^{2}}\,Z_{d+1}|)$.
This vector has i.i.d.\ coordinates with law $|N(a,1-a^{2})|$. The conditional
law does not depend on $x$, so it is also the unconditional law, whatever $\pi$ is.
\end{proof}

\textbf{Proof of Lemma \ref{lem:SL_tilt}}
\begin{proof}
Let $q(x)\propto \pi(x)e^{\tau \langle x, v \rangle}$. It is a property of the stochastic localization path the marginal density of $X_{t}$ is equivalent to a specific linear combination with a normal random variable,
\begin{align*}
X_{t} = t X+\sqrt{t}Z, p(x_{t}|x) \sim N(t x, t I).
\end{align*}
The conditional density for $p(x|x_{t})$ is then
\begin{align*}
p(x|x_{t}) &\propto q(x)e^{-\frac{1}{2 t}\|x_{t}-tx\|^{2}}\\
&\propto \pi(x)e^{\tau \langle x, v \rangle}e^{-\frac{1}{2 t}\|x_{t}\|^{2}+\frac{t}{t}\langle x, x_{t} \rangle- \frac{t}{2}\|x\|^{2}},
\end{align*}
with $x$ as the random variable and $x_{t}$ as the conditioning point, $\|x_{t}\|^{2}$ can be dropped as a normalizing constant scaling all $x$ locations equally. Furthermore, by the assumption $x$ lives in the hypercube, it is constant norm and $\|x\|^{2}$ can also be dropped. Thus $p(x|x_{t})$ is proportional to a pure linear tilt to the base density $\pi$,
\begin{align*}
p(x|x_{t}) &\propto \pi(x)e^{\langle \tau v + x_{t}, x \rangle}.
\end{align*}
Recall the definition of the tilt strength function $\beta(a) = \frac{a}{1-a^{2}}, a \in [0,1]$. The golden ratio $\theta = \frac{1+\sqrt{5}}{2}\approx 1.618$ is precisely the number such that $\theta+1= \theta^{2}$. We define $\varphi = \theta - 1 \approx 0.618$ as the number such that
\begin{align*}
1-\varphi^{2}&= 1-(\theta-1)^{2}= - \theta^{2}+2\theta= -1+\theta= \varphi.
\end{align*}
Thus it is the case that $\beta(\varphi) = 1$ and we can express the stochastic localization's denoiser in terms of $G$ using $\varphi$ as our mixture level and $\tau v + X_{t}$ as the evaluation location,
\begin{align*}
E_{\pi_{v}}[X|X_{t} = tX+\sqrt{t}Z]
=\sum_{x} x \frac{\pi(x)e^{ \beta(\varphi)\langle \tau \nu+X_{t}, x \rangle}}{E_{\pi}[e^{ \beta(\varphi) \langle \tau \nu+X_{t}, x \rangle}]} = G(\tau v+X_{t}, \varphi).
\end{align*}
\end{proof}

\textbf{Proof of Lemma \ref{renorm_lemma}}
\begin{proof}
For the mean-tilt operator $G(v,a)$, the VP tilt strength operator $\beta(a) = \frac{a}{1-a^{2}}$ captures the term that scales the inner product $\langle v,x \rangle$,
\begin{align*}
G(v,a) = \sum_{x} x \frac{\pi(x)e^{ \beta(a) \langle v, x \rangle}}{ \sum_{\bar{x}} \pi(\bar{x})e^{ \beta(a) \langle v, \bar{x} \rangle}}.
\end{align*}
The operator $\beta$ maps $[0,1] \to [0,\infty)$ and is increasing, thus any tilt strength scaling $b$ can be achieved by some signal level that is the $\beta$ mapping at that value $a = \beta^{-1}(b)$. Inverting $\beta$ results in finding the roots of a quadratic solved by quadratic formula,
\begin{align*}
\beta(a)&= C\\
a &= (1-a^{2})C\\
C a^{2}+a-C&=0\implies a = \frac{-1\pm \sqrt{1+4C^{2}}}{2C},
\end{align*}
only the positive root results in $a \in[0,1]$, so we have inverse mapping for any $C \in [0,\infty)$,
\begin{align*}
\beta^{-1}(C)&= \frac{1}{2}\left(\sqrt{\frac{1}{C^{2}}+4}-\frac{1}{C} \right).
\end{align*}
For any inputs $v,a$ to the base denoiser, we can multiply and divide the $v$ vector by any scaling $\lambda$ in the exponent and collect the remaining term in front of the inner product as some new tilt strength,
\begin{align*}
 \sum_{x} x \frac{\pi(x)e^{ \beta(a) \langle v, x \rangle}}{ \sum_{\bar{x}} \pi(\bar{x})e^{ \beta(a) \langle v, \bar{x} \rangle}}&=  \sum_{x} x \frac{\pi(x)e^{ (\beta(a)\lambda) \langle \frac{v}{\lambda}, x \rangle}}{ \sum_{\bar{x}} \pi(\bar{x})e^{ (\beta(a)\lambda) \langle \frac{v}{\lambda}, \bar{x} \rangle}}.
\end{align*}
Setting $\lambda = \frac{\|v\|}{\sqrt{d+1}}$ results in $\tilde{v} = \frac{v}{\lambda}$ as being a norm $\sqrt{d+1}$ vector, and we solve for the equivalent signal level $\tilde{a}$ via the inverse map $\beta^{-1}(\beta(a)\lambda)$.

\end{proof}

\textbf{Proof of Lemma \ref{lem:Lambda_focus}}
\begin{proof}
In this proof we take advantage of each coordinate of $x$ itself being $\{-1,1\}$ valued such that $x_{i}^{2} = 1$ for all coordinates $i$. This is a property of the hypercube $\{-1,1\}^{d}$, while the hypersphere only has the property that $\|x\|^{2} = d$. By construction the conditional density of $U|X$ is $N(a \Lambda X, (1-a^{2})I)$. Therefore, the reverse conditional density via Bayes rule is a quadratic tilt to the base density
\begin{align*}
p(x|u,a,\Lambda)&\propto \pi(x)p(u|x,a,\Lambda)\\
&\propto \pi(x)e^{-\frac{1}{2(1-a^{2})}\|u - a\Lambda x\|^{2}}\\
&\propto \pi(x)e^{\frac{a}{1-a^{2}} \langle \Lambda u, x \rangle}e^{-\frac{1}{2(1-a^{2})}\|u\|^{2}-\frac{a^{2}}{2(1-a^{2})} \sum_{i=1}^{d+1}\Lambda_{i,i}^{2}x_{i}^{2}}.
\end{align*}
$X$ is the random variable while $u$ is the conditioning value, thus $\|u\|^{2}$ can be ignored as a normalizing constant. Since each $x_{i}^{2} = 1$ for any element of the hypercube, no matter the values of $\Lambda_{i,i}$ the $x$ dependent quadratic term is also constant and is removed in normalization. Thus the conditional expectation is a pure linear tilt to the base density.
\end{proof}

\textbf{Proof of Corollary \ref{Lambda_train_cor}}
\begin{proof}
In this proof we collect the previous denoiser symmetry results of Lemma \ref{renorm_lemma} and \ref{lem:Lambda_focus} the specify the form of the base denoiser under $\Lambda$ focusing as a $\sqrt{d+1}$ vector with an equivalent signal level $\tilde{a}$. Via Lemma \ref{lem:Lambda_focus}, the conditional expectation after $\Lambda$ focusing and noising is $G$ at $\Lambda U$ and $a$,
\begin{align*}
E_{\pi}[X|U = a\Lambda X+\sqrt{1-a^{2}}Z]= G( \Lambda u, a),
\end{align*}
we can then renormalize to $\tilde{u} = \sqrt{d+1}\frac{\Lambda u}{\|\Lambda u\|}$ and solve for $\tilde{a}$ using Lemma \ref{renorm_lemma}.
\end{proof}

\textbf{Proof of Corollary \ref{SL_cor}}
\begin{proof}
In this proof we collect the previous denoiser symmetry results of Lemma \ref{lem:SL_tilt} and \ref{renorm_lemma} to specify the form of the tilted denoiser under a stochastic localization path as a $\sqrt{d+1}$ vector with an equivalent signal level $\tilde{a}$. By Lemma \ref{lem:SL_tilt}, the denoiser of the stochastic localization path for a tilt defined by a vector $\tau v$ is a linear shift to the evaluation point of the mean-tilt operator at the signal level $\varphi \approx 0.618$,
\begin{align*}
E_{\pi_{v}}[X|X_{t} = tX+\sqrt{t}Z] = G(\tau v+X_{t}, \varphi),
\end{align*}
we can then renormalize to $\tilde{u} = \sqrt{d+1}\frac{\tau v + X_{t}}{\|\tau v+ X_{t}\|}$ and solve for $\tilde{a}$ using Lemma \ref{renorm_lemma}.

\end{proof}

\textbf{Proof of Lemma \ref{lem:reg_target}}
\begin{proof}
The loss function here is set up as an inverse problem, where for any two random variables with a joint distribution, the minimizer of the expected squared loss of one variable minus a function of the other is the conditional mean. The previous Lemmas have established that various conditional means can be expressed as evaluations of $G$ at norm $\sqrt{d+1}$ vectors, which shows this loss function with renormalization is an equivalent expression of an inverse problem.

Throughout the proof let $(a,\Lambda)\sim\mu$, $X\sim\pi$ and $Z\sim N(0,I)$ be independent random variables and
$U = a\Lambda X + \sqrt{1-a^{2}}Z$. For a measurable $g$ write
\begin{align*}
L(g) = E_{(a,\Lambda)\sim\mu,\,X\sim\pi,\,Z}\big\|X - g(U,a,\Lambda)\big\|^{2}.
\end{align*}
The proof proceeds in three steps: (i) for fixed $(a,\Lambda)$ the minimizer of
the squared loss is the conditional mean of $X$ given $U$; (ii) for random
$(a,\Lambda)\sim\mu$, the tower property reduces the problem to (i) because
$(a,\Lambda)$ is independent of $X$; (iii) the resulting minimizer is
$G(\Lambda u,a)$, which Lemma~\ref{renorm_lemma} rewrites as $G(H(\Lambda u,a))$,
a function of the renormalized input alone.

\medskip
\noindent\emph{Step 1: fixed $(a,\Lambda)$.}
Fix $a\in[0,1]$ and a diagonal $\Lambda$, and let $U_{a,\Lambda} = a\Lambda X + \sqrt{1-a^{2}}Z$
with $X\sim\pi$, $Z\sim N(0,I)$ independent. For any measurable $f$,
the tower property gives the orthogonal decomposition
\begin{align*}
E_{X,Z}\big\|X - f(U_{a,\Lambda})\big\|^{2}
&= E_{X,Z}\big\|X - E[X\mid U_{a,\Lambda}]\big\|^{2}
 + E_{X,Z}\big\|E[X\mid U_{a,\Lambda}] - f(U_{a,\Lambda})\big\|^{2},
\end{align*}
It is then a standard fact that the minimizing function is the conditional mean
\begin{align*}
g_{a,\Lambda}(u) := E_{X\sim\pi}\big[X \mid U_{a,\Lambda} = u\big]
\quad\text{satisfies}\quad
g_{a,\Lambda} = \operatorname*{argmin}_{f} E_{X,Z}\big\|X - f(U_{a,\Lambda})\big\|^{2}.
\end{align*}
By Lemma~\ref{lem:Lambda_focus}, the posterior of $X$ given $U_{a,\Lambda}=u$ is
the linear tilt $\pi(x)e^{\beta(a)\langle\Lambda u,x\rangle}$, hence $g_{a,\Lambda}(u) = G(\Lambda u, a)$.

\medskip
\noindent\emph{Step 2: random $(a,\Lambda)\sim\mu$.}
Condition on $(a,\Lambda)$ and apply the tower property with the
$(a,\Lambda)$ expectation outermost,
\begin{align*}
L(g) = E_{(a,\Lambda)\sim\mu}\Big[\, E\big[\|X - g(U,a,\Lambda)\|^{2} \,\big|\, a,\Lambda\big] \Big].
\end{align*}
Because $(a,\Lambda)$ is independent of $(X,Z)$, the conditional law of $(X,U)$
given $(a,\Lambda)$ is the law of $(X, U_{a,\Lambda})$ from Step~1; in density form,
\begin{align*}
p(x,u\mid a,\Lambda) = p(u\mid x,a,\Lambda)\,p(x\mid a,\Lambda) = p(u\mid x,a,\Lambda)\,\pi(x),
\end{align*}
i.e.\ $X$ is conditionally independent of $(a,\Lambda)$, and $(a,\Lambda)$ enters only
through the noising kernel $p(u\mid x,a,\Lambda)$. Therefore the inner conditional expectation is
\begin{align*}
E\big[\|X - g(U,a,\Lambda)\|^{2} \,\big|\, a,\Lambda\big]
= E_{X\sim\pi,\,Z}\big\|X - g(U_{a,\Lambda},a,\Lambda)\big\|^{2},
\end{align*}
which, for each fixed $(a,\Lambda)$, is exactly the objective of Step~1 with
$f = g(\cdot,a,\Lambda)$. Since $g$ is allowed to depend on $(a,\Lambda)$, the choice of
$g(\cdot,a,\Lambda)$ for one value of $(a,\Lambda)$ does not constrain its choice for another,
so the outer expectation is minimized by minimizing the integrand pointwise in $(a,\Lambda)$:
\begin{align*}
\min_{g} L(g)
= E_{(a,\Lambda)\sim\mu}\Big[\min_{f} E_{X,Z}\big\|X - f(U_{a,\Lambda})\big\|^{2}\Big],
\qquad
g^{\star}(u,a,\Lambda) = g_{a,\Lambda}(u) = G(\Lambda u,a).
\end{align*}

\medskip
\noindent\emph{Step 3: renormalization.}
By Lemma~\ref{renorm_lemma}, for any $v\in\mathbb{R}^{d+1}\setminus\{0\}$ and $a\in[0,1]$,
\begin{align*}
G(v,a) = G\Big(\sqrt{d+1}\,\tfrac{v}{\|v\|},\ \tilde{a}\Big) = G\big(H(v,a)\big),
\qquad \tilde{a} = \beta^{-1}\Big(\beta(a)\tfrac{\|v\|}{\sqrt{d+1}}\Big),
\end{align*}
so $g^{\star}(u,a,\Lambda) = G(H(\Lambda u,a))$. In particular the unconstrained minimizer
depends on $(U,a,\Lambda)$ only through $H(\Lambda U,a)$. The class
$\{g\circ H(\Lambda\,\cdot,\cdot) : g \text{ measurable on } \mathcal{S}\times[0,1]\}$
is a subset of the measurable functions of $(U,a,\Lambda)$, and the minimizer over the
larger class lies in it, so it is also the minimizer over the smaller class:
\begin{align*}
G = \operatorname*{argmin}_{g} E\big\|X - g\big(H(\Lambda U,a)\big)\big\|^{2},
\end{align*}
uniquely up to sets of measure zero under the law of $H(\Lambda U,a)$.
\end{proof}

\FloatBarrier 
\section{Additional Discussion}
\subsection{Visualization of Denoiser Symmetry}
\FloatBarrier 
\begin{figure}[t]
\begin{center}
\includegraphics[width = \textwidth]{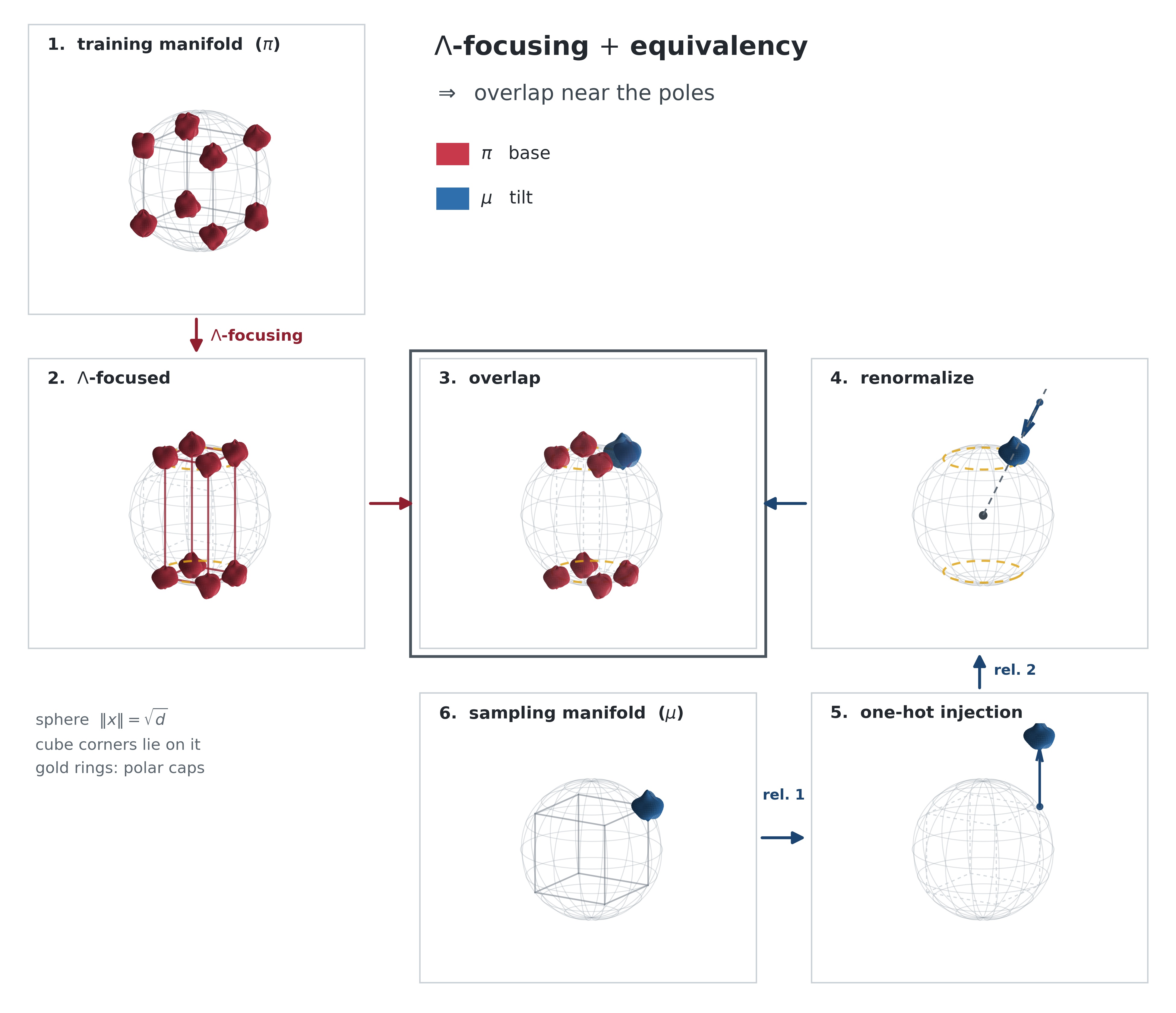}
\end{center}
\caption{Sketch of overlap between renormalized training set via $\Lambda$ focusing and points shifted by tilt along stochastic localization path.}\label{overlap_sketch}
\end{figure}
\FloatBarrier 

\subsection{Stochastic Localization}\label{app:SL}

Stochastic localization can be viewed as a different time scaling and spatial scaling than the usual score-based diffusion derived from the forwards flow Ornstein-Uhlenbeck process as outlined in \citet{montanari2023sampling}. We re-derive the equivalence here for an audience more familiar with the score-based construction. The score-based diffusion is variance preserving, initialized at a standard normal draw, and uses the score as its drift. The stochastic localization process is not variance preserving (it is growing in magnitude over time), is initialized at a Dirac mass at the origin, and uses the denoiser as its drift term. Both can be derived from the same reverse diffusion process theory.

\paragraph{Time reversal.}
Both constructions are reverse-time implementations of a forward diffusion producing a sequence of marginal densities that are equivalent to convolutions with Gaussian random variables of various mixture amounts.
By the time-reversal result of \citet{anderson1982reverse}, popularized in modern literature by
\citet{song2021score, lipman2023flow}, consider any target density $\pi$ and initialize $X_{0} \sim \pi$. Consider any forward process
\begin{align*}
dX_{t} = f(X_{t}, t)\,dt + g(t)\,dB_{t}, \qquad t \in [0,T], X_{0} \sim \pi,
\end{align*}
and let $p_{t}$ denote the marginal density of $X_{t}$ with score
$s_{t} = \nabla \log p_{t}$. The reverse process started from $Y_{0} \sim p_{T}$,
\begin{align}
dY_{t} = -\big[f(Y_{t}, T-t) - g(T-t)^{2}\, s_{T-t}(Y_{t})\big]dt + g(T-t)\,dB_{t},
\label{eq:reverse_sde}
\end{align}
satisfies $Y_{t} \sim p_{T-t}$, i.e.\ it traverses the forward marginals in
reverse order. The only quantity required to implement
\eqref{eq:reverse_sde} is the score, hence the name score-based diffusion.
For general $f$ and $g$ the marginal $p_{t}$, and hence $s_{t}$, has no specific closed form.

\paragraph{Linear forward processes.}
When the forward drift $f(X_{t}, t)$ is linear in $X_{t}$ the marginals are Gaussian
convolutions and the score is available through Tweedie's formula \citep{robbins1992empirical}. Let
$\alpha, \sigma : [0,T] \to \mathbb{R}_{+}$ be differentiable with
$\sigma/\alpha$ non decreasing, and define
\begin{align*}
\kappa(t) := \frac{d}{dt}\log\frac{\sigma(t)}{\alpha(t)}
 = \frac{\dot{\sigma}(t)}{\sigma(t)} - \frac{\dot{\alpha}(t)}{\alpha(t)} \;\ge\; 0 .
\end{align*}
Consider the linear forward process
\begin{align}
dX_{t} = \frac{\dot{\alpha}(t)}{\alpha(t)} X_{t}\,dt
 + \sigma(t)\sqrt{2\kappa(t)}\,dB_{t},
\qquad X_{0} = \alpha(0)X + \sigma(0)Z,
\label{gaussian_flow}
\end{align}
with $X \sim \pi$ and $Z \sim N(0,I)$ independent. A direct computation shows that its marginals are equivalent to convolution with a normal random variable at a specific scaling,
\begin{align*}
X_{t} \overset{d}{=} \alpha(t) X + \sigma(t) Z,
\qquad\text{i.e.}\qquad
p_{t}(x_{t} \mid X = x_{0}) = N\big(x_{t};\, \alpha(t) x_{0},\, \sigma(t)^{2} I\big).
\end{align*}
Writing $m_{t}(x) := E[X \mid X_{t} = x]$ for the denoiser, Tweedie's formula
\citep{robbins1992empirical} gives the score as a posterior mean under a quadratic tilt of $\pi$ and is related to the denoiser,
\begin{align*}
s_{t}(x_{t})
 = \int -\frac{x_{t} - \alpha(t) x_{0}}{\sigma(t)^{2}}\;
   \frac{\pi(x_{0})\, e^{-\frac{1}{2\sigma(t)^{2}}\|x_{t} - \alpha(t) x_{0}\|^{2}}}
        {\int \pi(x_{0}')\, e^{-\frac{1}{2\sigma(t)^{2}}\|x_{t} - \alpha(t) x_{0}'\|^{2}} dx_{0}'}\, dx_{0}
 = \frac{\alpha(t)\, m_{t}(x_{t}) - x_{t}}{\sigma(t)^{2}} .
\end{align*}
Substituting into \eqref{eq:reverse_sde}, and abbreviating
$\alpha = \alpha(T-t)$, $\sigma = \sigma(T-t)$, $\kappa = \kappa(T-t)$, the
reverse process is
\begin{align}
dY_{t}
 &= \Big[-\frac{\dot{\alpha}}{\alpha} Y_{t} + 2\sigma^{2}\kappa\, s_{T-t}(Y_{t})\Big]dt
    + \sigma\sqrt{2\kappa}\,dB_{t}.
\label{eq:reverse_score}
\end{align}
Equivalently stated in terms of the denoiser instead of the score,
\begin{align}
dY_{t}
 &=\Big[-\Big(\frac{\dot{\alpha}}{\alpha} + 2\kappa\Big) Y_{t}
    + 2\kappa\alpha\, m_{T-t}(Y_{t})\Big]dt
    + \sigma\sqrt{2\kappa}\,dB_{t}.
\label{eq:reverse_denoiser}
\end{align}

\paragraph{Signal level and growth.}
A linear combination $\alpha(t)X + \sigma(t)Z$ is called variance preserving (VP) when $\alpha(t)^{2}+\sigma(t)^{2} = 1$ for all $t$. Any general combination can be factored into a
scalar growth term and a variance-preserving (VP) combination. Define the signal-to-noise fraction $\lambda$ and growth term $\tau$
\begin{align*}
\lambda(t)^{2} := \frac{\alpha(t)^{2}}{\alpha(t)^{2} + \sigma(t)^{2}},
\qquad
\tau(t) := \sqrt{\alpha(t)^{2} + \sigma(t)^{2}},
\end{align*}
so that
\begin{align*}
\alpha(t) X + \sigma(t) Z
 = \tau(t)\Big(\lambda(t) X + \sqrt{1 - \lambda(t)^{2}}\, Z\Big).
\end{align*}
The VP case is $\tau \equiv 1$. The VP scaling uniquely identifies the probability distribution of the random variable induced by a convolution with a normal, any scaling does not change the ``shape'' of the density but simply the magnitude of the random variable.

Score-based diffusion and stochastic localization are defined by two different choices of $\alpha$ and $\sigma$.

\subsubsection{Score-Based Diffusion from the Ornstein--Uhlenbeck Process}
Taking $\alpha(t) = e^{-t}$, $\sigma(t) = \sqrt{1 - e^{-2t}}$ gives
$\dot{\alpha}/\alpha = -1$ and $2\sigma^{2}\kappa = 2$, so
\eqref{gaussian_flow} is the Ornstein--Uhlenbeck (OU) process with drift linear
in $X_{t}$ and constant in time, and \eqref{eq:reverse_score} reads
\begin{align*}
dX_{t} &= -X_{t}\,dt + \sqrt{2}\,dB_{t}, \qquad X_{0} \sim \pi,\\
dY_{t} &= \big(Y_{t} + 2 s_{T-t}(Y_{t})\big)dt + \sqrt{2}\,dB_{t}.
\end{align*}
This is a VP combination, $\tau(t) \equiv 1$, with signal level
$\lambda(t)^{2} = e^{-2t}$. The reverse drift is implemented by the score
alone.

\subsubsection{The Stochastic Localization Time Scaling}
Fix $T > 0$ and take $\alpha(t) = T-t$, $\sigma(t) = \sqrt{T-t}$ on
$t \in [0,T]$. Then
\begin{align*}
\frac{\dot{\alpha}(t)}{\alpha(t)} = -\frac{1}{T-t},
\qquad
\frac{\dot{\sigma}(t)}{\sigma(t)} = -\frac{1}{2(T-t)},
\qquad
\kappa(t) = \frac{1}{2(T-t)},
\qquad
\sigma(t)^{2}\, 2\kappa(t) = 1,
\end{align*}
so the forward process \eqref{gaussian_flow} is
\begin{align*}
dX_{t} = -\frac{1}{T-t} X_{t}\,dt + dB_{t},
\qquad X_{0} = TX + \sqrt{T}Z,
\end{align*}
which is driven to $X_{T} = 0$. In the reverse process
\eqref{eq:reverse_denoiser} the coefficient of $Y_{t}$ vanishes,
$\dot{\alpha}/\alpha + 2\kappa = 0$, and $2\kappa\alpha = 1$, so the
reverse drift is exactly the denoiser:
\begin{align*}
dY_{t} = E\big[X \,\big|\, tX + \sqrt{t}Z = Y_{t}\big]\,dt + dB_{t},
\qquad Y_{0} = 0 .
\end{align*}
This is the stochastic localization process of
\citet{montanari2023sampling}, with $Y_{t} \overset{d}{=} tX + \sqrt{t}Z$
and $Y_{t}/t \to X$ as $t \to \infty$.

\paragraph{Relation to the OU diffusion.}
Both processes are naturally indexed by their reverse-time clocks, along which the signal level increases. In stochastic localization time $s$, the
signal level and growth are
\begin{align*}
\lambda_{\mathrm{SL}}(s)^{2} = \frac{s}{1+s},
\qquad
\tau_{\mathrm{SL}}(s) = \sqrt{s(1+s)} .
\end{align*}
For the OU diffusion on a horizon $[0,T]$, write $Y^{\mathrm{OU}}_{r} \sim X_{T-r}$
for the reverse process at reverse time $r \in [0,T]$, so that
\begin{align*}
\lambda_{\mathrm{OU}}(r)^{2} = e^{-2(T-r)},
\qquad
\tau_{\mathrm{OU}} = 1 .
\end{align*}
Matching signal levels, $e^{-2(T-r)} = s/(1+s)$, gives the time change
\begin{align*}
r(s) = T - \tfrac{1}{2}\log\big(1 + 1/s\big),
\end{align*}
which is increasing in $s$ and maps $s \in \big[(e^{2T}-1)^{-1}, \infty\big)$
onto $[0,T)$; as $T \to \infty$ every $s > 0$ is covered. Therefore
\begin{align*}
Y^{\mathrm{SL}}_{s} \;\overset{d}{=}\; \sqrt{s(1+s)}\; Y^{\mathrm{OU}}_{r(s)} .
\end{align*}
Stochastic localization is thus the OU reverse diffusion run under the time
change $r(s)$ and rescaled in space by $\tau_{\mathrm{SL}}(s)$: the marginal
density produced at reverse time $s$ of stochastic localization is the
marginal produced at reverse time $r(s)$ of the score-based reverse
diffusion, scaled by $\sqrt{s(1+s)}$. An example of the equivalent marginals produced for a simple mixture of normals is shown in Figure \ref{fig:sl_ou} to visualize the equivalence.

\begin{figure}[H]
\begin{center}
\includegraphics[width = \textwidth]{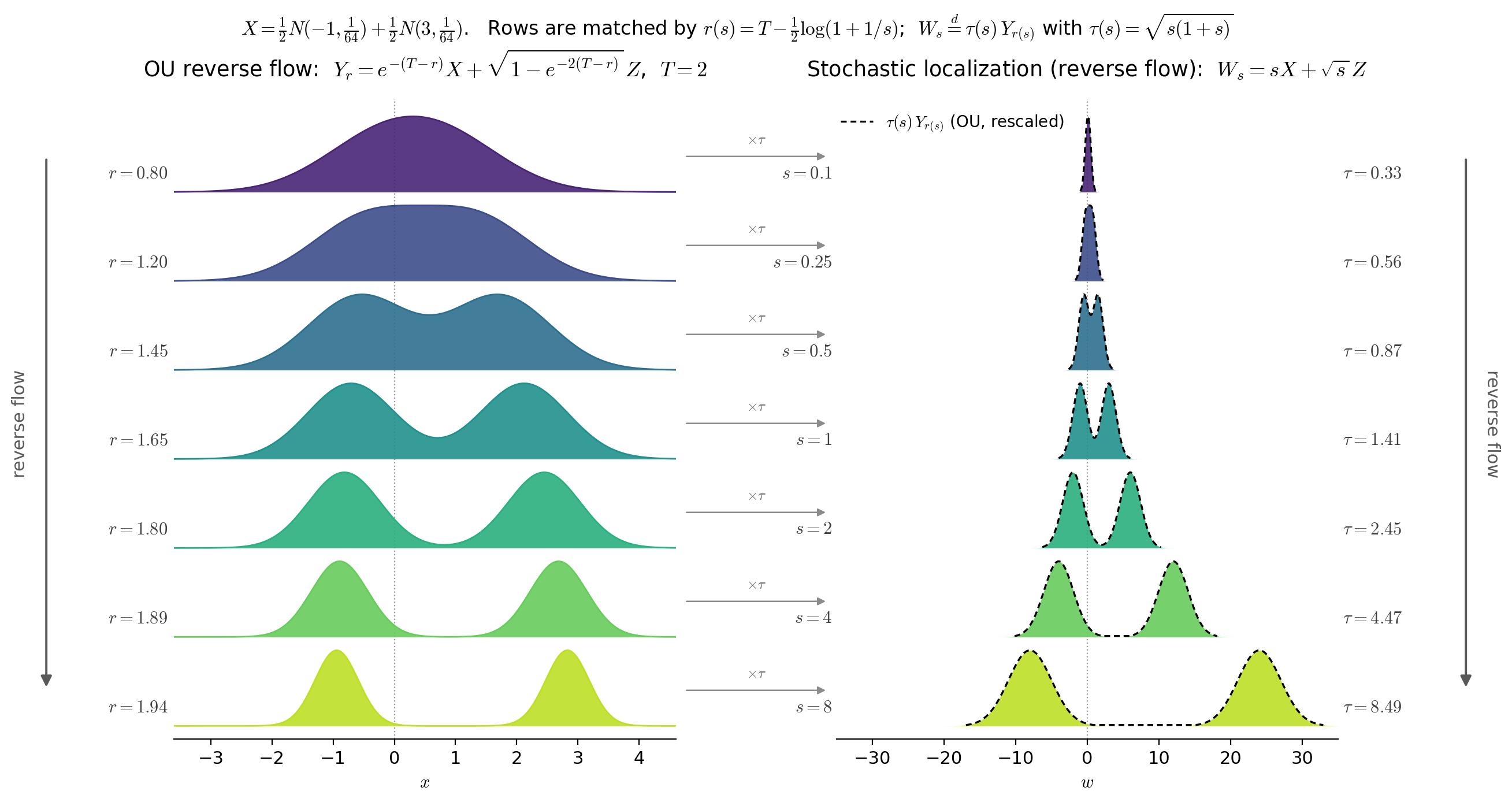}
\end{center}\caption{Visualization of score-based diffusion and stochastic localization reverse flow path for an asymmetric mixture of normals target density in one dimension. Marginal densities produced at each stage of the processes are equivalent under a time shift and rescaling.}
\label{fig:sl_ou}
\end{figure}

\subsection{Log-Odds Calculation}
\label{appendix:logodds}

All tilts used in the experiments are linear tilts of the extended base density by a vector supported on the reward coordinates. A linear tilt reweights $\pi$ between reward classes without altering the conditional distribution within a class, and the resulting class probabilities have a closed form in $\tau$. This provides an exact reference value at every tilt specifying the probability of producing a reward-achieving solution at that tilt strength.

\paragraph{Single reward bit.}
Let $r : \{-1,1\}^{d} \to \{-1,1\}$, $A = \{x : r(x) = 1\}$, and $p = \pi(A)$. The reward tilt is the linear tilt by $\tau e_{*}$ in the extended state space, and
\begin{align*}
q_{\tau}(x) = \frac{\pi(x)\,e^{\tau r(x)}}{E_{\pi}[e^{\tau r(X)}]}
 = \pi(x \mid A)\,\frac{p\,e^{\tau}}{p\,e^{\tau} + (1-p)\,e^{-\tau}}
 + \pi(x \mid A^{c})\,\frac{(1-p)\,e^{-\tau}}{p\,e^{\tau} + (1-p)\,e^{-\tau}} .
\end{align*}
The conditional laws $\pi(\cdot \mid A)$ and $\pi(\cdot \mid A^{c})$ are unchanged. With $\operatorname{logit} p = \log \frac{p}{1-p}$ and $p_{\tau} = q_{\tau}(A)$,
\begin{align}
\operatorname{logit} p_{\tau} = \operatorname{logit} p + 2\tau .
\label{eq:logodds_binary}
\end{align}
The slope $2$ is the difference $r(x) - r(x') = 2$ between the two reward values. In the Ising and Sudoku experiments $p$ is the fraction of reward-achieving examples in the training set.

\paragraph{One-hot class block.}
For $K$ classes the reward is a block of $K$ coordinates and class $c$ is encoded by $r_{c} \in \{-1,1\}^{K}$, $+1$ in position $c$ and $-1$ elsewhere. The tilt $h = \tau r_{c}$ assigns class $k$ the weight $\exp(\tau \langle r_{c}, r_{k} \rangle)$, with $\langle r_{c}, r_{c} \rangle = K$ and $\langle r_{c}, r_{k} \rangle = K - 4$ for $k \neq c$, since the two patterns differ in exactly two coordinates. Canceling the common factor $e^{(K-4)\tau}$,
\begin{align*}
\pi_{\tau}(c) = \frac{\pi(c)\,e^{4\tau}}{\pi(c)\,e^{4\tau} + 1 - \pi(c)},
\qquad
\pi_{\tau}(k) = \frac{\pi(k)}{\pi(c)\,e^{4\tau} + 1 - \pi(c)}, \quad k \neq c ,
\end{align*}
so that
\begin{align}
\operatorname{logit} \pi_{\tau}(c) = \operatorname{logit} \pi(c) + 4\tau ,
\label{eq:logodds_onehot}
\end{align}
and $\pi_{\tau}(k)/\pi_{\tau}(k') = \pi(k)/\pi(k')$ for $k, k' \neq c$. The slope $4$ is the gap $\langle r_{c}, r_{c} \rangle - \langle r_{c}, r_{k} \rangle$. Tilting along the single axis $e_{c}$ yields slope $2$ with the same proportionality among non-target classes. In general, a tilt $h$ supported on the reward block shifts the log-odds of class $c$ against class $k$ by $\langle h, r_{c} - r_{k} \rangle$; \eqref{eq:logodds_binary} and \eqref{eq:logodds_onehot} are the two cases used.

\paragraph{Tilt grid.}
Inverting \eqref{eq:logodds_binary}, the tilt strength at which reward is achieved with probability $p^{\star}$ is
\begin{align*}
\tau(p^{\star}) = \tfrac{1}{2}\big(\operatorname{logit} p^{\star} - \operatorname{logit} p\big),
\end{align*}
with $\tfrac{1}{4}$ in place of $\tfrac{1}{2}$ under \eqref{eq:logodds_onehot}. Experiments fix $\tau_{98} = \tau(0.98)$ ($\tau_{95}$ for MNIST, where the classifier ceiling limits resolution) and evaluate five tilt strengths below $\tau_{98}$ and five above, up to the maximum tilt strength $\sqrt{d+1}$.

\paragraph{Estimation.}
At each $\tau$, $N$ samples are drawn ($N = 2000$ unless stated), decoded, and graded by the reward function, or by the scoring classifier for MNIST. The observed fraction $\hat{p}_{\tau}$ is reported on the logit scale against \eqref{eq:logodds_binary} or \eqref{eq:logodds_onehot}, with $95\%$ Clopper--Pearson binomial intervals mapped through the logit. Near $p_{\tau} = 1$ the logit is steep: intervals narrow in probability are wide in log-odds and often include $\hat{p} = 1$ as an upper endpoint, and the largest log-odds resolvable at sample size $N$ is of order $\operatorname{logit}(1 - 1/N)$, approximately $7.6$ for $N = 2000$. Points whose exact value exceeds this are uninformative regarding under-sampling. Conversely, a discrepancy in log-odds at high $\tau$ may correspond to a difference of less than one percentage point in reward rate. For MNIST the exact rate is passed through the classifier's per-class sensitivity and false-positive rate prior to comparison, as described in Appendix \ref{appendix:experiments}.

\paragraph{Clopper--Pearson Confidence Intervals.}
Let $K \sim \operatorname{Binomial}(N, p_{\tau})$ be the number of rewarded samples at a given $\tau$, with observed value $k$ and $\hat{p}_{\tau} = k/N$. The Clopper--Pearson interval is obtained by inverting two one-sided exact binomial tests. Its endpoints are the values of $p$ at which the observed count would just be rejected at level $\alpha/2$ in each tail:
\begin{equation*}
    \sum_{j=k}^{N} \binom{N}{j} p_L^{\,j} (1 - p_L)^{N-j} = \frac{\alpha}{2},
    \qquad
    \sum_{j=0}^{k} \binom{N}{j} p_U^{\,j} (1 - p_U)^{N-j} = \frac{\alpha}{2}.
\end{equation*}
By the identity linking binomial tails to the regularized incomplete beta function, these endpoints have the closed form
\begin{equation*}
    p_L = B^{-1}\!\left(\tfrac{\alpha}{2};\, k,\, N - k + 1\right),
    \qquad
    p_U = B^{-1}\!\left(1 - \tfrac{\alpha}{2};\, k + 1,\, N - k\right),
\end{equation*}
where $B^{-1}(q; a, b)$ is the $q$-quantile of a $\operatorname{Beta}(a, b)$ distribution, with the conventions $p_L = 0$ when $k = 0$ and $p_U = 1$ when $k = N$. Because the interval comes from exact tail probabilities and not from a normal approximation, its coverage is at least $1 - \alpha$ for every $p_{\tau}$ and every $N$. This property is useful here, since many of our estimates lie close to the boundary $p_{\tau} = 1$, where Wald-type intervals are badly anti-conservative or collapse to zero width. Coverage is in general strictly greater than $1 - \alpha$, so the intervals are conservative. The logit is strictly increasing, so $[\operatorname{logit}(p_L), \operatorname{logit}(p_U)]$ is a valid $1 - \alpha$ interval for the log-odds. When every sample is rewarded ($k = N$), the interval becomes $[(\alpha/2)^{1/N}, 1]$. Its upper endpoint in log-odds is then $+\infty$, and its lower endpoint is $\operatorname{logit}\big((\alpha/2)^{1/N}\big) \approx 6.3$ for $N = 2000$ and $\alpha = 0.05$.

\paragraph{Model-reported reward.}
Since the reward coordinate is denoised jointly with the token coordinates, the sampler also produces an internal prediction of reward: the sign of the terminal $d+1$ coordinate, or the argmax of the class block. For a correctly sampling model this prediction and the externally graded reward agree in expectation. Comparing the two curves against the exact line separates two failure modes: a gap between the internal prediction and the exact line indicates the sampler does not produce the tilted marginal on the reward coordinate; a gap between the internal prediction and the graded reward indicates the token coordinates are inconsistent with the reward coordinate produced. The 9x9 Sudoku two-pass and carry results in Appendix \ref{appendix:experiments} are distinguished by this comparison.

\subsection{Self-Conditioning}\label{app:self-cond}

Self-conditioning as originally introduced in \citet{chen2023analog} is a method to improve the sample quality of a denoising diffusion model. It was originally derived as a mostly ad hoc engineering decision, rather than a method arising from sound theory. Its interaction with standard theory for error control for diffusion models has not yet been fully understood \citep{song2021maximum, lee2022convergence, chen2023sampling, chen2023improved, benton2024nearly, conforti2025kl, wang2024evaluating}. In this paper we consider two forms of self-conditioning: one method we call mode ``two-pass'' is a Markovian form of self-conditioning, it uses only information available at a given time point of a discrete diffusion to estimate the denoiser as in Algorithm 6 in \citet{chen2023analog}. This method can be considered an architecture choice for the denoiser estimator network and is compatible with existing diffusion model error control theory which views the reverse process as a Euler-Maruyama discretization of the continuous SDE with pointwise denoiser estimation.

The second method of self-conditioning we call mode ``carry'' is when the denoiser reuses its previous estimate at time $t-1$ to make its denoiser estimate at time $t$. This is the more common method implemented in modern denoising diffusion models, however it breaks the memory-one Markov structure under which diffusion model error control theory is currently understood.

The denoiser is a conditional expectation on the origin point of a noisy observation given partial information provided by the noised observation,
\begin{align*}
F(u,a) = E_{\pi}[X|U=aX+\sqrt{1-a^{2}}Z].
\end{align*}
If we had some additional side information $W$ from some other random variable, this could produce a more precise conditional estimate of the origin point $X$
\begin{align*}
F(u,w,a)&=E_{\pi}[X|U=aX+\sqrt{1-a^{2}}Z,W].
\end{align*}
A self-conditioned network takes the form $g_\theta(u,w,a)$, where $(u,a)$ are the noisy input and signal level and $w \in [-1,1]^{d+1}$ is a prior estimate of $x$. The input $w=0$ is equidistant from all points of $\{-1,1\}^{d+1}$ and encodes the absence of side information, so that $g_\theta(u,0,a)$ is an estimate of $G(u,a)$. In Algorithms \ref{alg:batch-sampler} and \ref{alg:sl-sampling}, $(u,a)$ is the renormalized pair $(v^*,s^*)$.

In the training algorithm, a self-conditioned network is trained by some portion of the time being fed $w = 0$ as its second argument, and some portion of the time begin fed its own no side information estimate as its second argument $w = g_{\theta}(u,0,a)$. This modifies the single step of batch SGD training algorithm in lines 13 and 14 in Algorithm \ref{alg:batch-sampler-sc}.

\begin{algorithm}[h]
\caption{Batch SGD step for the self-conditioned tilt--mean network $g_\theta(u,w,a)$}
\label{alg:batch-sampler-sc}
\begin{algorithmic}[1]
\Require Data $\{X_i\} \subset \{-1,1\}^{d+1}$ (reward in coordinate $d+1$); batch size $N$; step size $\gamma$; self-conditioning probability $p_{\mathrm{sc}}$
\State $\mathcal{L} \gets 0$
\For{$n = 1, \dots, N$}
    \State Draw $x \sim \{X_i\}$,\; $Z \sim \mathcal N(0, I_{d+1})$,\; $\beta \sim \mathrm{LogUnif}[\beta_{\min},\beta_{\max}]$,\; $a \gets a(\beta)$
    \State W.p.\ $p_{0}$: $q \gets \mathbf 1_d$;\; else draw $k \sim \mathrm{Unif}\{2,\dots,d\}$ and set $q \gets \sqrt{d/k}\,1_{S}$ for a random $k$-subset $S$
    \State W.p.\ $p_{\lambda}$: $\lambda \gets 1$;\; else $\lambda \sim \mathrm{LogUnif}[1,\sqrt{d+1}]$
    \State $\Lambda_{j,j} \gets \sqrt{1 + \tfrac1d - \tfrac{\lambda^2}{d}}\, q_j,\; j \le d$;\quad $\Lambda_{d+1,d+1} \gets \lambda$
    \State $\eta \gets a \Lambda^{2} x + \sqrt{1-a^2}\, \Lambda Z$ \Comment{tilt vector}
    \If{$q = \mathbf 1_d$ \textbf{and} $\lambda = 1$}
        \State $(v^*, s^*) \gets (\eta,\ a)$ \Comment{basic VP noising}
    \Else
        \State $C \gets \beta(a)\tfrac{\|\eta\|}{\sqrt{d+1}}$,\quad $s^* \gets \tfrac{1}{2}\Big(\sqrt{\tfrac{1}{C^{2}}+4}-\tfrac{1}{C}\Big)$,\quad $v^* \gets \sqrt{d+1}\;\eta/\|\eta\|$
    \EndIf
    \State W.p.\ $1-p_{\mathrm{sc}}$: $w \gets 0$ \Comment{no side information}
    \Statex \hspace{\algorithmicindent} else $w \gets \mathrm{sg}\big[g_\theta(v^*, 0, s^*)\big]$ \Comment{own first-pass estimate, no gradient}
    \State $\mathcal{L} \gets \mathcal{L} + \|x - g_\theta(v^*, w, s^*)\|^2$
\EndFor
\State $\theta \gets \theta - \gamma \nabla_\theta \mathcal{L}$
\end{algorithmic}
\end{algorithm}

Both self-conditioning modes ``two-pass'' and ``carry'' are trained in the same way, and differ in their implementation as a sampling algorithm.

Two-pass is the simpler method, as in sampling it reuses its own estimate with no side information as a second input to a second pass evaluation. This can be viewed as an architecture choice for a larger network that is ultimately still a function of a single pair $(u,a)$ alone,
\begin{align*}
\tilde g_\theta(u,a) = g_\theta\big(u,\, g_\theta(u,0,a),\, a\big),
\end{align*}
Using $\tilde{g}$ in place of $g$ in line 8 of Algorithm \ref{alg:sl-sampling} is the sampling algorithm implemented by two-pass.

The mode carry reuses the previous iteration's denoiser estimate as its second argument in self-conditioning. Initialize $\hat{G}_{-1} = 0$, and let $\hat{G}_{k}$ be the denoiser estimate at iteration $k$. The denoiser estimates are defined recursively as
\begin{align*}
\hat G_k &= g_\theta\big(v^*_k,\, \hat G_{k-1},\, s^*_k\big),\\
X_{t_{k+1}} &= X_{t_k} + \delta\, \hat G_k + \sqrt{\delta}\,\xi_k,
\end{align*}
This modifies the stochastic localization discrete sampling algorithm as shown in Algorithm \ref{alg:sl-sampling-carry}.

\subsubsection*{Implementation}
\textbf{Ising Model and 4x4 Sudoku}
Target distributions for the Ising Model and 4x4 Sudoku model have state spaces of size $2^{13}$ and $2^{33}$, small relative to the other target distributions. Self-conditioning was not implemented for these problems and the denoiser was trained using Algorithm \ref{alg:batch-sampler} and stochastic localization implemented using Algorithm \ref{alg:sl-sampling}.

\textbf{6x6 Sudoku}
For the 6x6 Sudoku self-conditioning mode ``two-pass'' was implemented and the sampler tracked performance well across log-odds levels considering the confidence intervals in Figure \ref{fig:6x6_logodds}.

\textbf{9x9 Sudoku}
The 9x9 Sudoku problem, both modes of self-conditioning were implemented and saw differing results.

Mode two-pass saturated its probability of producing proper puzzles at 60\% and did not increase at higher tilt strengths. The model also saw significant differences between the model's own prediction of if it was producing a reward in its $d+1$ coordinate and actual reward given to the token string produced by the reward function, seen in Figure \ref{fig:9x9_logodds_twopass}.

Mode carry tracked the log-odds better and saturated at 90\% reward-achieving solution, showing less mismatch between its prediction of reward and the actual reward of the token string produced, see Figure \ref{fig:9x9_logodds_carry}.

\textbf{MNIST Digits}
For the MNIST digit image generation problem, both modes of self-conditioning were implemented. In this problem mode two-pass follows the log-odds curve better than the carry mode. Mode carry over-samples the reward-achieving solutions at medium strength tilts as seen in Figure \ref{fig:mnist-logodds-carry-pooled}. As a conjecture of what may be the cause, carry self-conditioning may cause the model to become too confident in its early predictions, continuing to refine upon estimates made early in the process while two-pass is a more faithful implementation of a proper Euler-Maruyama discretization of the continuous-time SDE which defines stochastic localization.

\begin{algorithm}[h]
\caption{Stochastic-localization sampling from the tilted density $\mu \propto \pi\, e^{\langle \nu,\cdot\rangle}$ with carry-mode self-conditioning}
\label{alg:sl-sampling-carry}
\begin{algorithmic}[1]
\Require Trained self-conditioned network $g_\theta(u,w,a)$; tilt vector $\nu = \tau_{\rm infill}\bigl(e_{A^+} - e_{A^-}\bigr) + \tau_{\rm reward}\, e_*$; time grid $0 = t_0 < t_1 < \dots < t_K = T$
\State $X_{t_0} \gets 0$,\quad $\hat{G}_{-1} \gets 0$
\For{$k = 0, \dots, K-1$}
    \State $\delta \gets t_{k+1} - t_k$
    \State $\eta \gets \nu + X_{t_k}$
    \State $C\gets \tfrac{\|\eta\|}{\sqrt{d+1}}$,\quad $s^* \gets \tfrac{1}{2}\Big(\sqrt{\tfrac{1}{C^{2}}+4}-\tfrac{1}{C}\Big)$,\quad $v^* \gets \sqrt{d+1}\;\eta/\|\eta\|$
    \State $\hat{G}_{k} \gets g_\theta\big(v^*,\, \hat{G}_{k-1},\, s^*\big)$ \Comment{previous estimate as side information}
    \State Draw $\xi \sim \mathcal N(0, I_{d+1})$
    \State $X_{t_{k+1}} \gets X_{t_k} + \delta\, \hat{G}_{k} + \sqrt{\delta}\,\xi$
\EndFor
\State $\eta \gets \nu + X_{T}$, \quad compute $(v^*, s^*)$ from $\eta$ as above
\State \Return $\hat x \gets \mathrm{sign}\Big(g_\theta\big(v^*,\, g_\theta(v^*, 0, s^*),\, s^*\big)\Big)$ \Comment{two-pass final decode}
\end{algorithmic}
\end{algorithm}

\subsection{Discrete Step Size Scheduling}\label{appendix:schedule}

The stochastic localization discretization in Algorithm \ref{alg:sl-sampling} requires the sequence of discrete time steps $\{t_{k}\}_{k=0}^{K}$ to be specified as an input. Here we describe for an input time frame $[0,T]$, discretization level $K$, target tilt vector $\nu$, and dimension $d$ how we construct the $t_{k}$ schedule.

The learned denoiser $g_\theta(v^*, s^*)$ depends on the sampling time $t$ only through the equivalent signal level $s^*$ which is determined by the norm of the $\eta = X_{t}+\nu$ vector with the applied shift. Thus there is a non-linear relationship between how the $t_{k}$ scaling interacts with the $s^{*}$ values the denoiser will be evaluated at.

Recall that at time $t$ Algorithm \ref{alg:sl-sampling} queries the network at $\eta = \nu + X_t$ with
\begin{align*}
C = \frac{\|\eta\|}{\sqrt{d+1}}, \qquad s^* = \beta^{-1}(C) = \frac{1}{2}\Big(\sqrt{\tfrac{1}{C^{2}} + 4} - \tfrac{1}{C}\Big), \qquad \beta(s) = \frac{s}{1-s^{2}}.
\end{align*}
Since $\|\eta\|$ is random, we replace it by a deterministic proxy. Under the theoretical proper sampling marginal, $X_t = tX + \sqrt{t}Z$ with $\|X\|^2 = d+1$, we have $E\|X_t\|^2 = (d+1)(t^2 + t)$. Considering $E\|X_{t}+\nu\|^{2}$ and expanding the norm, the cross term $\sqrt{t}\langle \nu, Z\rangle$ has mean zero, and we evaluate the remaining cross term at its aligned value $\langle \nu, X \rangle = \|\nu\|$, which corresponds to the tilt succeeding. This gives
\begin{align*}
\bar{C}(t) = \sqrt{\frac{\|\nu\|^{2} + 2t\|\nu\| + (d+1)(t^{2}+t)}{d+1}}, \qquad \bar{s}(t) = \beta^{-1}\big(\bar{C}(t)\big),
\end{align*}
this is an approximation but gives a deterministic function of $t$ for the anticipated ratio of the norm of $\eta$ to $\sqrt{d+1}$ for the purposes of designing a time schedule. $\bar{s}$ is strictly increasing in $t$. Setting $s_k = \bar{s}(0) + \tfrac{k}{K}\big(\bar{s}(T) - \bar{s}(0)\big)$ for $k = 0, \dots, K$, the grid is $t_k = \bar{s}^{-1}(s_k)$. The inverse is available in closed form, since $\bar{C}(t)^2$ is quadratic in $t$: with $C_k = \beta(s_k)$,
\begin{align*}
t_k = \frac{-(d+1+2\|\nu\|) + \sqrt{(d+1+2\|\nu\|)^{2} - 4(d+1)\big(\|\nu\|^{2} - (d+1)C_k^{2}\big)}}{2(d+1)},
\end{align*}
with the endpoints fixed at $t_0 = 0$ and $t_K = T$. The grid is deterministic, is computed once before sampling, and depends on the tilt only through $\|\nu\|$; for $\nu = 0$ it starts from $\bar{s}(0) = 0$. The resulting steps are finest near $t = 0$ and grow monotonically towards $T$. All experiments use this schedule with $T = 20$ and $K = 1000$.

\begin{algorithm}[h]
\caption{Time grid equally spaced in the queried signal level $s^*$}
\label{alg:sstar-grid}
\begin{algorithmic}[1]
\Require Horizon $T$; number of steps $K$; dimension $d$; tilt norm $n = \|\nu\|$ ($n = 0$ if untilted)
\State $D \gets d+1$
\State Define $\bar{C}(t) = \sqrt{\big(n^{2} + 2tn + D(t^{2}+t)\big)/D}$ \Comment{proxy for $\|\nu + X_t\|/\sqrt{d+1}$}
\State Define $\beta^{-1}(C) = \tfrac{1}{2}\Big(\sqrt{\tfrac{1}{C^{2}}+4} - \tfrac{1}{C}\Big)$ for $C > 0$, and $\beta^{-1}(0) = 0$
\State $s_{\mathrm{lo}} \gets \beta^{-1}\big(\bar{C}(0)\big)$,\quad $s_{\mathrm{hi}} \gets \beta^{-1}\big(\bar{C}(T)\big)$
\For{$k = 0, \dots, K$}
    \State $s_k \gets s_{\mathrm{lo}} + \tfrac{k}{K}\,(s_{\mathrm{hi}} - s_{\mathrm{lo}})$ \Comment{uniform in $s^*$}
    \State $C_k \gets s_k / (1 - s_k^{2})$ \Comment{$C_k = \beta(s_k)$}
    \State $t_k \gets \dfrac{-(D + 2n) + \sqrt{(D+2n)^{2} - 4D\,(n^{2} - D\,C_k^{2})}}{2D}$ \Comment{solves $\bar{C}(t_k) = C_k$}
\EndFor
\State $t_0 \gets 0$,\quad $t_K \gets T$ \Comment{fix endpoints against rounding}
\State \Return $(t_0, t_1, \dots, t_K)$
\end{algorithmic}
\end{algorithm}

\FloatBarrier 
\section{ADDITIONAL EXPERIMENTS INFORMATION} \label{appendix:experiments}

Code for all experiments, including data generation, training, and evaluation is available at the the GitHub link:\\ \url{https://github.com/CMcDonald-1/internal_guidance_diffusion_models}

\noindent The datasets and trained checkpoints can be downloaded from the link:\\ \url{https://osf.io/kewg9/}

\FloatBarrier 
\subsection{12-Dimensional Ising Model}

\textbf{Target Density:}

The target density is on $\{-1,1\}^{d}$ with $d=12$, with pmf proportional to
\begin{align*}
\pi(x) \propto e^{\frac{1}{2}x^{T}Jx+ \langle s, x \rangle}.
\end{align*}

$J$ is a Wigner (GOE-like) coupling matrix, generated by drawing a full
$d\times d$ matrix of i.i.d.\ standard normals scaled by $J_{0}/\sqrt{d}$ with
$J_{0}=1.5$, retaining the strict upper triangle, and mirroring it onto the lower
triangle so that $J=J^{T}$. The diagonal is set to zero. The $s$
vector has i.i.d.\ Gaussian entries with standard deviation $\sigma = 0.3$.

The reward is the non-linear indicator function
\begin{align*}
r(x)&= \begin{cases}
1& \text{exactly 10 of the 12 coordinates positive}\\
-1& \text{else}
\end{cases}
\end{align*}
equivalently $r(x)=+1$ iff $\langle \mathbf{1},x\rangle = 8$. We append the
reward as a 13th coordinate in the extended state space $\{-1,1\}^{13}$, which
turns the non-linear reward tilt $q_{\tau}\propto \pi\, e^{\tau r(x)}$ into the
linear tilt $q_{\tau}\propto \pi(x,x^{*})e^{\tau x^{*}}$.

\textbf{Training Set:}

The training set is 100,000 samples from the base distribution with reward information. $4.41\%$ achieve reward, which is the base probability of reward under the target distribution.

\textbf{Model Architecture:}

The operator network is a residual MLP with a mean-field skip connection,
\begin{align*}
G_{\theta}(u,a) = \tanh\!\big(\beta(a)\,u + g_{\theta}(a)\odot r_{\theta}(u,a)\big),
\end{align*}
where $\beta(a)u$ is the exact operator for a uniform base distribution, so the
residual stack only has to learn the correlation correction, and the $\tanh$
enforces the $[-1,1]^{13}$ box by construction. The signal level enters only
through a 10-dimensional level embedding
$e(a)=[\,a,\ \log(1+\beta(a)),\ \sin(2\pi a f),\ \cos(2\pi a f)\,]$ with Fourier
frequencies $f\in\{1,2,4,8\}$. The input $u\in\mathbb{R}^{13}$ is scaled by
$1/\sqrt{d}$, concatenated with $e(a)$, and mapped by a linear layer to width
96. This is followed by 4 pre-activation residual blocks, each consisting of
LayerNorm followed by FiLM modulation $(1+g_k)\,\mathrm{LN}(h)+b_k$ with
$(g_k,b_k)$ predicted from $e(a)$, then two width-96 linear layers with SiLU
activations. A final LayerNorm and linear layer project to the 13 output
coordinates, and a separate gain head ($10\to32\to13$, SiLU) predicts the
per-coordinate, level-dependent gain $g_{\theta}(a)$. The output layer is
zero-initialized, so training begins exactly at the mean-field prediction. The
model has $88{,}254$ parameters. Full implementation details are in the supplementary code.

\textbf{Training Run:}

The model was trained for 30,000 Adam steps with batch size 1024, learning rate
$10^{-3}$, and a cosine-annealing learning rate schedule over the full 30,000
steps. Each batch re-noises freshly drawn training samples rather than using a
fixed noise realization: $\beta$ is drawn log-uniformly on $[0.02, 50]$, and the
$\Lambda$-augmentation is applied with $\lambda\sim\mathrm{Unif}(1,3.5)$. The $\Lambda$ profile is applied to every training row, with $\lambda$ drawn
independently per row; no infill-focused ($k$-sparse) profiles were used in
this run. The objective is the unweighted
per-coordinate squared error $\mathbb{E}\|X - G_{\theta}(V,a)\|^{2}$ summed over
the 13 coordinates.

Training ran on a single NVIDIA GeForce RTX 5070 Ti Laptop GPU (Blackwell, 5888
CUDA cores, 12.8\,GB, CUDA 12.9) and took 10.7 minutes of wall-clock time
(0.18 GPU-hours). Processing $30{,}000 \times 1024 \approx 3.07\times10^{7}$
samples through an $88{,}254$-parameter network corresponds to approximately
$1.6\times10^{13}$ training FLOPs under the standard $6ND$ estimate.

\textbf{Extended Results}

For all models in question, $\tau$ tilt strengths tested are capped at $\sqrt{d+1}$, as this is the maximum amount of $\Lambda$ focusing seen in the training runs. Note this tilt strength is still very strong producing near 100\% probability of reward for the problems analyzed. In Figure \ref{fig:ising_cov}, we note the untilted model matches the mean and covariance of the target Ising model, effectively recreating the first and second moments.

\begin{figure}
\begin{center}
\includegraphics[width = \textwidth]{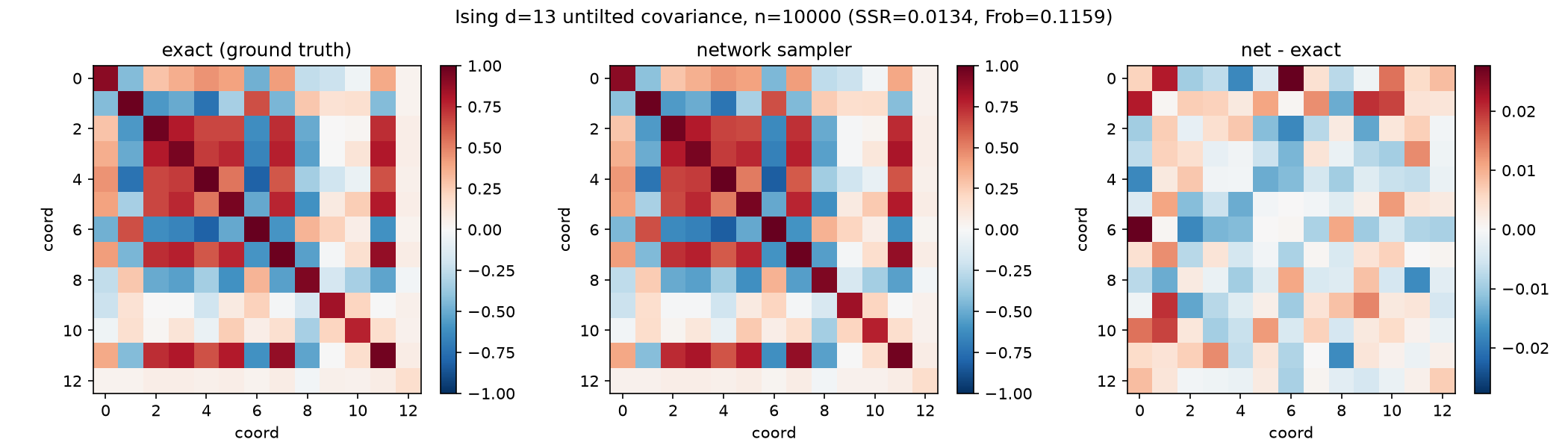}
\end{center}
\caption{Ising model covariance matrix ground truth and network sampled.}
\label{fig:ising_cov}
\end{figure}

For various signal levels $a$ and various tilt strengths $\tau$, 100 samples are generated for a pair of $(u,a)$ and average squared Euclidean distance between network denoiser and ground truth denoiser is computed in Figure \ref{fig:ising_ood}. Error does increase with tilt strength, as evaluations of tilted densities are out-of-distribution evaluations of the base denoiser only learned via $\Lambda$ focusing. However, there is still a uniform bound across tilt strengths and signal levels of $0.12$ which can be converted into a bound on terminal KL error at the final sample point across all $\tau$ strengths in the tested range of $0$ to $3.6$.

\begin{figure}
\begin{center}
\includegraphics[scale=0.7]{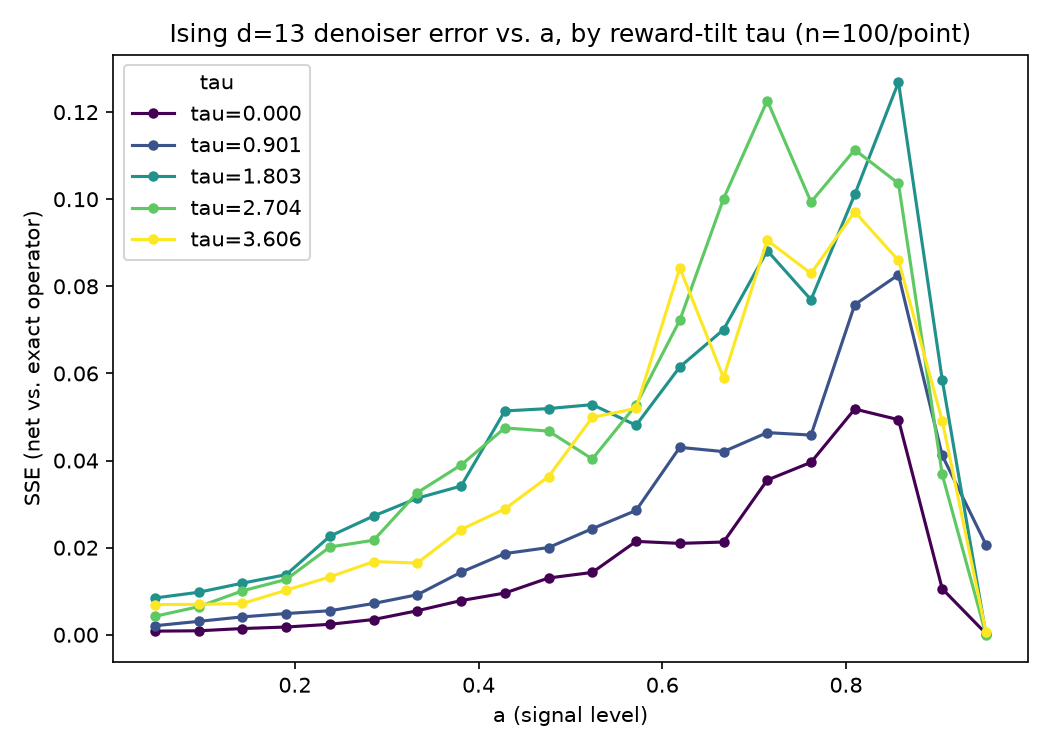}
\end{center}\caption{Comparison of $L^{2}$ error of learned denoiser for Ising model across tilt strengths.}\label{fig:ising_ood}
\end{figure}

The probability of reward and log-odds plot is seen in Figure \ref{fig:ising_logodds} across a range of $\tau$ from 0 to $\sqrt{d+1}$. At 2000 samples per $\tau$ the model tracks the log-odds perfectly.

The 12-dimensional Ising model is a small toy model with exact enumeration of the ground truth denoiser to demonstrate highly accurate sampling across tilts and denoiser learning when a ground truth is available.

\begin{figure}
\begin{center}
\includegraphics[width = \textwidth]{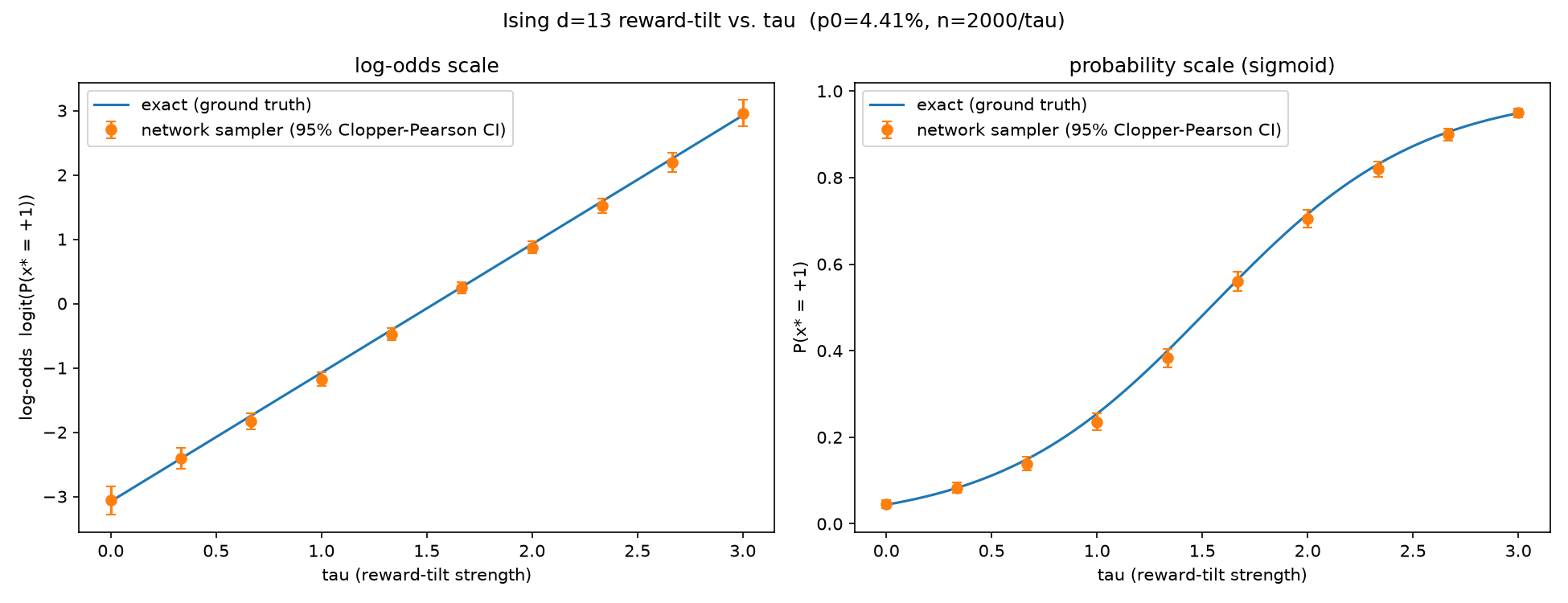}
\end{center}
\caption{Ising model log-odds across tilt strengths.}
\label{fig:ising_logodds}
\end{figure}

\FloatBarrier 

\FloatBarrier 
\subsection{4x4 Sudoku}

\textbf{Target Density}

$4\times4$ Sudoku has a state space of size $2^{32}$, with only 288 valid puzzles which satisfy row, column and box constraints. Each cell is a $2$-bit
encoding of the numbers $1$--$4$ with no dead code words, so the 16 cells give
exactly 32 bits and every bit string decodes to a grid.

\begin{figure}[H]
\begin{center}
\includegraphics[scale=0.5]{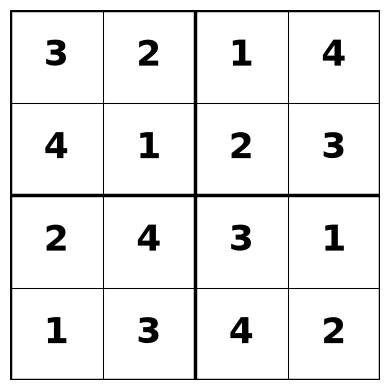}
\end{center}
\caption{Example valid 4x4 grid.}
\end{figure}

The base distribution is a weak generator sampling each row independently. With
probability $1-\rho$ it makes the row a uniformly random permutation of the
numbers $1$--$4$, and thus a valid row; with probability $\rho$ it instead
draws the row's 8 bits uniformly at random, giving any of the 256 possible
rows. We use $\rho = 0.02$.

A single row is therefore valid with probability
$(1-\rho) + \rho\cdot\frac{24}{256} = 98.19\%$, and all four rows are
simultaneously valid $92.94\%$ of the time. The generator has no knowledge of
column or box constraints, however, so it produces a proper puzzle only
$0.081\%$ of the time by chance.

The reward is the indicator of a valid puzzle appended as a 33rd coordinate, giving the extended state space
$\{-1,1\}^{33}$.

\textbf{Training Set}

Because a proper puzzle arises by chance only $0.081\%$ of the time, the base distribution is mixed with an injection component: each sample is drawn from the
uniform distribution over the 288 valid solutions with probability
$p_{\text{inject}} = 0.05$, and from the weak generator otherwise. The training
set consists of 150,000 samples drawn from this mixture, of which in
expectation 142,500 come from the weak generator and 7,500 are injected proper
Sudoku puzzles. The resulting exact reward probability is $5.08\%$; the
realized training file contains 7,590 genuinely valid grids ($5.06\%$).

\textbf{Model Architecture}

The operator network is the same residual MLP with mean-field skip connection
used for the Ising experiment,
\begin{align*}
G_{\theta}(u,a) = \tanh\!\big(\beta(a)\,u + g_{\theta}(a)\odot r_{\theta}(u,a)\big),
\end{align*}
scaled up for the larger state space. The signal level enters through a
10-dimensional level embedding
$e(a)=[\,a,\ \log(1+\beta(a)),\ \sin(2\pi a f),\ \cos(2\pi a f)\,]$ with Fourier
frequencies $f\in\{1,2,4,8\}$. The input $u\in\mathbb{R}^{33}$ is scaled by
$1/\sqrt{d}$, concatenated with $e(a)$, and mapped by a linear layer to width
1024, followed by 8 pre-activation residual blocks, each consisting of
LayerNorm, FiLM modulation $(1+g_k)\,\mathrm{LN}(h)+b_k$ with $(g_k,b_k)$
predicted from $e(a)$, and two width-1024 linear layers with SiLU activations.
A final LayerNorm and linear layer project to the 33 output coordinates, and a
separate gain head ($10\to32\to33$, SiLU) predicts the per-coordinate gain
$g_{\theta}(a)$. The output layer is zero-initialized so training begins exactly
at the mean-field prediction. The model has $17{,}072{,}582$ parameters. Full implementation details are in the supplementary code.

\textbf{Training Run}

The model was trained for 30,000 Adam steps with batch size 8192, learning rate
$10^{-3}$, and a cosine-annealing learning rate schedule over the full 30,000
steps. Each batch re-noises freshly drawn training samples, with $\beta$ drawn
log-uniformly on $[0.02, 50]$. Two focusing profiles are mixed per row: with
probability $0.25$ the row uses the sparse infill profile, in which $k$ is drawn
uniformly on $\{2,\dots,32\}$ and a random $k$-subset of the core coordinates is
set to $\sqrt{32/k}$ with the rest zeroed, and otherwise it uses the
reward-focusing profile. Independently, the reward-axis stretch $\lambda$ is set
to 1 with probability $0.25$ and otherwise drawn log-uniformly on $[1,5]$. Both
profiles preserve the budget identity $\sum_i \Lambda_{ii}^{2} = d$, so the
operator being learned is unchanged. The objective is the unweighted
per-coordinate squared error $\mathbb{E}\|X - G_{\theta}(V,a)\|^{2}$ summed over
the 33 coordinates.

Training ran on a single NVIDIA GeForce RTX 5070 Ti Laptop GPU (Blackwell, 5888
CUDA cores, 12.8\,GB, CUDA 12.9) and took 67.3 minutes of wall-clock time
(1.12 GPU-hours). Processing $30{,}000 \times 8192 \approx 2.46\times10^{8}$
samples through a $17{,}072{,}582$-parameter network corresponds to
approximately $2.5\times10^{16}$ training FLOPs under the standard $6ND$
estimate.

\textbf{Extended Results}

For the 4x4 Sudoku problem, the reward set is only 288 valid grids, and the weak generator decomposes into a product measure across rows. As such, the ground truth denoiser can be computed via enumeration at any reward tilt strength. The estimated $L^{2}$ error across tilt strength is shown in Figure \ref{fig:4x4_ood}. The denoiser error does increase with $\tau$ tilt strength, but does have a uniform bound across tilt and signal levels of $0.5$.

We show the actual 33 values of the base denoiser $E[X_{i}|u,a]$ and learned network for 5 sampled $u$ values from a $\tau= 0$ draw and a $\tau = 4$ draw in Figure \ref{fig:4x4_lollipop}. Expected values are all between $-1$ and $1$. Note the $\tau = 0$ draws are all noisy versions of non-reward-achieving solutions, and the 33rd coordinate (the reward coordinate) accurately assigns them all a -1. The 5 $\tau = 4$ draws are noisy reward-achieving solutions and the 33rd coordinate assigns them all a +1 value.

The untilted $\tau = 0$ network achieves near 0 error on these examples and is visually indistinguishable. The $\tau = 4$ network is still very good, but is noticeably a bit off on a few coordinates, which agrees with the average increase in squared error seen in Figure \ref{fig:4x4_ood}.

\begin{figure}
\begin{center}
\includegraphics[scale=0.45]{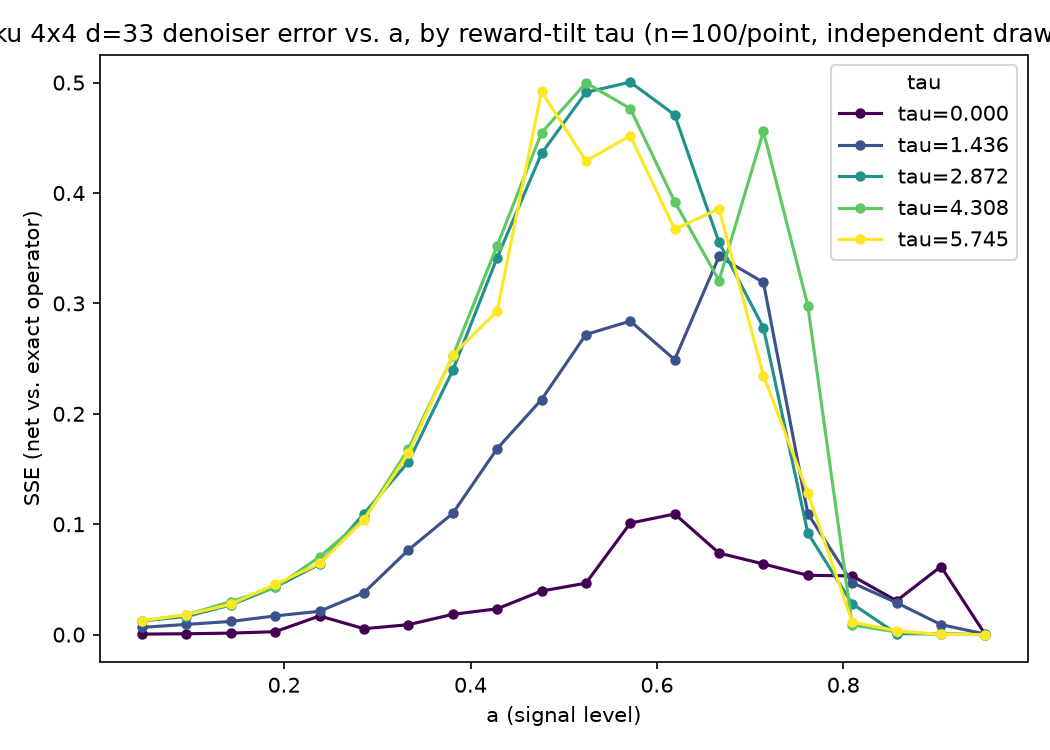}
\end{center}
\caption{Comparison of $L^{2}$ error of learned denoiser of 4x4 Sudoku problem across tilt strengths.}\label{fig:4x4_ood}
\end{figure}

For the 4x4 Sudoku problem we have access to the ground truth denoiser and the net estimated denoiser. Thus we can run stochastic localization with exact denoiser evaluations and with network approximated denoiser evaluation and track the error introduced by incorrect denoiser evaluations isolated from discretization error.

To track the log-odds in Figure \ref{fig:4x4_logodds}, we compute the $\tau$ that would achieve 98\% reward if sampled correctly. We then look at 5 $\tau$ levels before this value and 5 levels after up to maximum tilt value strength $\sqrt{d+1} = 5.74$. Note all levels past the dashed line should be achieving reward more than 98\% of the time. Confidence intervals are computed using a Clopper-Pearson method which often includes 100\% reward as an upper bound.

The localization samples using the ground truth denoiser tracks the log-odds across all levels, even past 98\% when considering the confidence intervals from 2000 samples. The network based sampler tracks well up to the 98\% level, but then does strictly under-sample past this high bar with the confidence intervals not overlapping the theoretical log-odds. Note when the network based sampler fails to exactly track the log-odds at this high tilt it is still generating over 98\% reward-achieving puzzles, but still under-sampling its theoretical value at this tilt.

\begin{figure}
\begin{center}
\includegraphics[width = \textwidth]{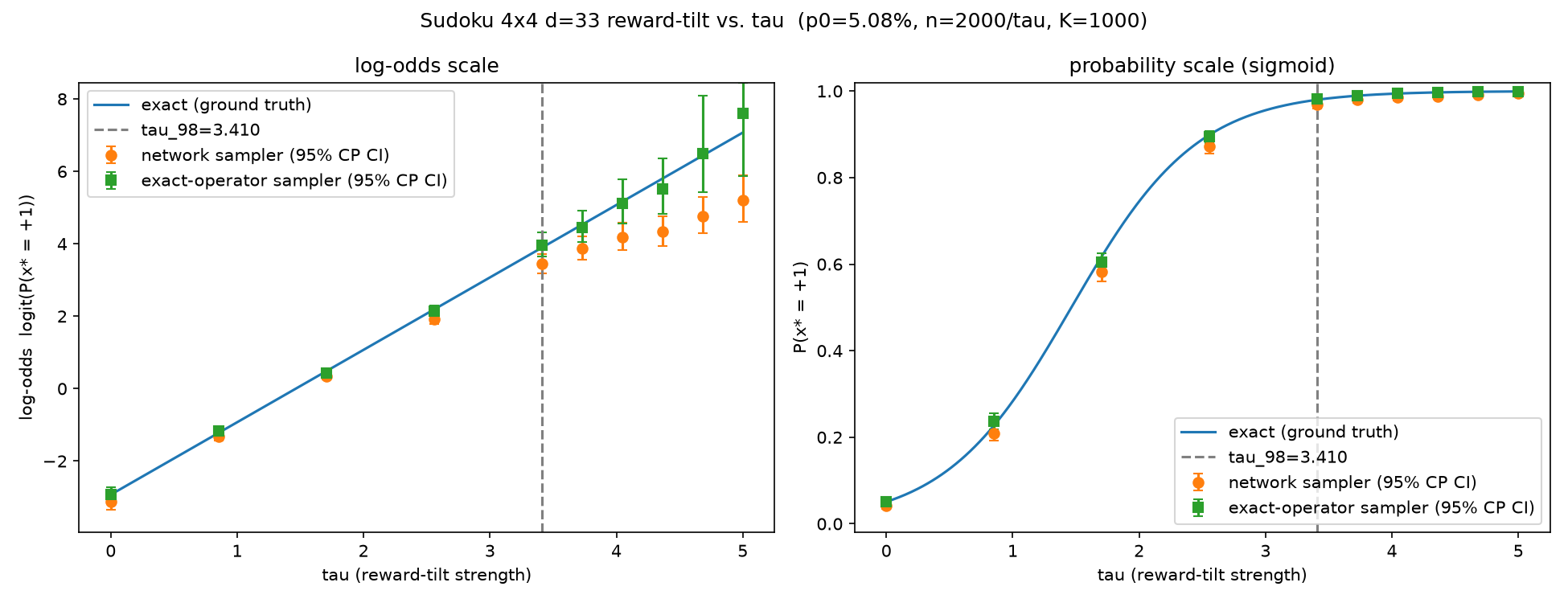}
\end{center}
\caption{4x4 Sudoku log-odds across tilt strengths.}
\label{fig:4x4_logodds}
\end{figure}

As an added bonus to the stochastic localization path, the denoiser $E[X_{i}|X_{t}]$ acts as both the drift term and a prediction of the terminal sign of a given coordinate. As such, if a coordinate ever passes a threshold $|E[X_{i}|X_{t}]|>\ell$, we can consider that bit ``confident'' and fix that bit to the sign of the denoiser at that moment. This allows for ``early decoding'' seeing which bits the network fixes in which order and at what point of the trajectory. Note a bit is allowed to flip signs, all that is decoded is the terminal sign value. Fixed bits determine what tokens (cell values) are still valid among the unfixed bits.

One of the arguments against continuous diffusion for discrete targets is that the trajectory remains a continuous value that cannot be interpreted in discrete token space until the terminal decoding step, unlike masked diffusion or autoregressive models which build partial discrete token strings along their trajectory.

By using the confidence based bit decoding, we can achieve a similar effect to see what discrete set of tokens the puzzle was considering at each moment in the trajectory and soft decode at any time.

When both bits in a cell are fixed we record the cell value, when one bit is fixed such as $(*,-1)$ this means the cell will get either a $(1,-1$) which is token 2 or $(-1-1)$ which is token 1. We record this with Sudoku ``pencil marks'' showing what is considered still possible in the cell. 

We note the network is not ``solving'' the puzzle ``logically'' as a human would by applying row, column, and box constraints in sequence. For example, in the second frozen grid in Figure \ref{fig:4x4_fuzzy}, there is a 1 fixed in the bottom left box. This would rule out all 1's in that column, but the upper box cells in the same column have not yet ruled out a 1 in that column since the bits for those two cells have not yet reached their own confidence level. This is therefore a sequence of most likely puzzles tracking the most confident bits at each moment in the trajectory.

We call this path a ``fuzzy trajectory'' tracking the denoiser coordinate values across the signal levels, noting when a coordinate achieves a confidence level more than 0.75 in magnitude, and recording events when a cell gets fixed in Figure \ref{fig:4x4_fuzzy}. Note some cells do become confident but ultimately flip their token value by the end of the run. Note the row 2 column 4 cell flips from a 2 to a 4 by the end of sampling. We run our stochastic localization up to a $T$ value of $20$ but the network has mostly decided on its final value much earlier by time $T = 3$ it would seem.

\begin{figure}
\begin{center}
\includegraphics[width = \textwidth]{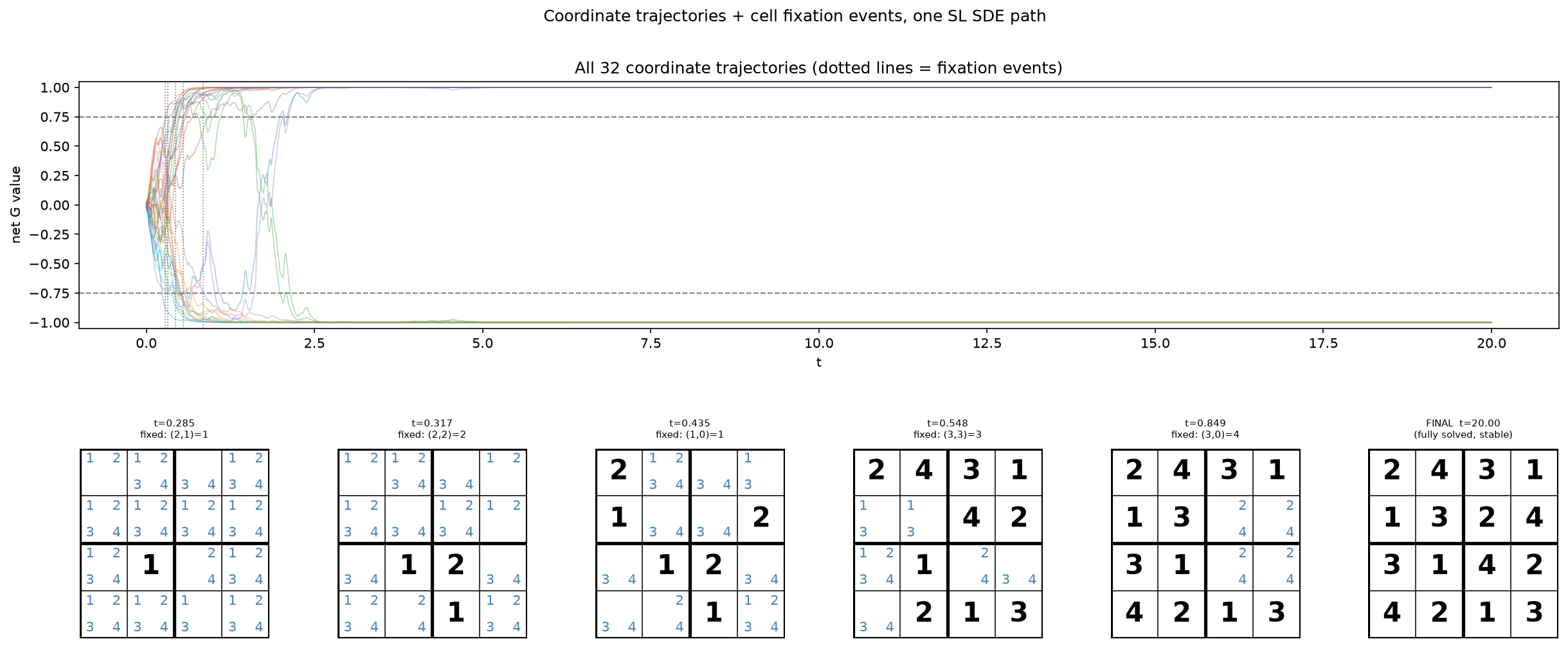}
\end{center}
\caption{Sudoku 4x4 denoiser trajectory and early decoding. Partially complete discrete token sequences fixed when denoiser coordinates pass magnitude threshold.}\label{fig:4x4_fuzzy}
\end{figure}

\begin{figure}[t]
    \centering
    \begin{subfigure}[t]{0.49\textwidth}
        \centering
        \includegraphics[width=0.9\linewidth]{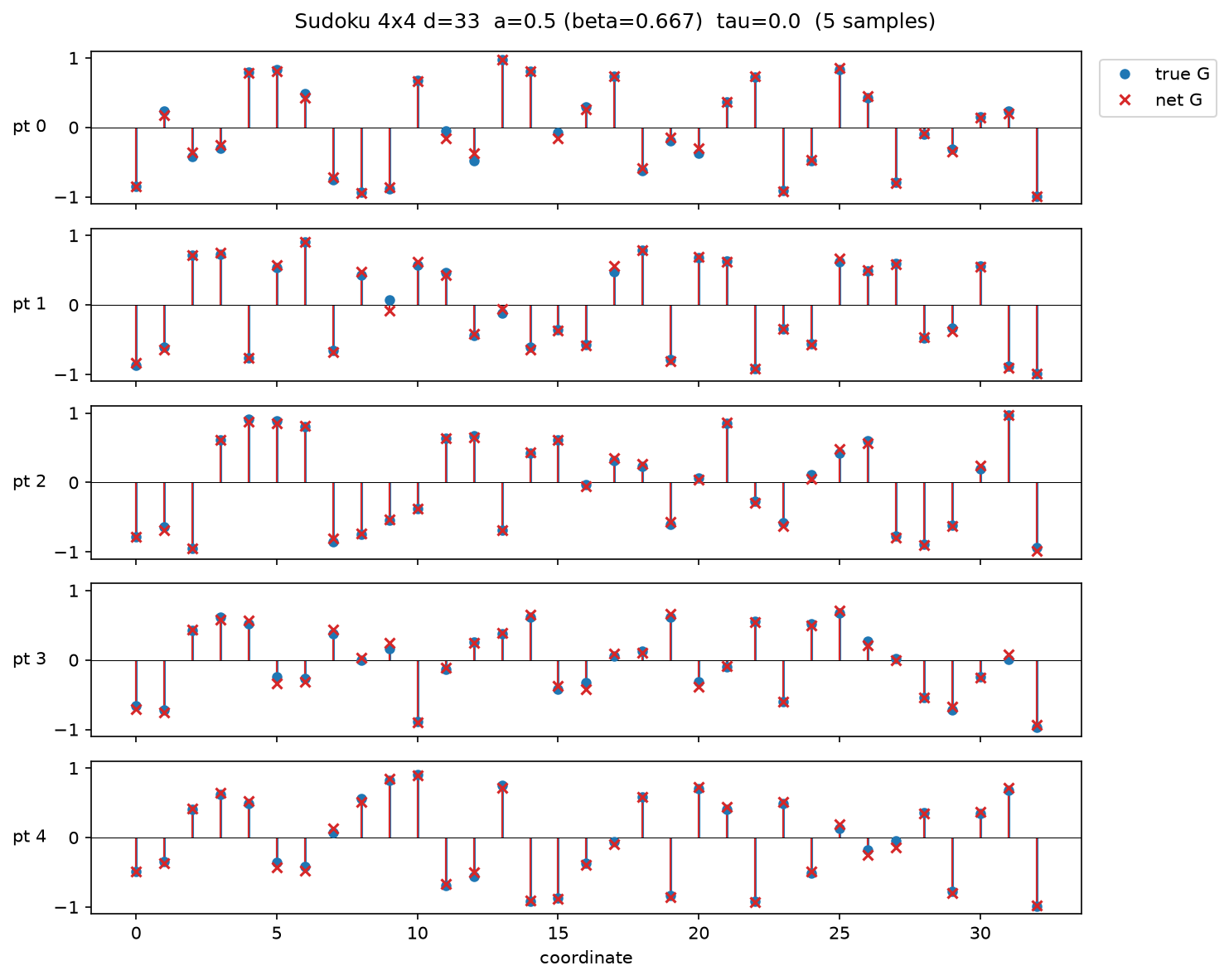}
        \caption{}
        \label{}
    \end{subfigure}
    \hfill
    \begin{subfigure}[t]{0.49\textwidth}
        \centering
        \includegraphics[width=0.9\linewidth]{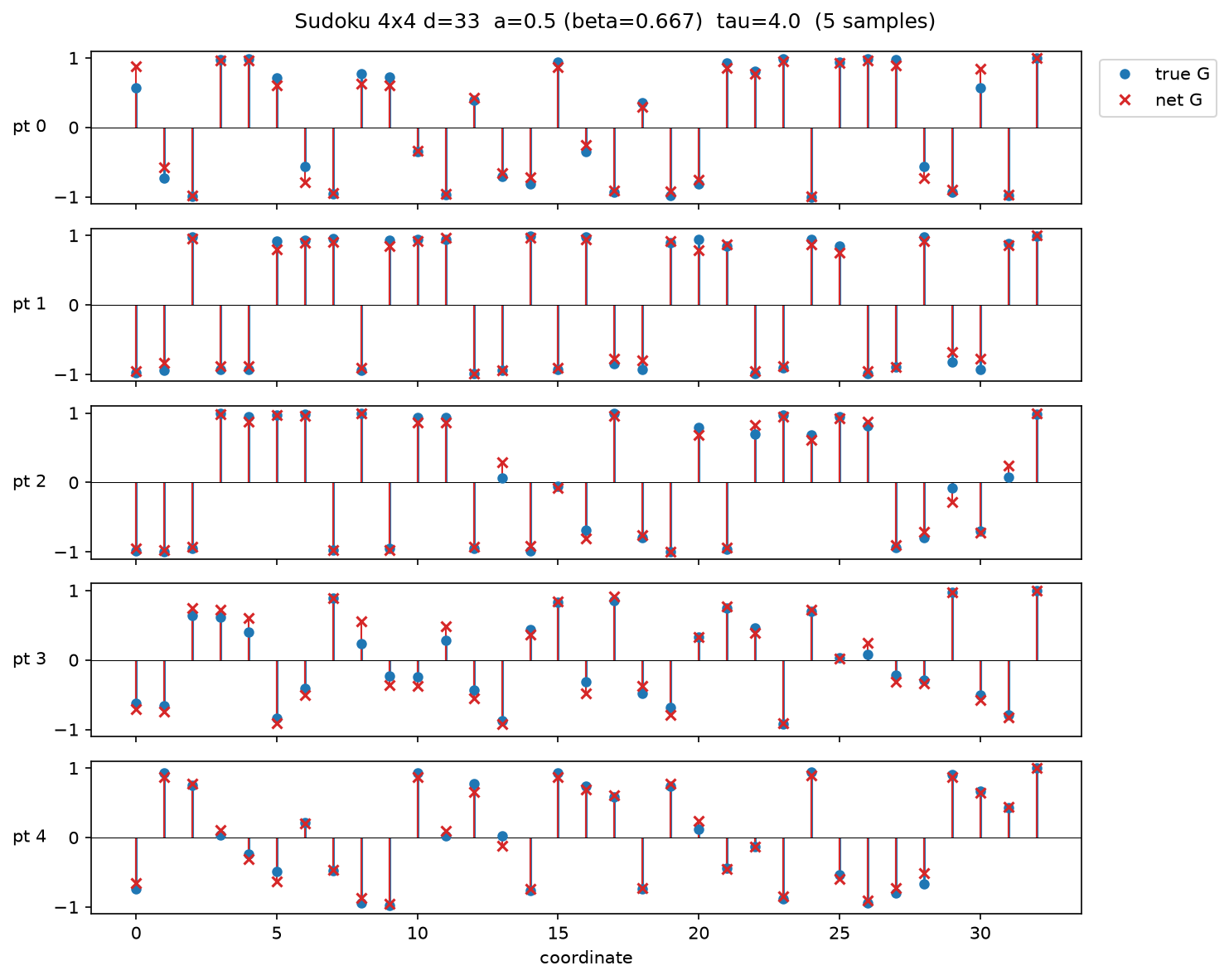}
        \caption{}
        \label{}
    \end{subfigure}
    \caption{Example ground truth denoiser and network denoiser in 33 coordinates at sampled points. Samples in a) are untilted sampled and b) tilted samples.}
    \label{fig:4x4_lollipop}
\end{figure}

Note there are 288 valid 4x4 Sudoku grids, and the training set contained 7,500 examples so each valid puzzle was seen many times in the training run. therefore, the model should be judged on its ability to generate each puzzle with equal probability in a large collection of generated outputs.

We generate 3000 valid puzzles with the stochastic localization path using the network denoiser and the ground truth and perform a chi-squared test for uniformity. The network based sampler visits all puzzles except for 1, and its maximum visit to one particular puzzle is 25. The ground truth denoiser visits every solution, with minimum visits being 4 and maximum being 21. Visually, they are not that distinct in their visit histogram in Figure \ref{fig:visit_hist}. However, the ground truth sampler passes a chi-squared test for uniformity, while the network based sampler fails the test. As such, it has learned some preference for certain solutions, although it is not a strong preference.

\begin{figure}
    \centering
    \begin{subfigure}[t]{\textwidth}
        \centering
        \includegraphics[width=\linewidth]{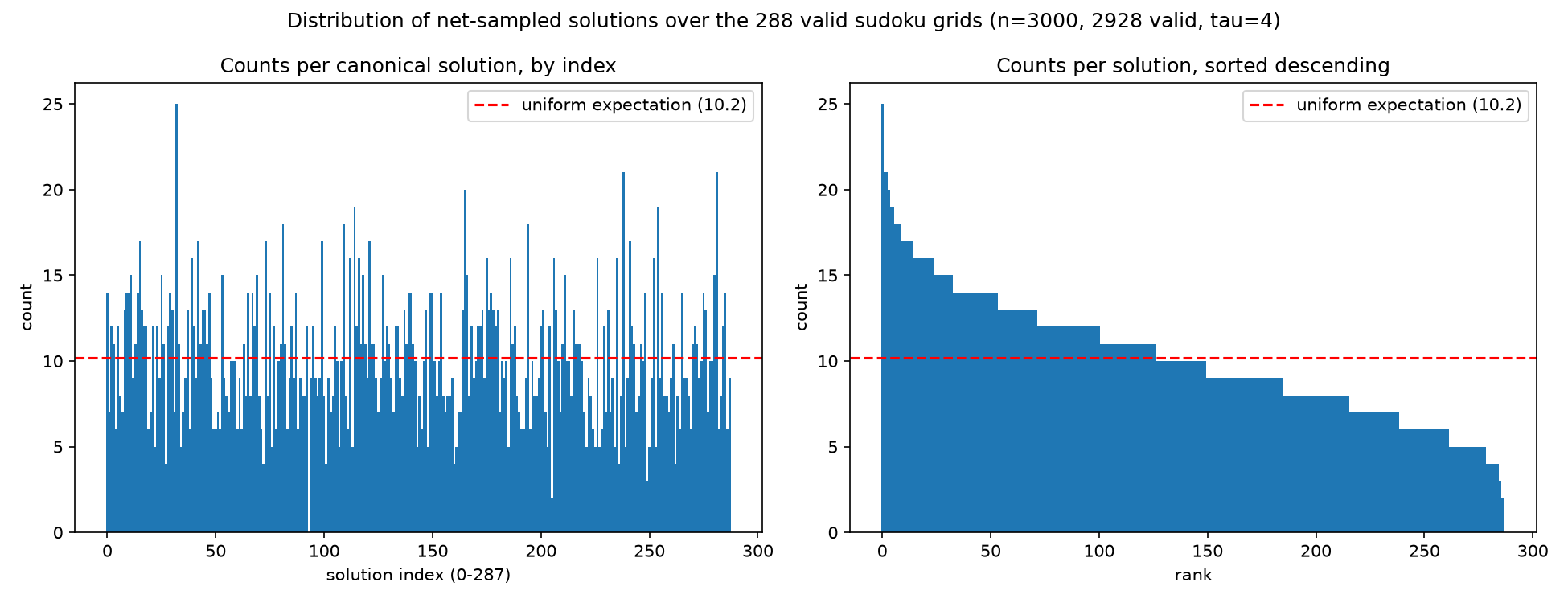}
        \caption{4x4 Sudoku chi-squared test network samples.}
        \label{}
    \end{subfigure}
    \hfill
    \begin{subfigure}[t]{\textwidth}
        \centering
        \includegraphics[width=\linewidth]{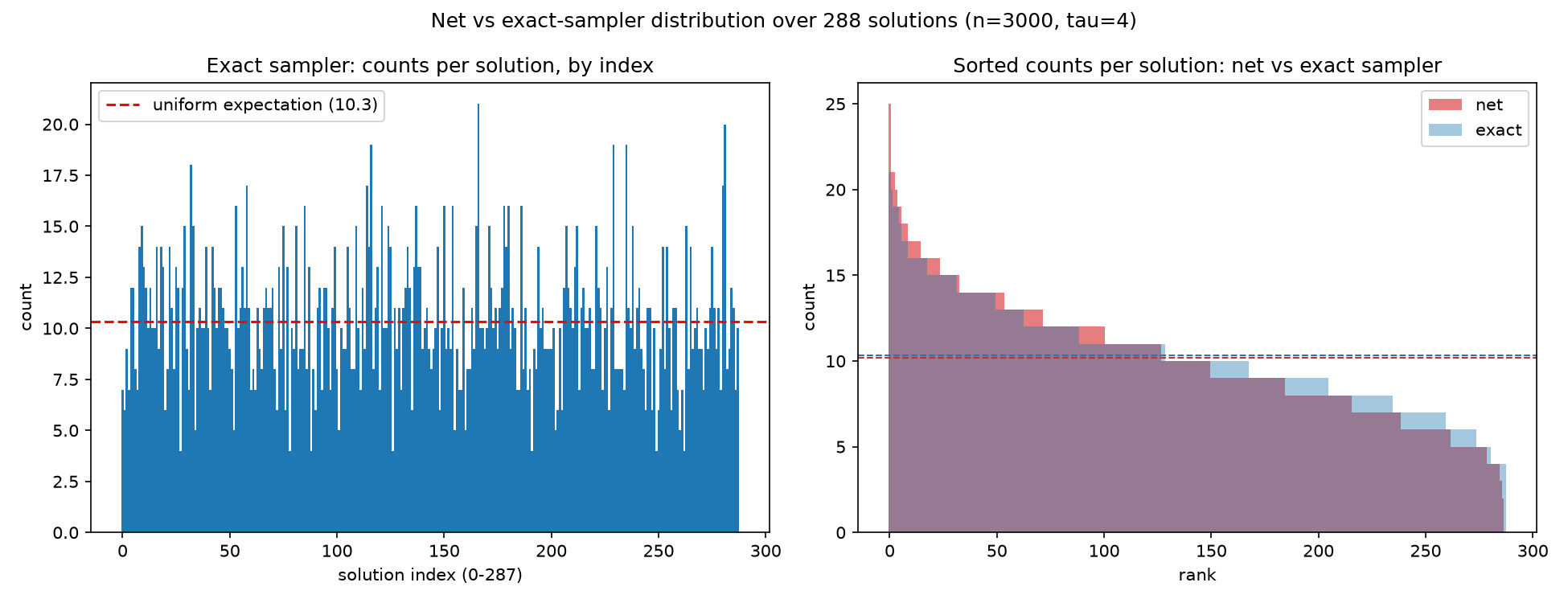}
        \caption{4x4 Sudoku chi-squared test ground truth samples.}
        \label{}
    \end{subfigure}
    \caption{}
    \label{fig:visit_hist}
\end{figure}

\FloatBarrier

\FloatBarrier

\subsection{6x6 Sudoku}

\textbf{Target Density}

$6\times6$ Sudoku has 6 rows, 6 columns, and six non-overlapping
$2\times3$ rectangular boxes, each of which must be a permutation of $1$--$6$ to be valid. Each cell is a $3$-bit encoding of the numbers $1$--$6$, giving
$36 \times 3 = 108$ core bits and a state space of $2^{108}$. Unlike the
$4\times4$ case, this encoding has \emph{dead code words}: three bits admit eight
values, of which only six are legal, so two codes per cell decode to the
invalid values $7$ and $8$. The generator never emits them, but the model
operates on the full $2^{108}$ space and can.

\begin{figure}[H]
\begin{center}
\includegraphics[scale=0.5]{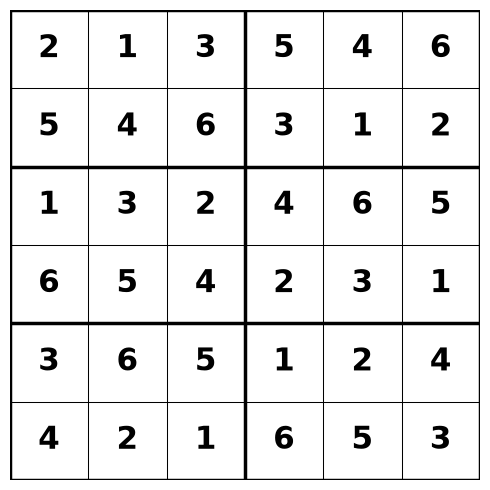}
\end{center}
\caption{Example valid 6x6 grid.}
\end{figure}

There are
$28{,}200{,}960$ proper puzzles, vastly more than the 4x4 puzzle. The training set does not contain all proper puzzles, thus the network's generations should be novel making new proper puzzles not seen in the training set.

The base distribution is again a weak generator sampling each row
independently. With probability $1-\rho$ the row is a uniformly random
permutation of $1$--$6$, and thus valid; with probability $\rho$ it is instead
six i.i.d.\ uniform digits drawn from $1$--$6$, which may repeat. We use
$\rho = 0.02$. A row is therefore valid with probability
$(1-\rho) + \rho \cdot \frac{720}{6^{6}} = 98.03\%$, and all six rows are
simultaneously valid $88.75\%$ of the time.

Because the generator enforces no column or box constraints, and the grid is
much larger than in the $4\times4$ case, it produces a proper puzzle with
probability only $1.8\times10^{-10}$ --- effectively never. The reward is again
the validity indicator.

\textbf{Training Set}

Since the weak generator essentially never produces a proper puzzle on its own,
reward-achieving examples must be injected: each sample is drawn uniformly from
the set of valid solutions with probability $p_{\text{inject}} = 0.05$, and from
the weak generator otherwise. The training is scaled up to $5{,}000{,}000$
samples from this mixture, roughly 33 times larger than the $4\times4$ training
set, which is necessary because the injected solutions are drawn from a pool of
$28.2$ million rather than $288$. The realized set contains $249{,}903$ valid
grids ($5.00\%$), of which $248{,}778$ are distinct.

\textbf{Model Architecture}

The target distributions and models were trained in order from smallest to largest state space starting with the Ising Model, then 4x4, 6x6, 9x9 Sudoku, then the MNIST problem. As the state space scaled up in size, larger data sets were needed to learn the more complicated density and tilt, and more expressive model architectures.

The previous Ising and 4x4 Sudoku problems used MLP models with no attention mechanism. For 6x6, 9x9 Sudoku and MNIST we use a transformer architecture with attention.

This problem uses a transformer featurizer rather than the residual MLP, since
the row, column and box constraints are joint combinatorial structure that a
flat trunk must discover entirely from data, whereas attention can relate cells
directly. The 108 core bits are split into 6 tokens of 18 bits (one per row),
and the reward coordinate is given its own token rather than folded into the
level conditioning, so attention can relate $x_{*}$ to every row and vice versa;
attention is non-causal over the resulting 7 slots. The model has $4$ layers
with $d_{\text{model}} = 256$, $8$ attention heads, and feed-forward dimension
$1024$, and noise-level conditioning enters by AdaLN from the same Fourier level
embedding $e(a)$ used elsewhere.

The output head is unchanged, preserving the mean-field skip
\begin{align*}
G_{\theta}(u,a) = \tanh\!\big(\beta(a)\,u + g_{\theta}(a)\odot r_{\theta}(u,a)\big),
\end{align*}
with every residual branch zero-initialized so that at initialization the
network is exactly the mean-field denoiser. The network is additionally
self-conditioned, as discussed in Appendix \ref{app:self-cond}, the previous estimate of $x_{0}$ is concatenated to the input
along the feature axis of each token, doubling the per-slot input channels from
18 to 36. The model has $3{,}249{,}028$ parameters, substantially fewer than the
$4\times4$ MLP despite the much larger state space. Full implementation details are in the supplementary code.

\textbf{Training Run}

The model was trained for 180,000 Adam steps with batch size 8192, learning rate
$10^{-3}$, and a cosine-annealing schedule over the full 180,000 steps. During
training the self-conditioning input is supplied with probability $0.5$, so the
network also learns the zero-input fallback needed at the first sampling step.
Each batch re-noises freshly drawn samples, with $\beta$ drawn log-uniformly on
$[0.02, 50]$; with probability $0.3$ the row instead draws $\beta$ log-uniformly
from the narrower band $[0.5, 2.5]$, which over-samples the regime empirically
hardest to fit. With probability $0.25$ a row uses the sparse infill profile
($k$ uniform on $\{2,\dots,108\}$, a random $k$-subset set to $\sqrt{108/k}$),
and otherwise the reward-focusing profile; independently, the reward-axis
stretch $\lambda$ is $1$ with probability $0.25$ and otherwise log-uniform on
$[1,10]$. The objective is the unweighted per-coordinate squared error summed
over the 109 coordinates; the final training loss was $53.92$.

Training ran on a single NVIDIA H100 80GB HBM3 GPU (Hopper, CUDA 12.6) on the
PSC Bridges-2 cluster and took 3.70 hours. Processing
$180{,}000 \times 8192 \approx 1.47\times10^{9}$ samples through a
$3{,}249{,}028$-parameter network corresponds to approximately
$2.9\times10^{16}$ training FLOPs under the standard $6ND$ estimate.

\textbf{Extended Results}

For the 6x6 Sudoku problem, the weak generator samples each row independently thus the ground truth denoiser expectation can be enumerated for the untilted density. With a much larger set of valid puzzles we do not enumerate the denoiser with reward tilting applied. The $L^{2}$ error for $\tau = 0$ with the base denoiser is shown in Figure \ref{fig:6x6_ood}.

\begin{figure}
\begin{center}\includegraphics[scale=0.5]{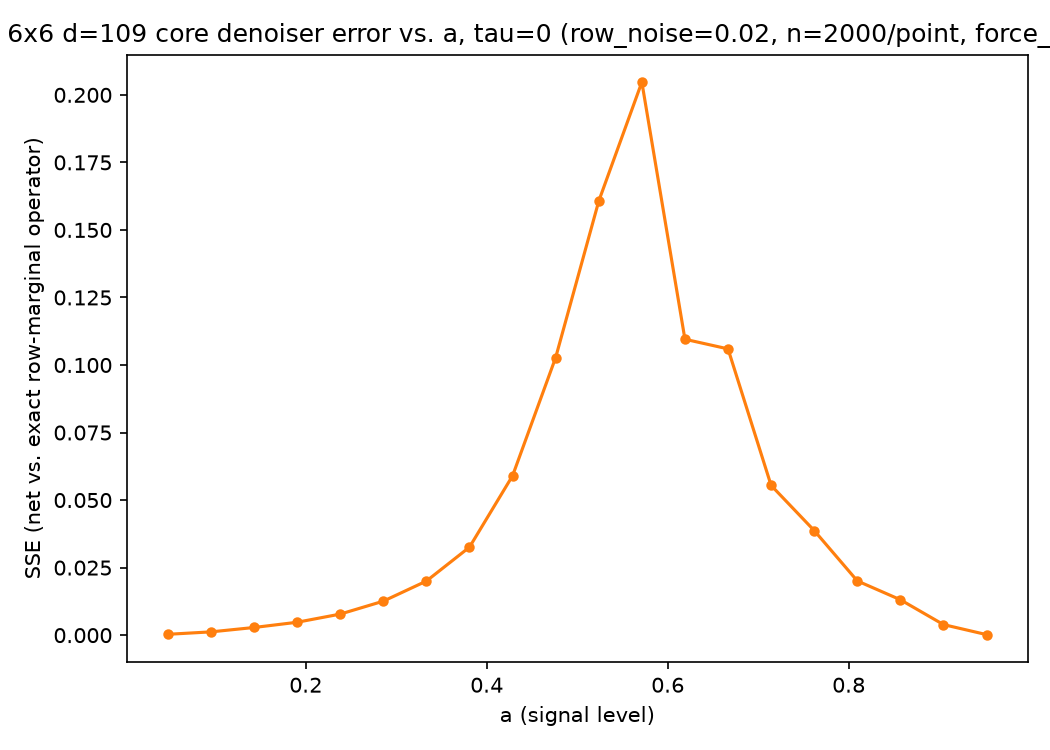}
\caption{$L^{2}$ error with enumerated ground truth denoiser untilted.}
\label{fig:6x6_ood}
\end{center}
\end{figure}

For the log-odds ratio, the $\tau$ achieving 98\% reward is computed, and $5$ values below and 5 values up to $\sqrt{d}$ are computed. With only 2000 samples, we do not have the specificity to detect very high probabilities, hence the confidence intervals on the log-odds start to become uninformative for very high tilts. Note these tilts are already producing over 98\% correct puzzles.

\begin{figure}
\begin{center}\includegraphics[scale=0.5]{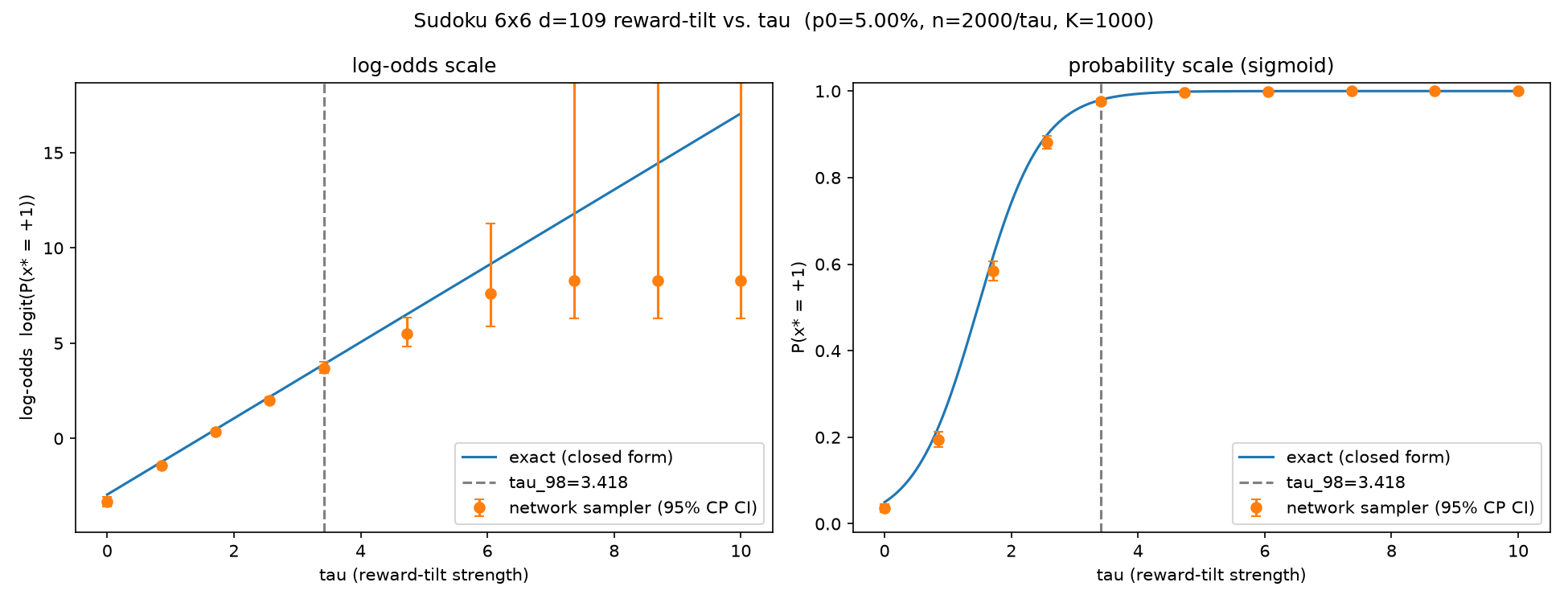}
\end{center}
\caption{6x6 Sudoku log-odds across tilt strengths.}\label{fig:6x6_logodds}
\end{figure}

At $\tau_{98} = 3.42$ we drew $1000$ samples, of which $917$ decoded to valid grids, and checked each for an exact match against all $5{,}000{,}000$ training rows. The training set covers $248{,}778$ of the $28{,}200{,}960$ valid $6\times6$ grids ($0.88\%$), so a sampler drawing uniformly from valid grids would produce $8.1$ exact matches by chance (null $95\%$ interval $3$--$14$); we observe $10$ ($1.09\%$ of valid samples, $95\%$ CI $0.52$--$2.00\%$, $p = 0.29$ for $\geq 10$ under the null), consistent with no memorization.

\FloatBarrier

\FloatBarrier 

\subsection{9x9 Sudoku}

\textbf{Target Density}

Full-size $9\times9$ Sudoku has 9 rows, 9 columns, and nine non-overlapping
$3\times3$ boxes, each required to be a permutation of $\{1,\dots,9\}$. There
are $\approx 6.67\times10^{21}$
proper puzzles.

\begin{figure}[H]
\begin{center}
\includegraphics[scale=0.5]{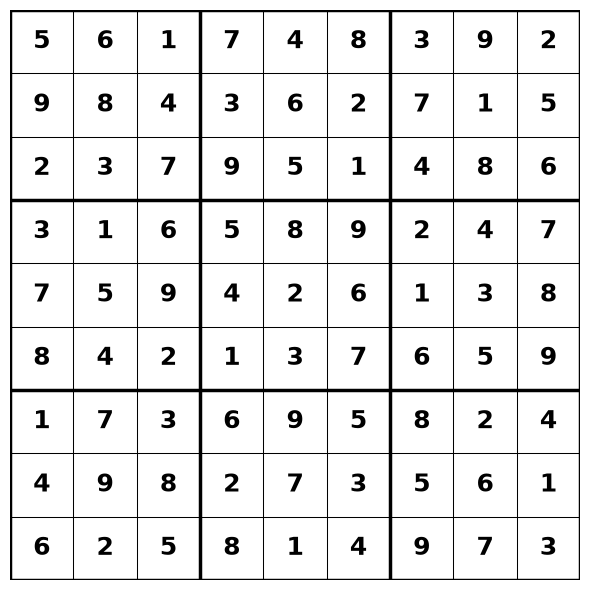}
\end{center}
\caption{Example valid 9x9 grid.}
\end{figure}

Each cell is a $4$-bit encoding of the numbers $1$--$9$, giving
$81 \times 4 = 324$ core bits and a state space of $2^{324}$. The dead-code word
problem is considerably worse than at $6\times6$: four bits admit sixteen
values of which only nine are legal.

The base distribution is again a weak generator sampling each row
independently: with probability $1-\rho$ a uniformly random permutation of
$1$--$9$, and with probability $\rho$ nine i.i.d.\ uniform digits from
$1$--$9$, which may repeat. We use $\rho = 0.02$, giving a row validity of
$98.00\%$ and all nine rows simultaneously valid $83.39\%$ of the time. With no
column or box constraints enforced, the generator produces a proper puzzle with
probability $5.1\times10^{-29}$.

The reward is the validity indicator, appended as a 325th coordinate to give
the extended state space $\{-1,1\}^{325}$.

\textbf{Training Set}

The training set consists of $10{,}000{,}000$ samples, drawn from the weak
generator with probability $1 - p_{\text{inject}}$ and from a valid-solution
sampler with probability $p_{\text{inject}} = 0.05$. Unlike $6\times6$, no
efficient exactly uniform sampler over valid grids exists at this scale: the
$6\times6$ construction enumerated the $39{,}168$ grids with first row
$1,\dots,6$ and composed them with relabellings, whereas the corresponding
$9\times9$ quotient has $1.84\times10^{16}$ reduced grids and is not
enumerable. Injection instead uses randomized backtracking composed with a
uniform draw from the full symmetry group ($|G| = 1{,}218{,}998{,}108{,}160$),
which is exactly uniform within each symmetry orbit and biased only in the
orbit label.

The realized training set contains $500{,}883$ valid grids ($5.01\%$), all of them distinct, so each is
seen roughly 147 times over the training run. Proper sampling should essentially have 0 probability of recreating a training example, thus the model should generate 100\% novel valid puzzles and never recreate a valid puzzle seen in the training set.

\textbf{Model Architecture}

The architecture is the same self-conditioned transformer used at $6\times6$,
scaled up. The 324 core bits are split into 9 tokens of 36 bits (one per row),
the reward coordinate is given its own token, and non-causal attention runs
over the resulting 10 slots. The model has $8$ layers with
$d_{\text{model}} = 512$, $8$ heads, and feed-forward dimension $2048$, with
AdaLN level conditioning and the same zero-initialized mean-field skip
\begin{align*}
G_{\theta}(u,a) = \tanh\!\big(\beta(a)\,u + g_{\theta}(a)\odot r_{\theta}(u,a)\big),
\end{align*}
so that at initialization the network is exactly the mean-field denoiser.
Self-conditioning concatenates the previous $x_{0}$ estimate along each token's
feature axis, doubling per-slot input channels from 36 to 72. The model has
$25{,}558{,}382$ parameters. Full implementation details are in the supplementary code.

\textbf{Training Run}

The model was trained for 180,000 Adam steps with batch size 8192, learning
rate $10^{-3}$, and a cosine-annealing schedule, with self-conditioning
supplied with probability $0.5$. Signal levels are drawn log-uniformly on
$[0.02, 50]$, with probability $0.3$ instead drawn from the harder band
$[0.5, 2.5]$. With probability $0.25$ a row uses the sparse infill profile
($k$ uniform on $\{2,\dots,324\}$, a random $k$-subset set to $\sqrt{324/k}$),
and otherwise the reward-focusing profile; the reward-axis stretch $\lambda$ is
$1$ with probability $0.25$ and otherwise log-uniform on $[1, 17]$.

Training ran on a single NVIDIA H100 80GB HBM3 GPU (Hopper, CUDA 12.6) on the
PSC Bridges-2 cluster and took 22.5 hours. Processing
$180{,}000 \times 8192 \approx 1.47\times10^{9}$ samples through a
$25{,}558{,}382$-parameter network corresponds to approximately
$2.3\times10^{17}$ training FLOPs under the standard $6ND$ estimate.

\textbf{Extended Results}

For the 9x9 Sudoku, there is no ground truth denoiser to compare against since the row space is too large to enumerate. We consider the log-odds ratio.

For the 9x9 model, we consider the two forms of self-conditioning. Mode ``two-pass'' is a Markovian form of self-conditioning which may be interpreted as an architecture choice where the denoiser is fed its own prediction
\begin{align*}
\tilde{G}_{t} = G(X_{t}, G(X_{t},0,a),a).
\end{align*}
Mode ``carry'' uses the previous time point's estimate as input, which breaks the Markov structure and the sampling path is no longer a pure Euler-Maruyama discretization of the continuous SDE with pointwise denoiser estimation,
\begin{align*}
\tilde{G}_{t} =  G(X_{t},\tilde{G}_{t-1}, a).
\end{align*}

For the 9x9 model we see noticeable difference in performance between the two-pass self-conditioning and the carry self-conditioning.

Two-pass fails to properly track the log-odds. It plateaus around 60\% reward-achieving solutions, and is already lagging behind at lower tilt strengths. Note the model has its own prediction of if it thinks it is producing a reward-achieving solution in its d+1 coordinate, and we can determine if the produced token string is reward-achieving as well. For a properly sampling model, these two should agree. For the two-pass mode, we see the model's own predictions of reward (orange) do trace the log-odds, but the actual proportion of reward-achieving solutions does not. Therefore, the model thinks it is producing 90\% plus reward-achieving, but in reality only 60\% are.

The carry mode does much better at tracking the log-odds, plateauing at 90\% reward-achieving. At higher tilt strengths, it also experiences some discrepancy between the solutions it thinks are reward-achieving via its $d+1$ bit being positive, but are graded as invalid puzzles by the Sudoku conditions. This discrepancy is more on the order of 5-10\% for the carry mode, rather than the 30\% of the two-pass mode.

\begin{figure}
    \centering
    \begin{subfigure}[t]{\textwidth}
        \centering
        \includegraphics[width=\linewidth]{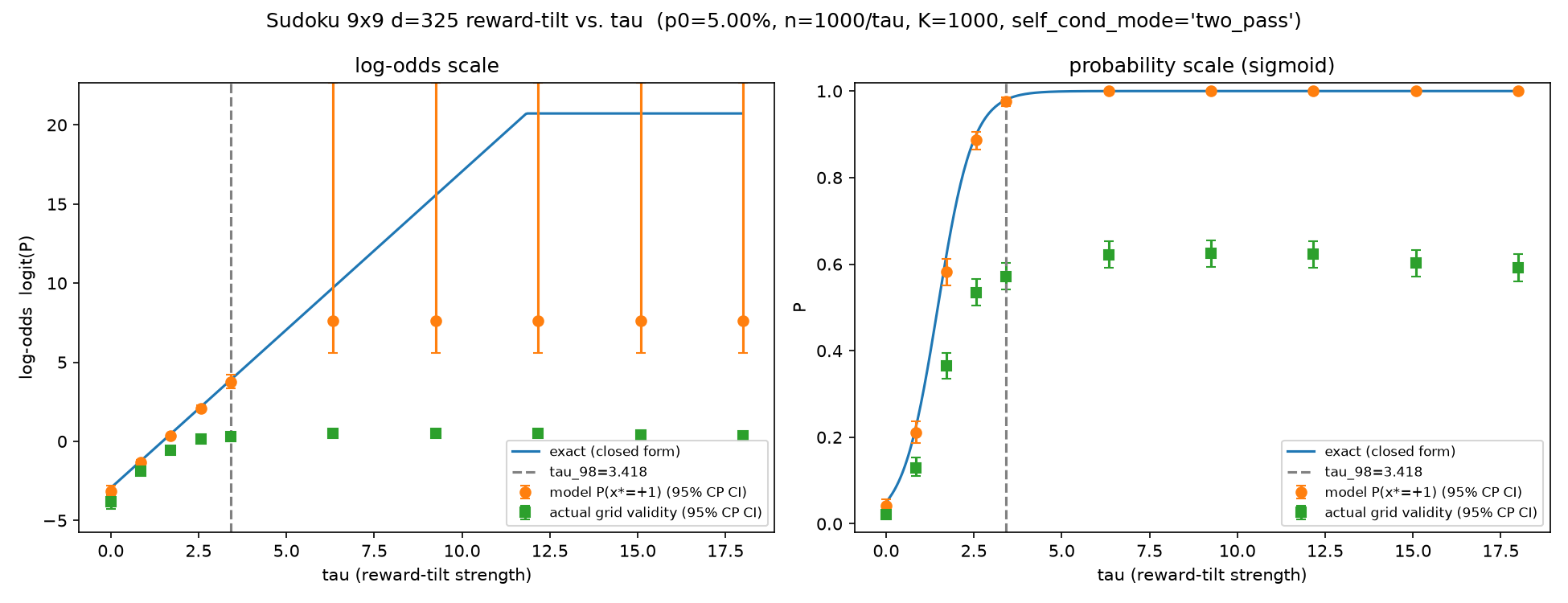}
        \caption{Self-conditioning mode two-pass.}
        \label{fig:9x9_logodds_twopass}
    \end{subfigure}
    \hfill
    \begin{subfigure}[t]{\textwidth}
        \centering
        \includegraphics[width=\linewidth]{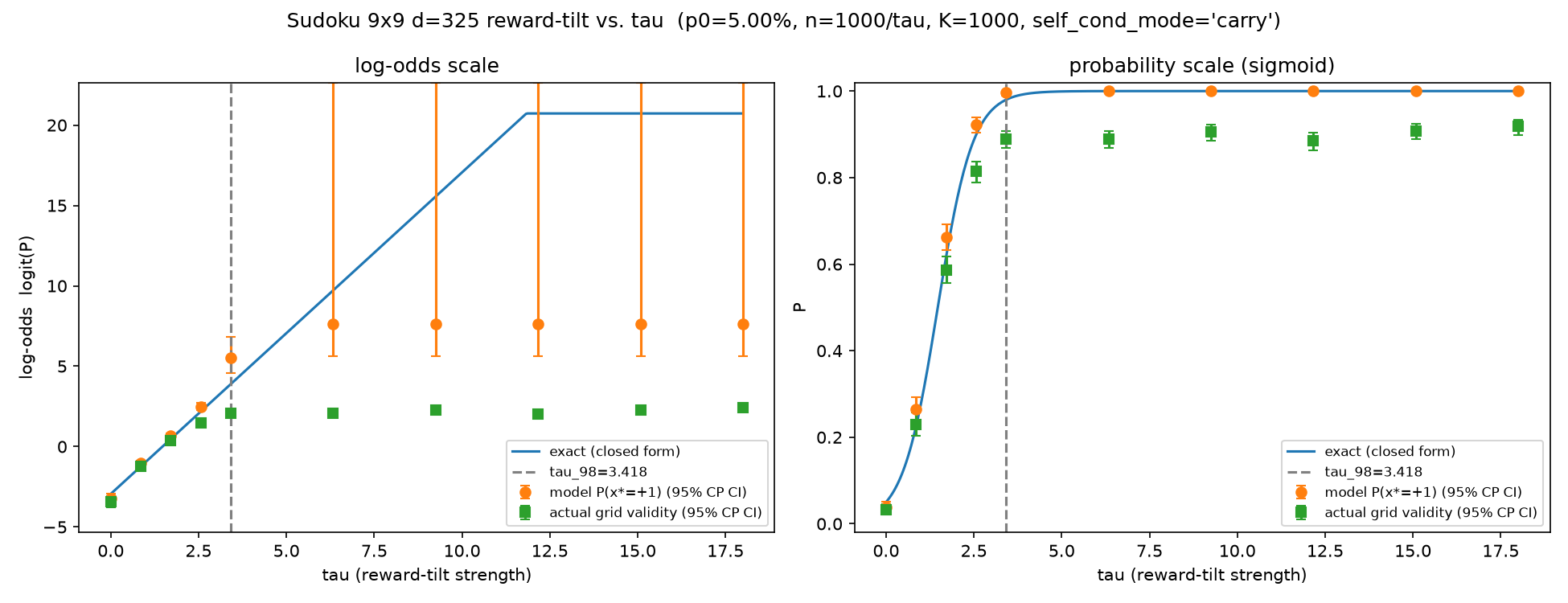}
        \caption{Self-conditioning mode carry.}
        \label{fig:9x9_logodds_carry}
    \end{subfigure}
    \caption{9x9 Sudoku log-odds across tilt strengths for both self-conditioning modes.}
    \label{}
\end{figure}

To consider why this may be the case, it is possible the few missing percent of samples that were ``supposed'' to terminate at a valid puzzle mis-sampled a few bits due to errors in the learned denoiser accumulating over the sample path, which would get a -1 reward from the sharp correct/incorrect puzzle reward function for any cell violations.

Using a CP-SAT optimization problem we can find for each generated puzzle the nearest valid solution in Hamming distance. We plot the CDF of how many produced puzzles in 300 samples are valid (distance 0), 1 bit away from correct, 2, and so on from a valid puzzle in Figure \ref{fig:ham_cdf}. We see an example puzzle in Figure \ref{fig:CP_bit_correct} that was generated by two-pass self-conditioning at a tilt strength of 3 and is 8 bit flips away from a valid puzzle. We see in Figure \ref{fig:ham_cdf} for the two-pass mode, while only 60\% of puzzles are valid, 90\% are within 10 bits out of 324 from a valid puzzle. Two-pass and carry have comparable proportions of puzzles produced within 10 bits of a valid puzzle.

\begin{figure}
\begin{center}
\includegraphics[scale=0.5]{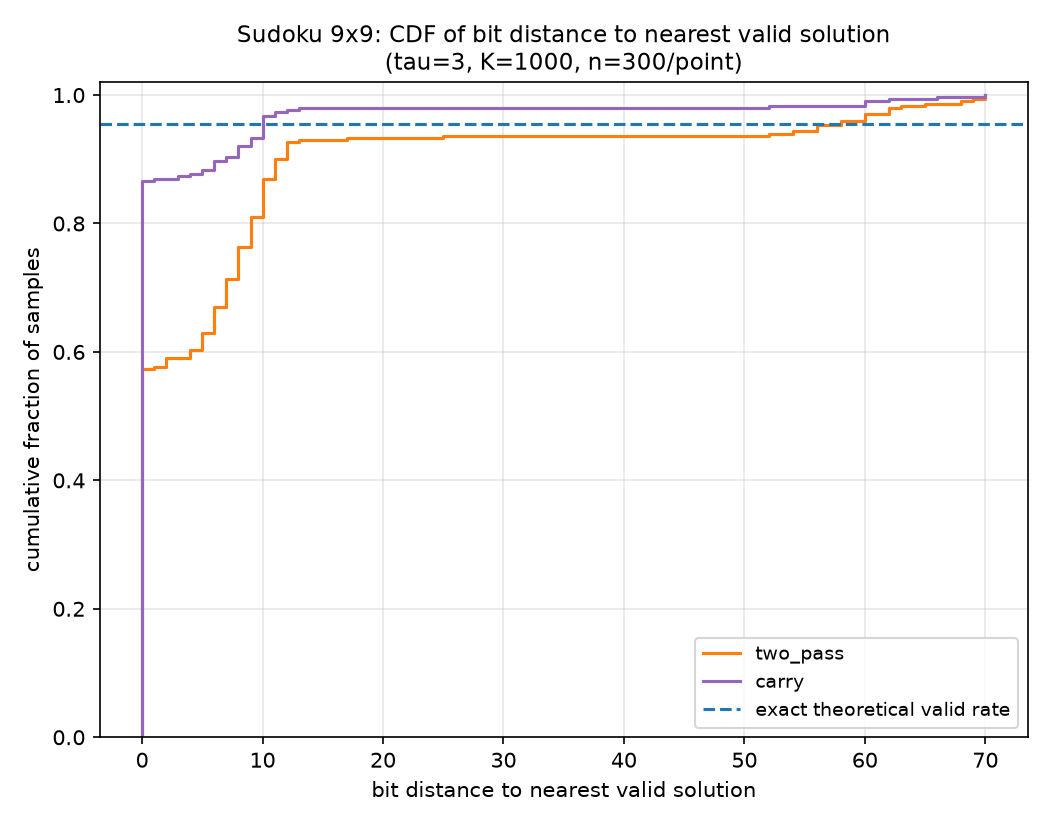}
\end{center}\caption{Cumulative distribution function of Hamming distance to nearest valid puzzle for 300 samples.}\label{fig:ham_cdf}
\end{figure}

Here we summarize the infilling performance of the model, using an infill strength and a reward tilt strength of $\tau = 3$ across all tests. Given less initial cell values, the sampler does produce less valid puzzles and does not always retain the initial cell values given to it. Probability of returning the given initial cells and generating a valid puzzle ranges from 80\% to 100\% across the amount of initial cells given.

\begin{table}[H]
\centering
\caption{Sudoku $9\times9$ infilling as a function of the number of given cells.
Each row is based on 100 uniformly random valid solutions $\times$ 20 draws per puzzle
($n=2000$), with a fresh random mask per puzzle. Tilt strength $\tau = 3$ for both the
reward and the cell-fixing tilts, 1000 Euler steps, self-conditioning mode \texttt{carry}.
Brackets give 95\% Clopper--Pearson intervals. ``Joint'' is the fraction of samples that
are valid grids and also keep every given cell.}
\label{tab:sudoku9-infill}
\small
\begin{tabular}{@{}rcccc@{}}
\toprule
Given & Removed & Valid grid & Keeps given cells & Joint success \\
\midrule
71 & 10 & $1.000\,[0.998,\,1.000]$ & $1.000\,[0.998,\,1.000]$ & $1.000\,[0.998,\,1.000]$ \\
61 & 20 & $1.000\,[0.998,\,1.000]$ & $0.999\,[0.997,\,1.000]$ & $0.999\,[0.997,\,1.000]$ \\
51 & 30 & $0.999\,[0.997,\,1.000]$ & $0.998\,[0.994,\,0.999]$ & $0.997\,[0.994,\,0.999]$ \\
41 & 40 & $0.996\,[0.992,\,0.998]$ & $0.989\,[0.983,\,0.993]$ & $0.985\,[0.979,\,0.990]$ \\
31 & 50 & $0.953\,[0.942,\,0.961]$ & $0.944\,[0.933,\,0.954]$ & $0.899\,[0.885,\,0.912]$ \\
21 & 60 & $0.925\,[0.913,\,0.936]$ & $0.918\,[0.905,\,0.930]$ & $0.851\,[0.835,\,0.866]$ \\
11 & 70 & $0.867\,[0.851,\,0.882]$ & $0.924\,[0.911,\,0.935]$ & $0.803\,[0.784,\,0.820]$ \\
\bottomrule
\end{tabular}
\end{table}

\begin{figure}
\begin{center}
\includegraphics[scale=0.5]{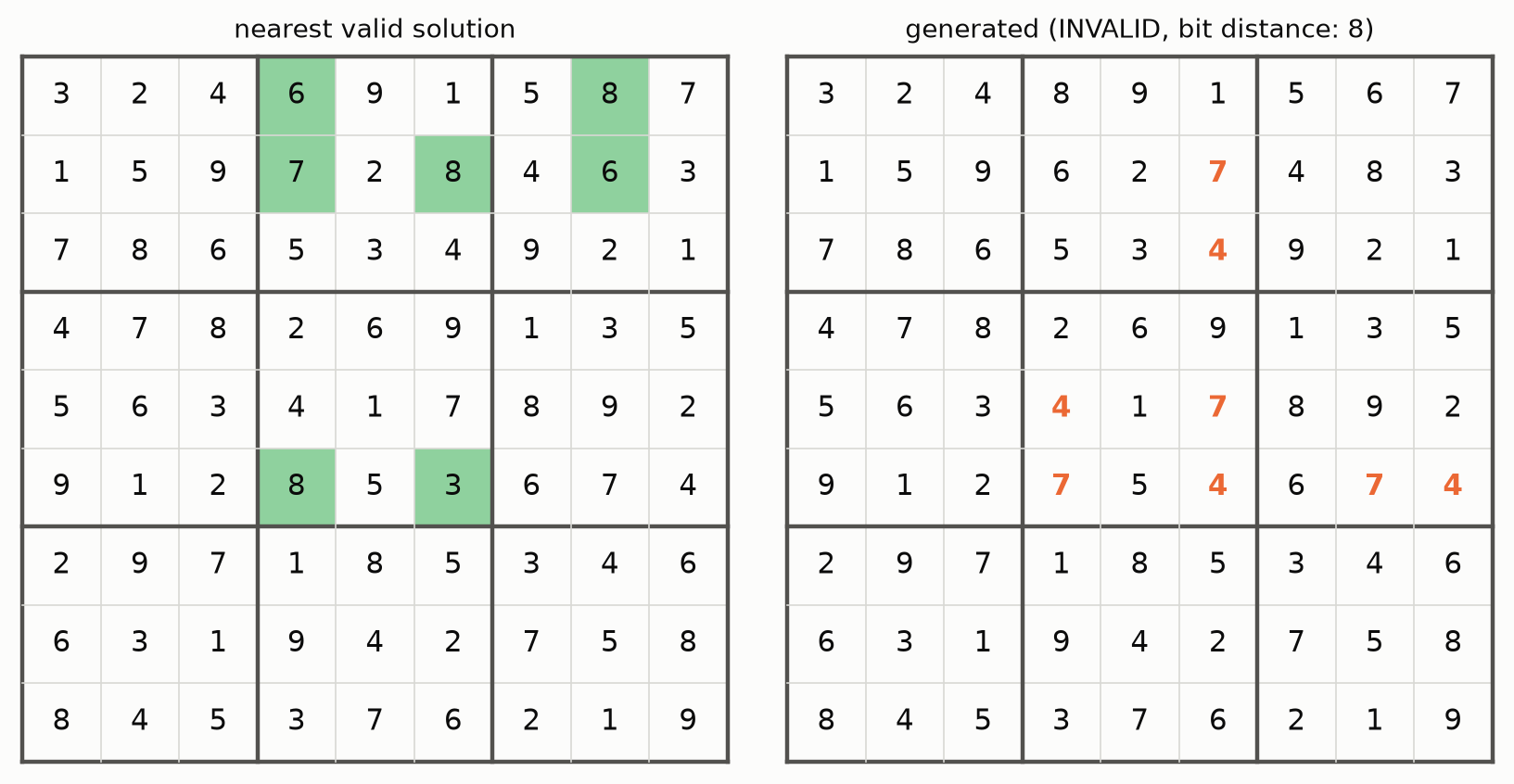}
\end{center}
\caption{Generated invalid puzzle (right) with constraint violations and nearest valid puzzle (left) in Hamming distance. Nearest valid grid would require 8 bit flips to the binary encoding. The bit flips are not all restricted to the invalid bits of the produced puzzle, but previously valid bits may be flipped to get to a valid puzzle.}
\label{fig:CP_bit_correct}
\end{figure}

A natural concern for a sampler that reaches high validity is that it has simply
memorized its training set. We test this on Sudoku $9\times9$ with the
self-conditioned transformer (\texttt{carry}, $\tau=3$, 1000 Euler steps). The
$500{,}883$ valid training grids, all distinct, are split at random into a
reference half ($250{,}000$ grids) and a held-out half. Of $N=1000$ generations,
869 are valid Sudoku solutions, and none of the $1000$ matches any reference set
exactly. Given that there are
about $6.67\times10^{21}$ valid grids, a single exact match would have been strong
evidence of copying. We also compute, for
each valid generation, the cell-wise Hamming distance to its nearest reference
grid, and compare this with the same distance computed for $869$ held-out
training grids, which are true unseen solutions from the data distribution
(Figure \ref{fig:9x9_mem_check}). The two distributions are almost
identical. Generated grids are therefore no closer to the training data
than new samples from the data distribution would be, and the closest generation
(39 cells away) still differs from every reference grid in almost half its cells.
The model has learned the Sudoku constraints rather than copying its training
grids.

\begin{figure}
\begin{center}
\includegraphics[scale=0.5]{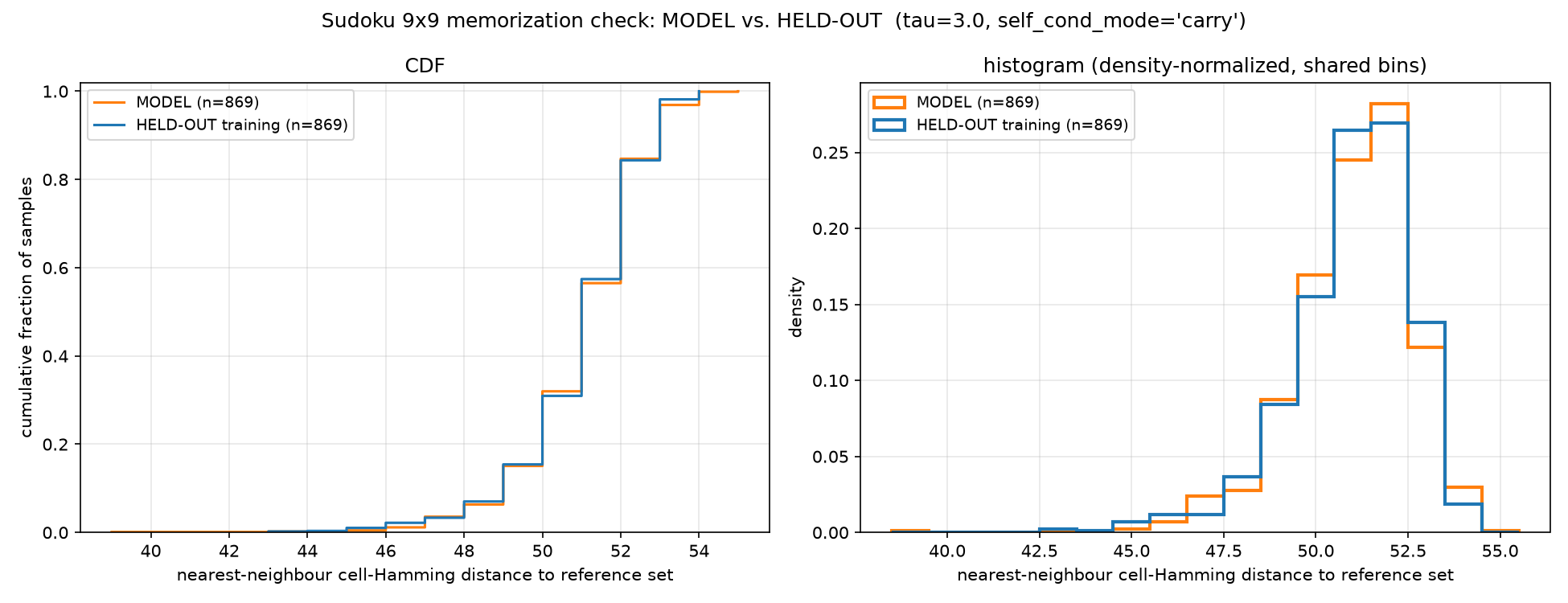}
\end{center}\caption{CDF and density of nearest distance reward-achieving gird in training set for network generated samples and held out set. Nearly identical distributions, so generated puzzles no more similar to training set than a held out portion. No puzzles are exact match, closet puzzle differ in 39 cells}
\label{fig:9x9_mem_check}
\end{figure}

\FloatBarrier

\FloatBarrier

\subsection{MNIST Conditional Image Generation}

\textbf{Target Density}

MNIST consists of $28\times28$ grayscale images of handwritten digits $0$--$9$,
with pixel intensities in $\{0,\dots,255\}$. Each pixel is encoded with $8$ bits,
giving $784 \times 8 = 6{,}272$ core bits and a state space of $2^{6272}$. Since the pixel values have a natural ordering, we use
a Gray code encoding for each pixel, so that adjacent intensity levels differ in exactly one bit.

Since there are 10 digit classes, the reward is a $10$-bit block rather than a
single coordinate. The class label $c$ is stored as a one-hot pattern
$r_{c} \in \{-1,1\}^{10}$, $+1$ in position $c$ and $-1$ elsewhere, appended
to the image bits to give the extended state space $\{-1,1\}^{6282}$. To
generate class $c$ we tilt along the whole pattern, $h = \tau\, r_{c}$, rather
than along the single axis $e_{c}$. Thus one class receives a positive tilt while the others all receive a negative one. This causes the log-odds of a class to have a slope of 4,
\begin{align*}
\operatorname{logit} \pi_{\tau}(c) = \operatorname{logit} \pi(c) + 4\tau,
\end{align*}
where $\pi(c)$ is the training-set frequency of class $c$ ($0.090$--$0.112$).
The slope is $4$ rather than the slope $2$ of the single-bit reward tilt in the
Sudoku and Ising problems.

\textbf{Training Set}

The training set is the standard $60{,}000$ MNIST training images and their
labels, each encoded as above. The official $10{,}000$-image test split is
held out and never used for training. It is used for three things:
(i) a fixed batch of $4{,}096$ held-out images, noised with the training
recipe, is scored every $1{,}000$ steps to monitor overfitting;
(ii) it provides the real-digit comparison in
Figure~\ref{fig:three-panel}; and
(iii) it is used to measure the error rates of the classifier that scores
generated samples (see Extended Results).

\textbf{Model Architecture}

We use the same self-conditioned transformer as for $6\times6$ and $9\times9$
Sudoku, with tokens built from image patches. The image is divided into a
$7\times7$ grid of non-overlapping $4\times4$ pixel patches. Each patch becomes
one token of $16 \times 8 = 128$ bits (pixels in row order within a patch, Gray
bits most significant first), giving $49$ image tokens. The $10$-bit class
block is given its own token, and non-causal attention runs over the resulting
$50$ slots. The model has $4$ layers with $d_{\text{model}} = 256$, $8$
attention heads, and feed-forward dimension $1024$, with AdaLN noise-level
conditioning from the Fourier level embedding $e(a)$ and the same
zero-initialized mean-field skip
\begin{align*}
G_{\theta}(u,a) = \tanh\!\big(\beta(a)\,u + g_{\theta}(a)\odot r_{\theta}(u,a)\big).
\end{align*}
Self-conditioning concatenates the previous $x_{0}$ estimate to each token's
features, doubling the per-slot input channels from $128$ to $256$. The model
has $3{,}555{,}256$ parameters. Full implementation details are in the supplementary code.

\textbf{Training Run}

The model was trained for $30{,}000$ Adam steps with batch size $2048$,
learning rate $10^{-3}$, and a cosine-annealing schedule, with the
self-conditioning input supplied with probability $0.5$. Signal levels are drawn
log-uniformly on $\beta \in [0.02, 50]$. $\Lambda$ focusing is applied to the whole 10-bit reward block, maintaining the sign of those inputs but increasing their magnitude. The $\lambda$ is $1$ with probability $0.25$ and otherwise
log-uniform on $[1, 24]$. The upper limit differs from the Sudoku runs: because
the whole $10$-coordinate block is stretched at once, the admissible range is
$\lambda < \sqrt{d/10} = 25.06$ rather than $\sqrt{d}$, and we train just below
this limit.

$30{,}000 \times 2048 \approx 6.1\times10^{7}$ draws amount to roughly
$1{,}000$ passes over the $60{,}000$ training images. The held-out loss
decreased steadily from $720.6$ at step $1{,}000$ to $632.1$ at step $30{,}000$
and was still flat-to-decreasing at the end. This gives no
sign of overfitting.

Training ran on a single NVIDIA H100 80GB GPU on the PSC Bridges-2 cluster for 1 hour and 15 minutes. Processing $6.1\times10^{7}$ samples through a
$3{,}555{,}256$-parameter network corresponds to approximately
$1.3\times10^{15}$ training FLOPs under the standard $6ND$ estimate.

\textbf{Extended Results}

To detect if a generated image is of a given class, we train an independent score CNN classifier. The scorer is a two-layer CNN (1.2M parameters) trained for 3 epochs with Adam on the 60,000 MNIST training images, reaching 98.74\% accuracy on the held-out test split.

The classifier is an imperfect
instrument and cannot detect very strong tilt accuracies beyond its own classification error. Its error is also not constant across classes. As such, our ability to detect tilt strength saturates at the classifier accuracy.

Near the top of the tilt range its error is the dominant
source of uncertainty. Let $s_{c}$ be the probability the classifier labels a
true class-$c$ image as $c$, and $e_{c}$ the probability it labels an image of
another class as $c$. For a true class-$c$ rate $p$, the rate we observe is
\begin{align*}
m = p\, s_{c} + (1-p)\, e_{c}.
\end{align*}
We measure $s_{c}$ and $e_{c}$ on the held-out set ($s_{c}$ ranges from
$0.962$ for class $9$ to $0.997$ for class $1$; $e_{c} \le 0.0023$). We then
compare the observed rate with the exact marginal passed through this equation,
rather than correcting the observed rate. The observed rate therefore levels
off at about $s_{c}$ rather than $1$ (log-odds $3.2$--$5.9$ depending on the
class). Once the corrected exact rate comes within two standard errors of this
limit, which happens for $\tau \geq 1.69$ in every class, the sample size is
not what limits resolution; the classifier cannot distinguish further
improvements. We treat these points as uninformative about the classifier-scored
rate.

We report two rates from the same samples: the classifier-scored rate, and the
class the network itself writes in its $10$-bit reward block (argmax of the
block). The second is read directly from the sample and involves no
classifier. The $\tau$ grid has $6$ points on $[0, \tau_{95}]$, where
$\tau_{95} = 1.285$ is the strength at which the exact rate reaches $95\%$ for
the mean class frequency, and $3$ points on $(\tau_{95}, 2.5]$. We anchor at
$95\%$ rather than the $98\%$ used for Sudoku because the classifier limit
leaves less usable range. Each (class, $\tau$) pair uses $256$ samples with
$1000$ Euler steps, giving $2{,}560$ samples per $\tau$ pooled over classes.
Tables~\ref{tab:mnist-logodds-twopass} and~\ref{tab:mnist-logodds-carry} give
the pooled results, and
Figures~\ref{fig:mnist-logodds-twopass}--\ref{fig:mnist-logodds-carry-pooled}
show each class separately.

\begin{table}[H]
\centering
\caption{MNIST pooled class-conditional log-odds, \texttt{two-pass}
self-conditioning, $n = 2{,}560$ samples per $\tau$ ($256$ per class). ``Exact''
is $\operatorname{logit}$ of the exact tilted class marginal; ``Exact
(scored)'' is the same rate passed through each class's measured classifier
errors and averaged over classes, which is the right reference for the
classifier-scored column. ``Classifier'' is the fraction of samples the CNN
assigns to the requested class, with a 95\% Clopper--Pearson interval.
``Reward block'' is the fraction whose reward block has its argmax at the
requested class. $^{\dagger}$Beyond the classifier's resolution limit.}
\label{tab:mnist-logodds-twopass}
\small
\begin{tabular}{@{}rcccc@{}}
\toprule
$\tau$ & Exact & Exact (scored) & Classifier & Reward block \\
\midrule
0.00 & $-2.20$ & $-2.20$ & $-2.11\,[-2.24,\,-1.99]$ & $-2.10$ \\
0.26 & $-1.17$ & $-1.18$ & $-1.19\,[-1.29,\,-1.10]$ & $-1.12$ \\
0.51 & $-0.14$ & $-0.16$ & $\phantom{-}0.00\,[-0.07,\,\phantom{-}0.08]$ & $\phantom{-}0.10$ \\
0.77 & $\phantom{-}0.89$ & $\phantom{-}0.84$ & $\phantom{-}0.93\,[\phantom{-}0.84,\,\phantom{-}1.02]$ & $\phantom{-}1.14$ \\
1.03 & $\phantom{-}1.91$ & $\phantom{-}1.82$ & $\phantom{-}1.78\,[\phantom{-}1.67,\,\phantom{-}1.89]$ & $\phantom{-}2.03$ \\
1.29 & $\phantom{-}2.94$ & $\phantom{-}2.71$ & $\phantom{-}2.28\,[\phantom{-}2.15,\,\phantom{-}2.42]$ & $\phantom{-}2.91$ \\
$1.69^{\dagger}$ & $\phantom{-}4.56$ & $\phantom{-}3.75$ & $\phantom{-}2.76\,[\phantom{-}2.60,\,\phantom{-}2.93]$ & $\phantom{-}3.78$ \\
$2.10^{\dagger}$ & $\phantom{-}6.18$ & $\phantom{-}4.20$ & $\phantom{-}3.13\,[\phantom{-}2.94,\,\phantom{-}3.34]$ & $\phantom{-}4.34$ \\
$2.50^{\dagger}$ & $\phantom{-}7.80$ & $\phantom{-}4.32$ & $\phantom{-}3.36\,[\phantom{-}3.14,\,\phantom{-}3.59]$ & $\phantom{-}6.05$ \\
\bottomrule
\end{tabular}
\end{table}

\begin{table}[H]
\centering
\caption{As Table~\ref{tab:mnist-logodds-twopass}, for \texttt{carry}
self-conditioning, with the same $\tau$ grid, sample sizes and seeds.
$^{\dagger}$Beyond the classifier's resolution limit.}
\label{tab:mnist-logodds-carry}
\small
\begin{tabular}{@{}rcccc@{}}
\toprule
$\tau$ & Exact & Exact (scored) & Classifier & Reward block \\
\midrule
0.00 & $-2.20$ & $-2.20$ & $-2.19\,[-2.32,\,-2.06]$ & $-2.19$ \\
0.26 & $-1.17$ & $-1.18$ & $-0.92\,[-1.01,\,-0.84]$ & $-0.92$ \\
0.51 & $-0.14$ & $-0.16$ & $\phantom{-}0.31\,[\phantom{-}0.23,\,\phantom{-}0.39]$ & $\phantom{-}0.33$ \\
0.77 & $\phantom{-}0.89$ & $\phantom{-}0.84$ & $\phantom{-}1.33\,[\phantom{-}1.24,\,\phantom{-}1.43]$ & $\phantom{-}1.42$ \\
1.03 & $\phantom{-}1.91$ & $\phantom{-}1.82$ & $\phantom{-}2.12\,[\phantom{-}2.00,\,\phantom{-}2.25]$ & $\phantom{-}2.20$ \\
1.29 & $\phantom{-}2.94$ & $\phantom{-}2.71$ & $\phantom{-}2.83\,[\phantom{-}2.66,\,\phantom{-}3.00]$ & $\phantom{-}3.19$ \\
$1.69^{\dagger}$ & $\phantom{-}4.56$ & $\phantom{-}3.75$ & $\phantom{-}3.23\,[\phantom{-}3.03,\,\phantom{-}3.45]$ & $\phantom{-}3.80$ \\
$2.10^{\dagger}$ & $\phantom{-}6.18$ & $\phantom{-}4.20$ & $\phantom{-}3.73\,[\phantom{-}3.47,\,\phantom{-}4.00]$ & $\phantom{-}5.28$ \\
$2.50^{\dagger}$ & $\phantom{-}7.80$ & $\phantom{-}4.32$ & $\phantom{-}4.05\,[\phantom{-}3.75,\,\phantom{-}4.37]$ & $\phantom{-}5.45$ \\
\bottomrule
\end{tabular}
\end{table}

The self-conditioning modes two-pass (Markovian) and carry (non-Markovian) behave in the opposite way to $9\times9$
Sudoku. The carry model over-samples the reward class at a given tilt, producing more outputs of a given class than the log-odds would predict. The two-pass mode tracks the log-odds more truthfully.

With \texttt{two-pass}, the classifier-scored rate matches the corrected
exact rate within its confidence interval, or nearly so, from $\tau = 0$
through $\tau \approx 1$ (e.g.\ $1.78$ vs.\ $1.82$ at $\tau = 1.03$), and the
network's reward block follows the same curve. With \texttt{carry}, the
sampler \emph{overshoots} over the whole rising region. At $\tau = 0.51$ it
produces the requested class $57.8\%$ of the time against an exact rate of
$45.9\%$ (log-odds $0.31$ vs.\ $-0.16$), and the $95\%$ interval excludes the
exact value at every $\tau$ from $0.26$ to $1.03$. The untilted rate is
correct under both modes, so this is an error in how \texttt{carry} responds to
the tilt, not in the base distribution.

\begin{figure}[H]
\begin{center}
\includegraphics[width = \textwidth]{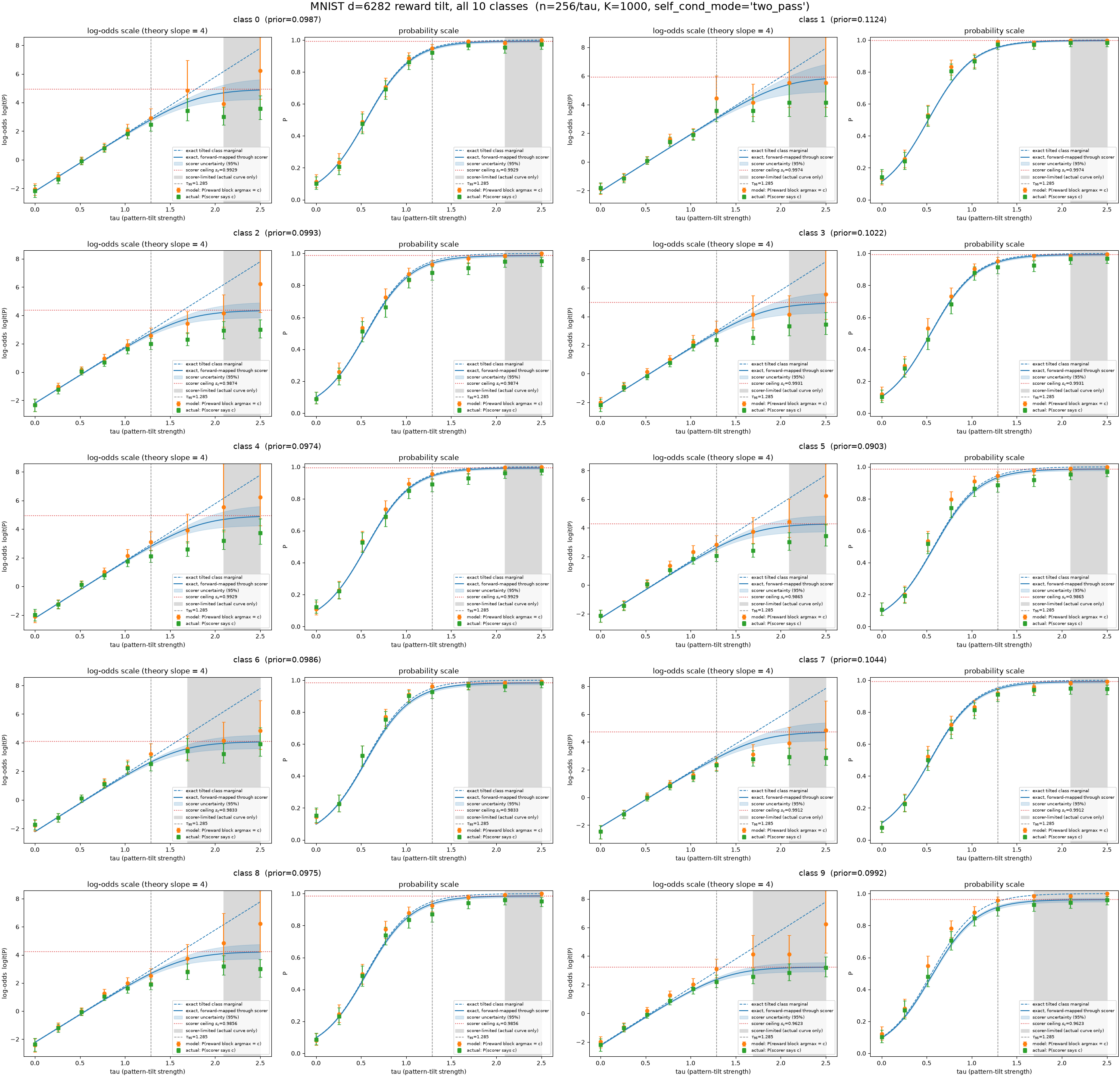}
\end{center}
\caption{MNIST class-conditional log-odds against tilt strength $\tau$ for each
class, \texttt{two-pass} self-conditioning. Exact marginal passed through each
class's classifier errors (line, with a band from the uncertainty in those
errors), classifier-scored rate (points with 95\% Clopper--Pearson intervals),
and the network's own reward-block rate.}
\label{fig:mnist-logodds-twopass}
\end{figure}

\begin{figure}[H]
\begin{center}
\includegraphics[scale=0.5]{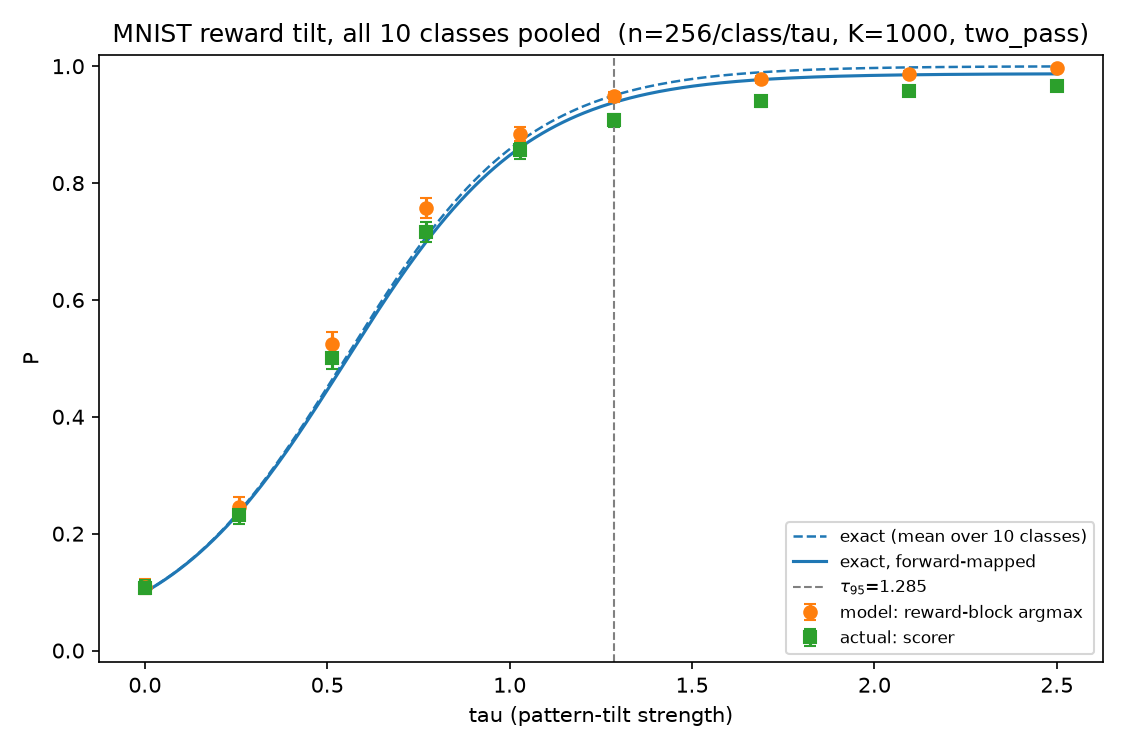}
\end{center}
\caption{MNIST log-odds pooled over all ten classes, \texttt{two-pass}
self-conditioning (Table~\ref{tab:mnist-logodds-twopass}).}
\label{fig:mnist-logodds-twopass-pooled}
\end{figure}

\begin{figure}[H]
\begin{center}
\includegraphics[width = \textwidth]{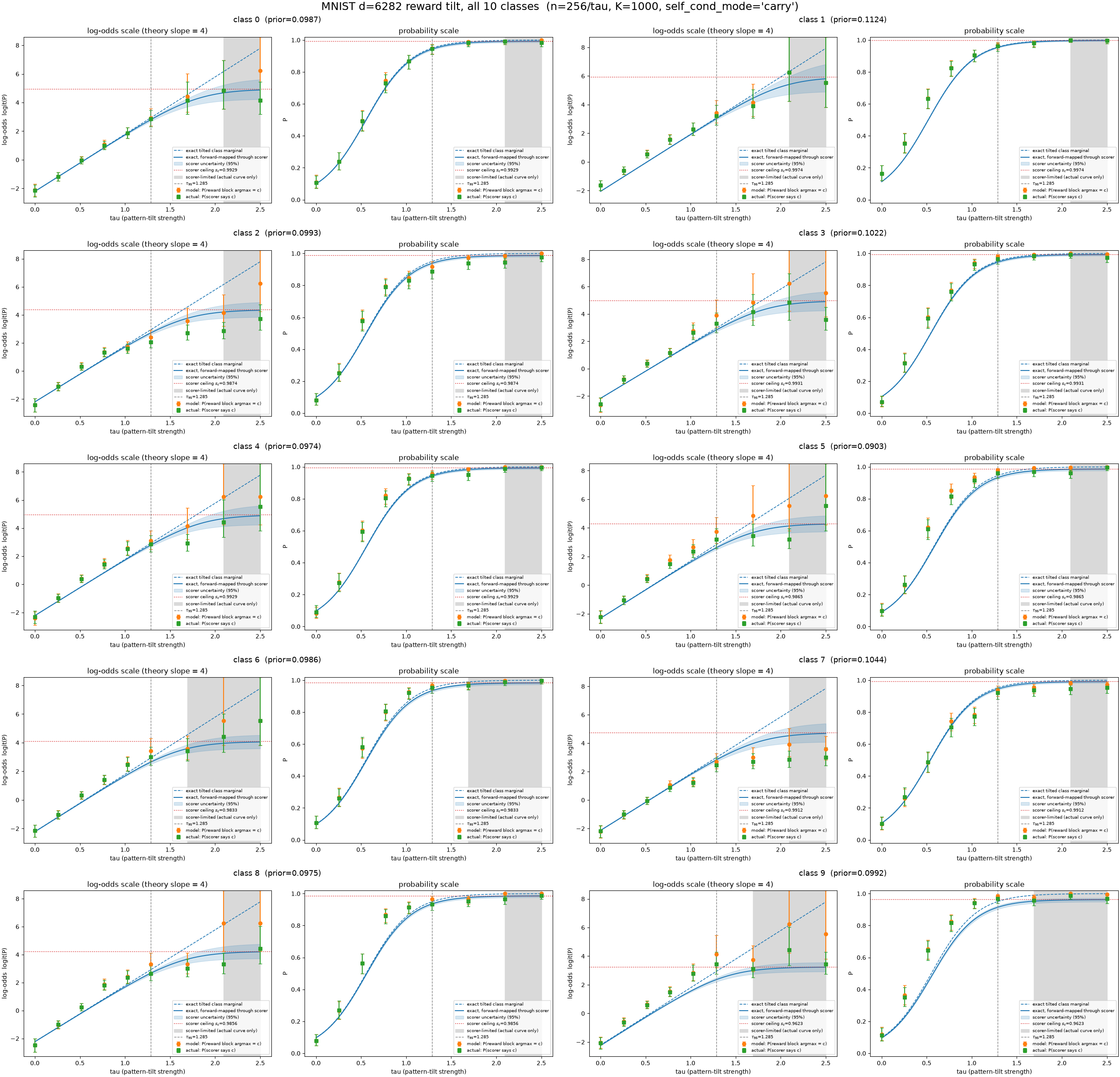}
\end{center}
\caption{As Figure~\ref{fig:mnist-logodds-twopass}, for \texttt{carry}
self-conditioning. The overshoot over $0.26 \le \tau \le 1.03$ appears in
every class.}
\label{fig:mnist-logodds-carry}
\end{figure}

\begin{figure}[H]
\begin{center}
\includegraphics[scale=0.5]{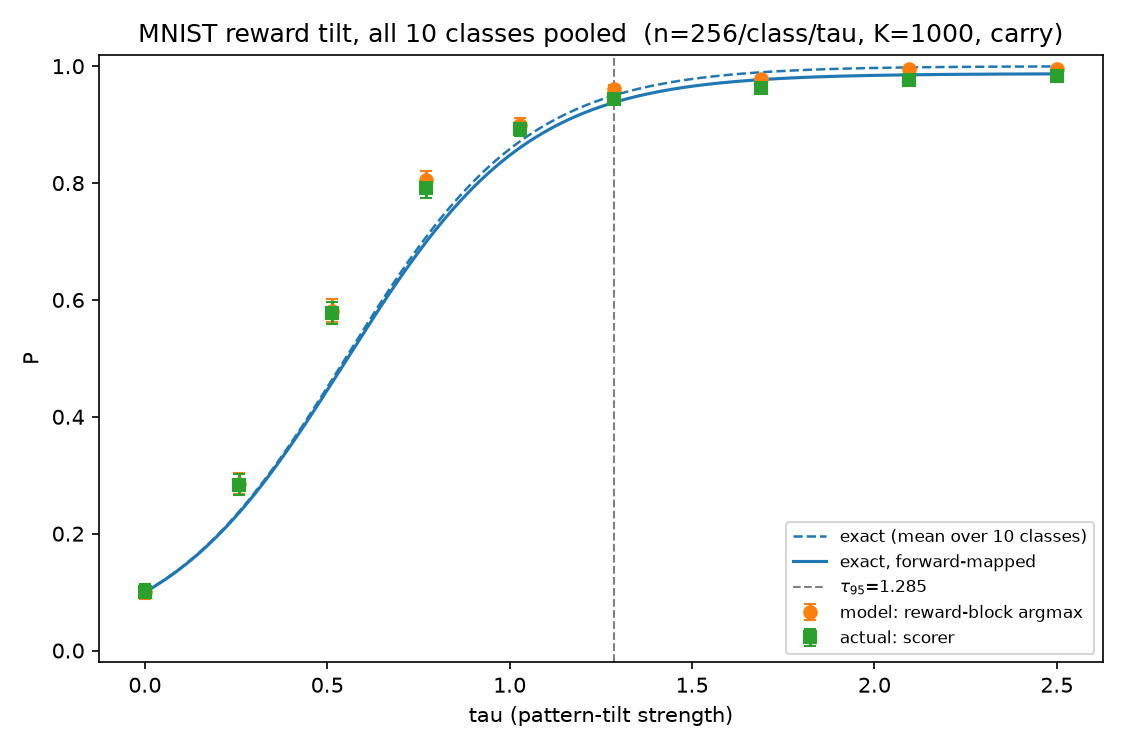}
\end{center}
\caption{MNIST log-odds pooled over all ten classes, \texttt{carry}
self-conditioning (Table~\ref{tab:mnist-logodds-carry}).}
\label{fig:mnist-logodds-carry-pooled}
\end{figure}

The log-odds check does not consider the sample quality or variety of digits produced. We sample 25 class 0 images from the hold out set, 25 samples images at the carry mode, and 25 samples at two-pass mode at a tilt $\tau_{0.95}$ with a projected 95\% success rate. We compare the outputs visually for sample quality and variety to detect mode collapse or corruption of digit quality in Figure~\ref{fig:three-panel}. Digits generated by the network show the same variety and quality as true hold out images. Note one of the generated samples was mis-sampled as a class 8 image, but this is to be expected with a 5\% non-reward rate at this tilt strength. Shared seed and randomness was used in the two sampling runs which results in similarity in the images generated in pane b) and c).

\begin{figure}[H]
    \centering
    \begin{subfigure}[t]{0.32\textwidth}
        \centering
        \includegraphics[width=\linewidth]{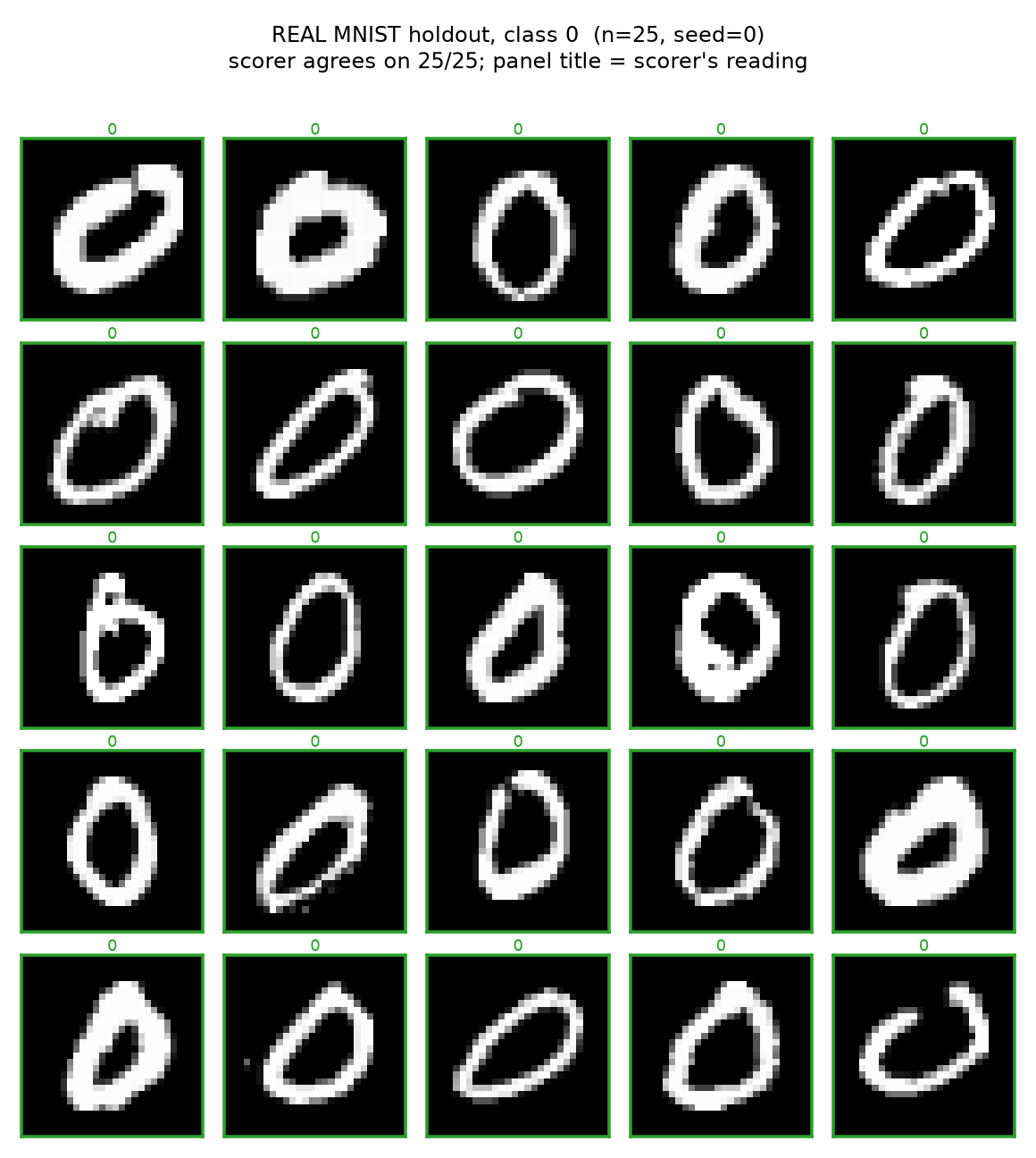}
        \caption{25 class-0 images from the held-out set.}
        \label{fig:mnist-real}
    \end{subfigure}
    \hfill
    \begin{subfigure}[t]{0.32\textwidth}
        \centering
        \includegraphics[width=\linewidth]{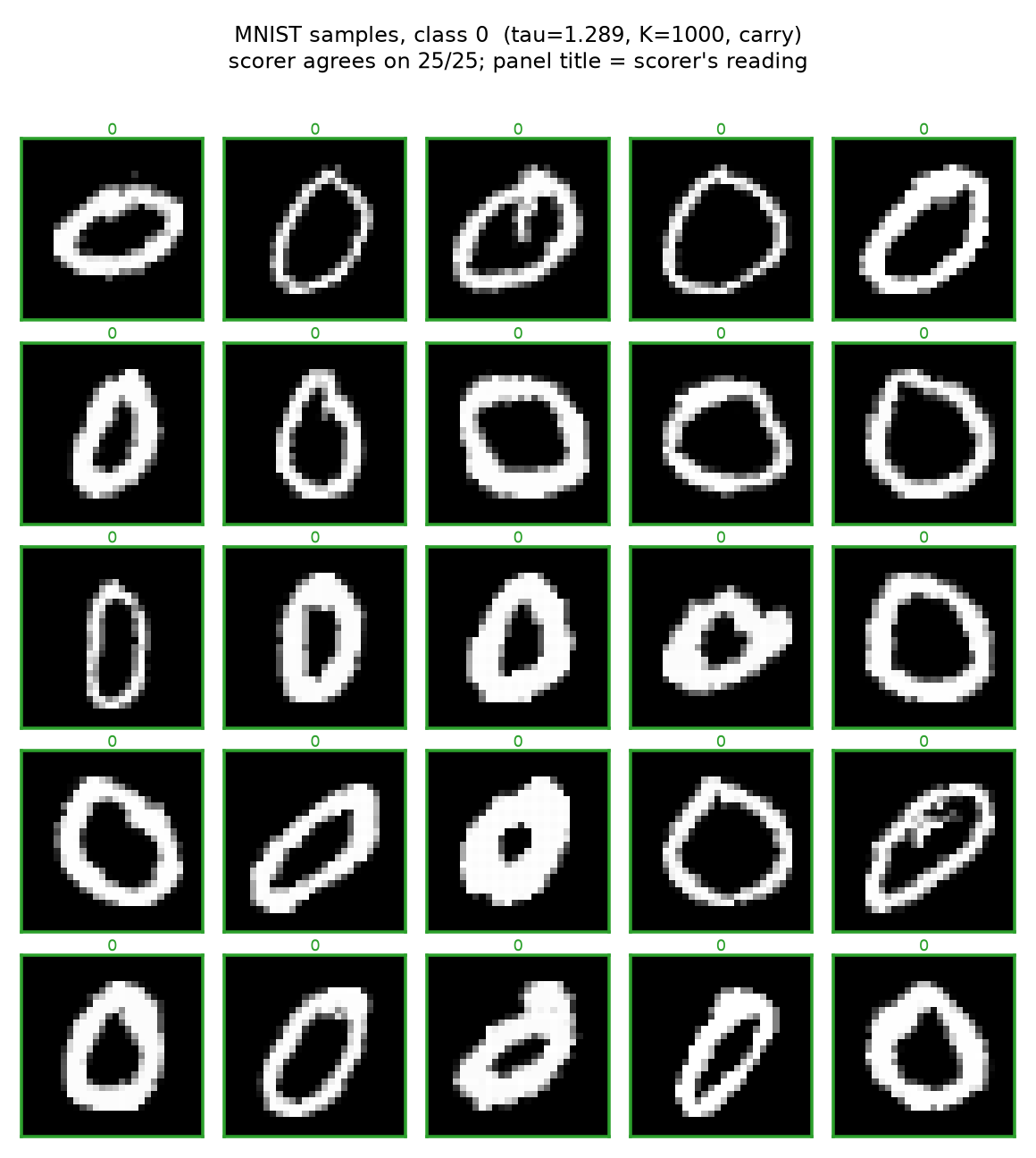}
        \caption{25 generated class-0 images, \texttt{carry} (non-Markovian)
        self-conditioning.}
        \label{fig:mnist-carry}
    \end{subfigure}
    \hfill
    \begin{subfigure}[t]{0.32\textwidth}
        \centering
        \includegraphics[width=\linewidth]{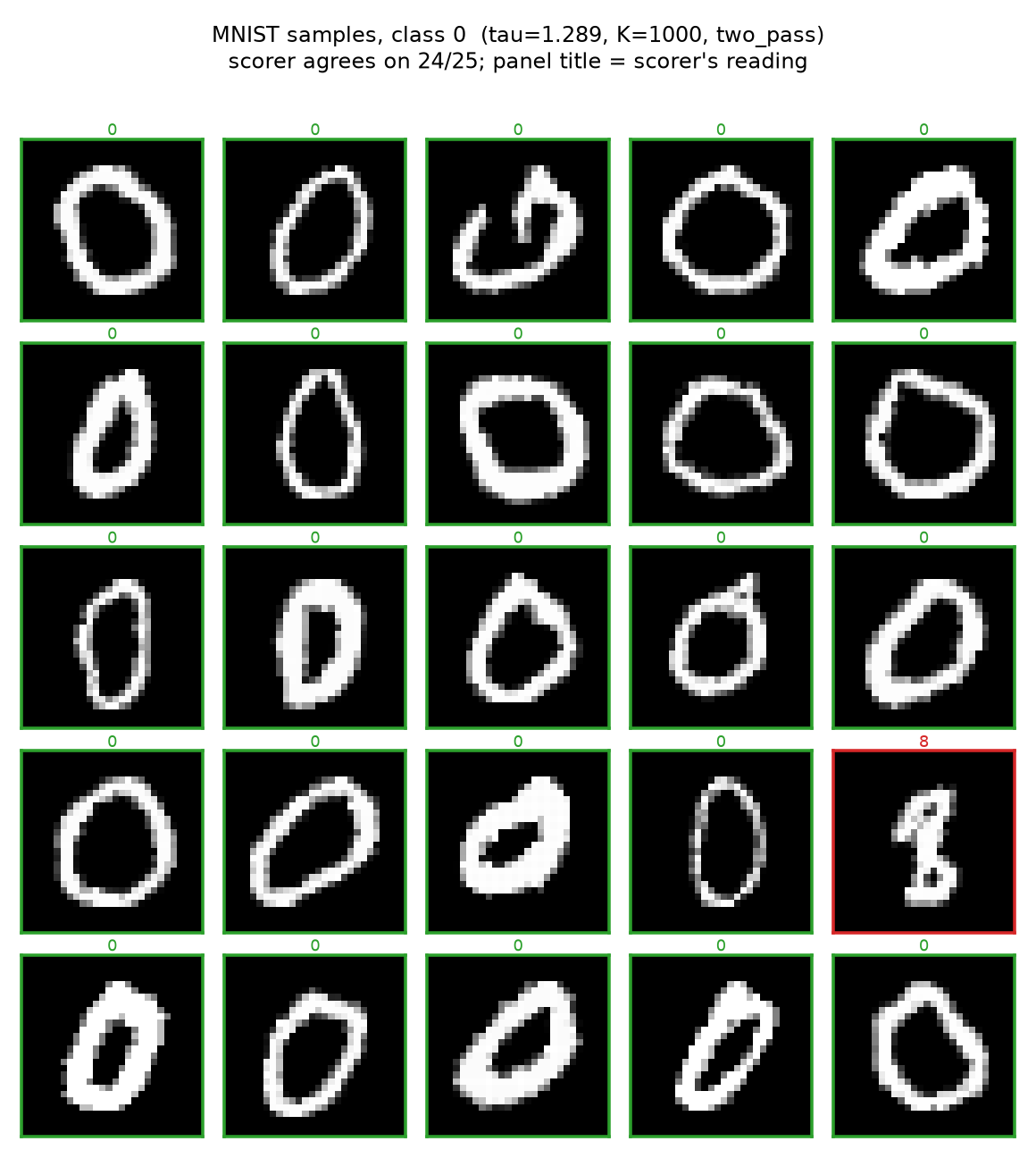}
        \caption{25 generated class-0 images, \texttt{two-pass} (Markovian)
        self-conditioning.}
        \label{fig:mnist-twopass}
    \end{subfigure}
    \caption{25 class-0 images from the held-out set (a), and network samples
    under both self-conditioning modes with the same random seed (b, c). Panel
    titles give the classifier's label (green if it matches the requested
    class, red if not). The samples are similar in quality and variety to the
    held-out digits, with no sign of mode collapse. $\tau = 1.289$ is set so
    that the exact class-$0$ rate is $95\%$, so about $1$ in $25$ samples is
    expected to be of another class.}
    \label{fig:three-panel}
\end{figure}

\begin{figure}
\begin{center}
\includegraphics[width = \textwidth]{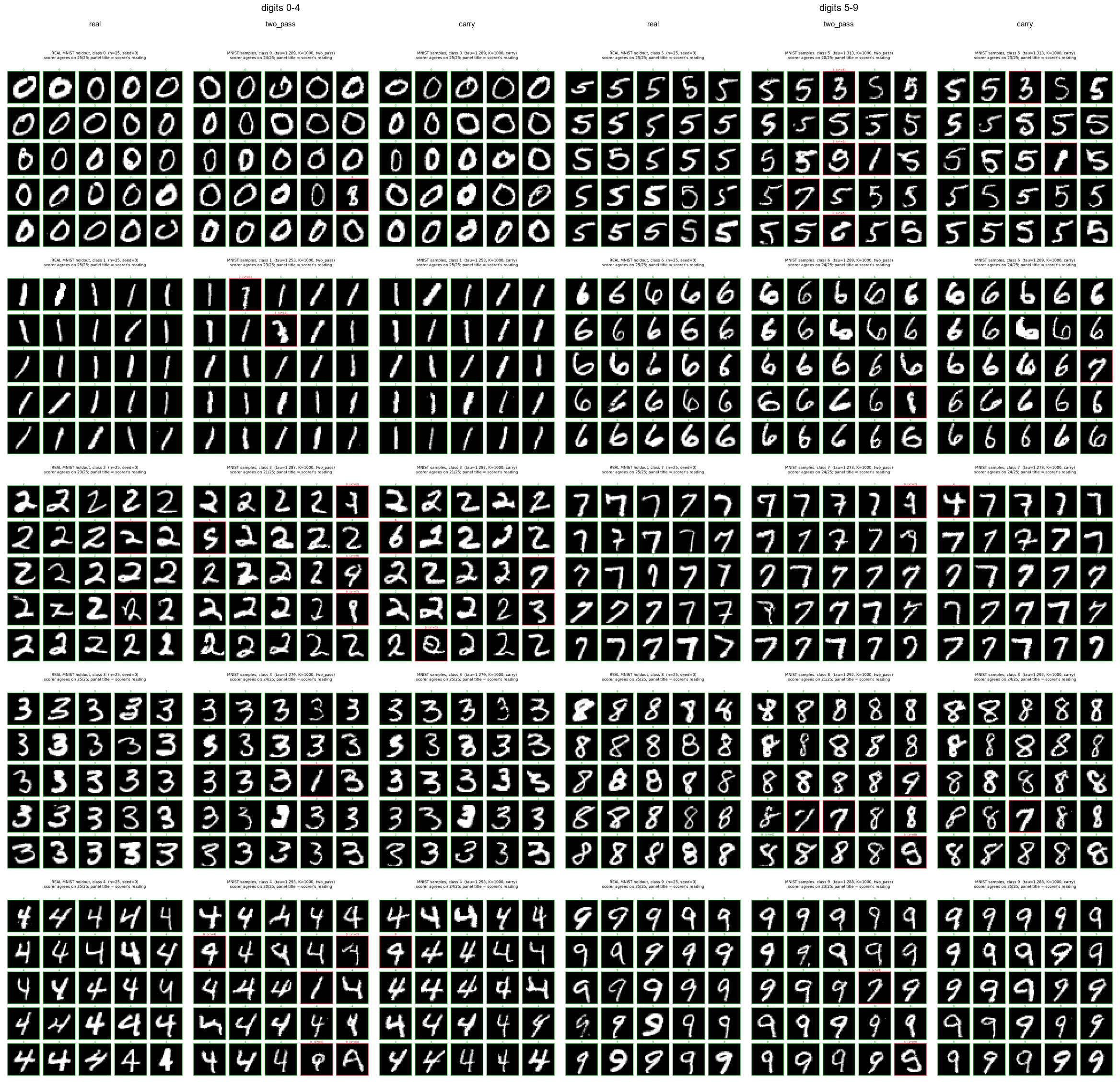}
\end{center}
\caption{Example digit generations tilted towards each class. Leftmost image per digit is a sample of 25 images from hold out set network was not trained on, second image is from sample mode carry and third image is sample mode two-pass. Middle and right grids share the same random seed in their sampling, resulting in often similar samples which is not a form of memorization. Tilt strength is calibrated to 95\% reward-achieving in log-odds theory, thus it is reasonable to see a few mis-samples in repeated 25 size sets.}
\label{fig:MNIST_big_grid}
\end{figure}
\FloatBarrier

\end{document}